\documentclass[journal]{IEEEtran}
\IEEEoverridecommandlockouts
\usepackage{times}

\usepackage{subfigure}
\usepackage{subfigmat}

\usepackage{epsfig}

\usepackage{cite}
\usepackage{xspace}
\usepackage[usenames]{color}
\usepackage{bm}
\usepackage{amsmath}
\usepackage{amssymb}
\usepackage{stmaryrd}
\usepackage[vlined,linesnumbered,boxruled]{algorithm2e}

\usepackage{hyperref}
\hypersetup{
  colorlinks=false,
  linkbordercolor={1 1 1}, 
  citebordercolor={1 1 1} 
  } 

\usepackage{fancybox}

\usepackage{theorem}
\newtheorem{theorem}{Theorem}
\newtheorem{lemma}[theorem]{Lemma}

\newtheorem{problem}[theorem]{Problem}
\newtheorem{assumption}[theorem]{Assumption}
\newtheorem{remark}[theorem]{Remark}

\newcommand{\True}{{\tt True}}

\newcommand{\given}{\,\vert\,}
\newcommand{\biggiven}{\,\big\vert\,}

\newcommand{\PP}{\mathbb{P}}
\newcommand{\EE}{\mathbb{E}}
\newcommand{\NearNodes}{Near}

\DeclareMathOperator{\Cl}{cl}

\title{Incremental Sampling-based Algorithms\\ for Optimal Motion Planning} \date{}

\author{Sertac Karaman
  \quad\quad\quad\and \quad\quad\quad Emilio Frazzoli\thanks{The authors are with the Laboratory for
    Information and Decision Systems, Massachusetts Institute of Technology, Cambridge,
    MA.}\thanks{Manuscript submitted to International Journal of Robotics Research.}}

\begin{document}

\maketitle

\begin{abstract}
  During the last decade, incremental sampling-based motion planning algorithms, such as the
  Rapidly-exploring Random Trees (RRTs) have been shown to work well in practice and to possess
  theoretical guarantees such as probabilistic completeness. However, no theoretical bounds on the
  quality of the solution obtained by these algorithms have been established so far. The first
  contribution of this paper is a negative result: it is proven that, under mild technical
  conditions, the cost of the best path in the RRT converges almost surely to a non-optimal value.
  Second, a new algorithm is considered, called the Rapidly-exploring Random Graph (RRG), and it is
  shown that the cost of the best path in the RRG converges to the optimum almost surely. Third, a
  tree version of RRG is introduced, called the RRT$^*$ algorithm, which preserves the asymptotic
  optimality of RRG while maintaining a tree structure like RRT. The analysis of the new algorithms
  hinges on novel connections between sampling-based motion planning algorithms and the theory of
  random geometric graphs. In terms of computational complexity, it is shown that the number of
  simple operations required by both the RRG and RRT$^*$ algorithms is asymptotically within a
  constant factor of that required by RRT.
\end{abstract}

\section{Introduction} \label{section:introduction}

The robotic motion planning problem has received a considerable amount of attention, especially over
the last decade, as robots started becoming a vital part of modern industry as well as our daily
life. Even though modern robots may possess significant differences in sensing, actuation, size,
workspace, application, etc., the problem of navigating through a complex environment is embedded
and essential in almost all robotics applications. Moreover, this problem has several applications
in other disciplines such as verification, computational biology, and computer
animation~\cite{latombe.ijrr99, bhatia.frazzoli.hscc04, branicky.curtis.ea.ieeeproc06,
  cortes.jailet.ea.icra07, liu.badler.comp_anim_conf03, finn.kavraki.algorithmica99}.

Informally speaking, given a robot with a description of its dynamics, a description of the
environment, an initial state, and a set of goal states, the motion planning problem is to find a
sequence of control inputs so as the drive the robot from its initial state to one of the goal
states while obeying the rules of the environment, e.g., not colliding with the surrounding
obstacles. An algorithm that solves this problem is said to be {\em complete} if it returns a
solution when one exists and returns failure otherwise.

Early algorithmic approaches to the motion planning problem mainly focused on developing complete
planners for, e.g., polygonal robots moving among polygonal
obstacles~\cite{lozanoperez.wesley.comm_acm79}, and later for more general cases using algebraic
techniques~\cite{schwartz.sharir.adv_app_math83}. These algorithms, however, suffered severely from
computational complexity, which prohibited their practical implementations; for instance, the
algorithm in\cite{schwartz.sharir.adv_app_math83} had doubly exponential time complexity. Regarding
the computational complexity of the problem, a remarkable result was proven as early as 1979: Reif
showed that a most basic version of the motion planning problem, called the piano movers problem,
was PSPACE-hard~\cite{reif.sym_foun_com_sci79}, which strongly suggested that complete planners are
would be unsuitable for practical applications.

Practical planners came around with the development of cell decomposition
methods~\cite{brooks.lozanoperez.icai83, barraquand.latombe.ijrr93}, potential
fields~\cite{khatib.ijrr86}, and roadmap methods~\cite{canny.book88}. These approaches relaxed the
completeness requirement to, for instance, {\em resolution completeness}, which guarantees
completeness only when the resolution parameter of the algorithm is set fine enough. These planners
demonstrated remarkable performance in accomplishing various tasks in complex environments within
reasonable time bounds~\cite{ge.cui.autorobo02}. However, their practical applications were mostly
limited to state spaces with up to five dimensions, since decomposition-based methods suffered from
large number of cells, and potential field methods from local minima~\cite{koren.borenstein.icra91,
  latombe.ijrr99}.

Arguably, some of the most influential recent advances in robotic motion planning can be found in
sampling-based algorithms~\cite{kavraki.latombe.icra94,
  kavraki.svetska.ea.tro96,lavalle.kuffner.ijrr01}, which attracted a large amount of attention over
the last decade, including very recent work (see, e.g.,~\cite{ prentice.roy.ijrr09,
  tedrake.manchester.ea.ijrr, luders.karaman.ea.acc10, berenson.kuffner.ea.icra08,
  yershova.lavalle.rep08, stilman.schamburek.ea.icra07, koyuncu.ure.ea.j_intell_robot_syst10}). In
summary, sampling-based algorithms connect a set of points that are chosen randomly from the
obstacle-free space in order to build a roadmap, which is used to form a strong hypothesis of the
connectivity of the environment, in particular the connectivitiy of the initial state and the set of
goal states.
Informally speaking, sampling-based methods provide large amounts of computational savings by
avoiding explicit construction of obstacles in the state space, as opposed to most complete motion
planning algorithms. Even though these algorithms do not achieve completeness, they provide {\em
  probabilistic completeness} guarantees in the sense that the probability that the planner fails to
return a solution, if one exists, decays to zero as the number of samples approaches
infinity~\cite{barraquand.kavraki.ea.ijrr97}. Moreover, the rate of decay of the probability of
failure is exponential, under the assumption that the environment has good ``visibility''
properties~\cite{barraquand.kavraki.ea.ijrr97}. More recently, the empirical success of
sampling-based algorithms was argued to be strongly tied to the hypothesis that most practical
robotic applications, even though involving robots with many degrees of freedom, feature
environments with such good visibility properties~\cite{hsu.latombe.ea.ijrr06}.

\subsection{Sampling-Based Algorithms}
In practical applications, the class of sampling-based motion planning algorithms, including, for
instance, Probabilistic Roadmaps (PRMs)~\cite{kavraki.svetska.ea.tro96,
  kavraki.kolountzakis.ea.tro98}, and Rapidly-exploring Random Trees
(RRTs)~\cite{kuffner.lavalle.icra00, lavalle.kuffner.ijrr01}, is one of the most widely-used
approaches to the motion planning problem~\cite{lavalle.book06,lindemann.lavalle.symp_rr05}. Even
though the idea of connecting samples from the state space is essential in both approaches, these
two algorithms mainly differ in the way that they construct the roadmap.

The PRM algorithm and its variants are multiple-query methods that first construct a graph (the
roadmap), which represents a large portion of the collision-free trajectories, and then answer
queries by computing a shortest path that connects the initial state with a final state through the
graph. Shortly after their introduction in the literature, PRMs have been reported to extend to
high-dimensional state spaces~\cite{kavraki.svetska.ea.tro96}, since they avoid explicit
construction of obstacles in the state space of the robot. Subsequently, it was shown that the PRM
algorithm is probabilistically complete and the probability of failure decays to zero
exponentially~\cite{kavraki.kolountzakis.ea.tro98}.
During the last decade, the PRM algorithm has been a focus of robotics research. In particular,
several improvements were suggested by many authors and the reasons to why it performs well in
several practical cases were better understood (see, e.g., ~\cite{branicky.lavalle.ea.icra01,
  hsu.latombe.ea.ijrr06, ladd.kavraki.tro04} for some examples).

Even though multiple-query methods are valuable in highly structured environments such as factory
floors, most online planning problems do not require multiple queries, since, for instance, the
robot moves from one environment to another, or the environment is not known a priori. Moreover, in
some applications, computing a roadmap a priori may be computationally challenging or even
infeasible. Tailored mainly for these applications, incremental sampling-based planning algorithms
have emerged as an online counterparts to the PRMs (see, e.g.,~\cite{kuffner.lavalle.icra00,
  hsu.kindel.ea.ijrr02}). The incremental nature of these algorithms allows termination as soon as a
solution is found. Furthermore, slight changes in the environment do not require planning from
scratch, since most of the trajectories in the tree are still feasible (this is a property shared
with most other roadmap-based methods) or can be improved to produce feasible trajectories quickly.
These properties provide significant computational savings making online real-time implementations
possible.

A widely-studied incremental algorithm is the RRT. Closely after its introduction in the literature,
several theoretical guarantees provided by the RRT algorithm were shown, including probabilistic
completeness~\cite{kuffner.lavalle.icra00} and exponential rate of decay of the probability of
failure~\cite{frazzoli.dahleh.ea.jgcd02}.
Subsequently, the RRT algorithm has been improved and found many applications in the robotics domain
and elsewhere (see for instance~\cite{frazzoli.dahleh.ea.jgcd02, bhatia.frazzoli.hscc04,
  cortes.jailet.ea.icra07, branicky.curtis.ea.ieeeproc06, branicky.curtis.ea.cdc03,
  zucker.kuffner.ea.icra07}). In particular, RRTs have been shown to work effectively for systems
with differential constraints, nonlinear dynamics, and non-holonomic
constraints~\cite{lavalle.kuffner.ijrr01, frazzoli.dahleh.ea.jgcd02} as well as purely discrete or
hybrid systems~\cite{branicky.curtis.ea.cdc03}. Moreover, the RRT algorithm was demonstrated in
major robotics events on various experimental robotic platforms~\cite{bruce.veloso.lncs02,
  kuwata.teo.ea.cst09, teller.walter.ea.icra10, shkolnik.levashov.ea.unpub09,
  kuffner.kagami.ea.autorobo02}.

An interesting extension of the RRT that is worth mentioning at this point is the Rapidly-exploring
Random Graph (RRG) algorithm~\cite{karaman.frazzoli.cdc09}. The RRG algorithm was proposed as an
incremental sampling-based method in order to extend RRTs to generate feasible trajectories for
satisfying specifications other than the usual ``avoid obstacles and reach the goal''.
In~\cite{karaman.frazzoli.cdc09}, the authors have shown that the RRG algorithm can handle any
specification given in the form of $\mu$-calculus, which includes the widely-used specification
language Linear Temporal Logic (LTL) as a strict subclass. The RRG algorithm incrementally builds a
graph, instead of a tree, since several specifications in LTL (e.g., $\omega$-regular
properties~\cite{clarke.grumberg.ea.99}) require cyclic trajectories, which are not included in
trees.

\subsection{Optimal Motion Planning}
Generally, the standard robotic motion planning problem does not take the ``quality'' of the
solution into account. That is, the focus of the problem is to find a trajectory that starts from
the initial state and reaches one of the goal states while staying inside the obstacle-free portion
of the state space. Let us note that, as mentioned earlier, even this form of the problem is
computationally challenging~\cite{reif.sym_foun_com_sci79}. Consequently, there has been a
significant amount of research effort devoted to constructing feasible solutions in a
computationally effective manner. Hence, most of the existing algorithms pay almost no attention to
the other properties of solution, e.g., its length, or its cost. As a result, the analysis of the
quality of the solution returned by the algorithm, or algorithms tailored to provably improve the
quality of the solution if kept running, have received relatively little attention, even though
their importance is mentioned in early seminal papers~\cite{lavalle.kuffner.ijrr01}. In fact, in many field implementations of such algorithms
(see, e.g.,~\cite{kuwata.teo.ea.cst09}), it is often the case that a feasible path is found quickly,
and additional computation time is devoted to improving the solution with heuristics until the
solution is executed.

Yet, the importance of the quality of the solution returned by the planners have been noticed, in
particular, from the point of view of incremental sampling-based motion planning.
In~\cite{urmson.simmons.iros03}, Urmson and Simmons have proposed heuristics to bias the tree growth
in the RRT towards those regions that result in low-cost solutions. They have also shown
experimental results evaluating the performance of different heuristics in terms of the quality of
the solution returned. In~\cite{ferguson.stentz.iros06}, Ferguson and Stentz have considered running
the RRT algorithm multiple times in order to progressively improve the quality of the solution. They
showed that each run of the algorithm results in a path with smaller cost, even though the procedure
is not guaranteed to converge to an optimum solution. (Criteria for restarting multiple RRT runs, in a different context, were also proposed in~\cite{wedge.branicky.aaai_ai_conf08}.)

A different approach that also offers optimality guarantees is based on graph search algorithms (such as A$^*$) applied over a grid that is generated offline.
Recently, these algorithms received a large amount of research attention. In particular, they were
extended to run in an anytime fashion~\cite{likhachev.gordon.ea.nips04,
  likhachev.ferguson.ea.aij08}, deal with dynamic environments~\cite{stentz.ijcai95,
  likhachev.ferguson.ea.aij08}, and handle systems with differential
constraints~\cite{likhachev.ferguson.ijrr09}.
The have also been successfully demonstrated on various robotic
platforms~\cite{likhachev.ferguson.ijrr09, dolgov.thrun.ea.exp_robotics09}.
However, optimality guarantees of these algorithms are only ensured up to the grid resolution.
Moreover, since the number of grid points grows exponentially with the dimensionality of the
state-space, so does the (worst-case) running time of these algorithms.

\subsection{Statement of Contributions}
To the best of our knowledge, this paper provides the first thorough analysis of optimality
properties of incremental sampling-based motion planning algorithms. In particular, we show that the
probability that the RRT converges to an optimum solution, as the number of samples approaches
infinity, is zero under some reasonable technical assumptions. In fact, we give a formal proof of the
fact that the RRT algorithm almost always converges to a suboptimal solution.
Second, we show that the probability of the same event for the  RRG algorithm is
one. That is, the RRG algorithm is asymptotically optimal in the sense that it converges to an
optimal solution almost surely as the number of samples approaches infinity.
Third, we propose a novel variant of the RRG algorithm, called RRT$^*$, which inherits the
asymptotic optimality of the RRG algorithm while maintaining a tree structure. To do so, the RRT$^*$
algorithm essentially ``rewires'' the tree structure as it discovers new lower-cost paths reaching
the nodes that are already in the tree.
Finally, we show that the asymptotic computational complexity of the RRG and RRT$^*$ algorithms is
essentially the same as that of RRTs. We also provide experimental evidence supporting our
theoretical computational complexity results.
We discuss the relation between the proposed algorithms and PRM-based methods, and propose a
PRM-like algorithm that improves the efficiency of PRM while maintaining the same feasibility and
optimality properties.

To our knowledge, the algorithms considered in this paper are the first computationally efficient
incremental sampling-based motion planning algorithms with asymptotic optimality guarantees. Indeed,
our results imply that these algorithms are optimal also from an asymptotic computational complexity
point of view, since they closely match lower bounds for computing nearest neighbors. The key
insight is that connections between vertices in the graph should be sought within balls whose radius
vanishes with a certain rate as the size of the graph increases, and is based on new connections
between motion planning and the theory of random geometric graphs~\cite{penrose.book03,
  dall.christensen.phys_rev_e02}.

In this paper, we consider the problem of navigating through a connected bounded subset of a
$d$-dimensional Euclidean space. As in the early seminal papers on incremental sampling-based motion
planning algorithms such as~\cite{kuffner.lavalle.icra00}, we assume no differential constraints,
but our methods can be easily extended to planning in configuration spaces and applied to several
practical problems of interest. We defer the extension to systems with differential constraints to
another paper.

\subsection{Paper Organization}
This paper is organized as follows. In Section~\ref{section:notation}, lays the ground in terms of
notation and problem formulation. Section~\ref{section:algorithms} is devoted to the introduction of
the RRT and RRG algorithms. In Section~\ref{section:analysis}, these algorithms are analyzed in
terms of probabilistic completeness, asymptotic optimality, and computational complexity. The
RRT$^*$ algorithm is presented in Section~\ref{section:optrrt}, where it is shown that RRT$^*$
inherits the theoretical guarantees of the RRG algorithm. Experimental results are presented and
discussed in Section~\ref{section:simulations}. Concluding remarks are given in
Section~\ref{section:conclusion}. In order not to disrupt the flow of the presentation, proofs of
important results are presented in the Appendix.

\section{Preliminary Material} \label{section:notation}
\subsection{Notation}
Let $\mathbb{N}$ denote the set of non-negative integers and $\mathbb{R}$ denote the set of reals.
Let $\mathbb{R}_{>0}$ and $\mathbb{R}_{\ge 0}$ denote the sets of positive and non-negative reals,
respectively. The set $\mathbb{N}_{>0}$ is defined similarly. A sequence on a set $A$ is a mapping
from $\mathbb{N}$ to $A$, denoted as $\{ a_i\}_{i \in \mathbb{N}}$, where $a_i \in A$ is the element
that $i \in \mathbb{N}$ is mapped to. Given $a, b \in \mathbb{R}$, closed and open intervals between
$a$ and $b$ are denoted by $[a,b]$ and $(a,b)$, respectively. More generally, if $a, b \in
\mathbb{R}^d$, then $[a,b]$ is used to denote the line segment that connects $a$ with $b$, i.e.,
$[a,b] := \{\tau a + (1 - \tau) b \,\vert\, \tau \in \mathbb{R}, 0 \le \tau \le 1\}$. The
Euclidean norm is denoted by $\Vert \cdot \Vert$. Given a set $X \subset \mathbb{R}^d$, the closure
of $X$ is denoted by $\Cl(X)$. The closed ball  of radius $r>0$ centered at $x \in
\mathbb{R}^d$ is defined as ${\cal B}_{x,r}:=\{ y \in \mathbb{R}^d \,\vert\,\, \Vert y - x \Vert \le r\}$; ${\cal
  B}_{x,r}$ is also called the $r-$ball centered at $x$. Given a set $X \subseteq \mathbb{R}^d$, the
Lebesgue measure of $X$, i.e., its volume, is denoted by $\mu (X)$. The volume of the unit ball in
$\mathbb{R}^d$ is denoted by $\zeta_d$.

Given a set $X \subset \mathbb{R}^d$, and a scalar $s \ge 0$, a path in $X$ is a continuous function
$\sigma : [0, s] \to X$, where $s$ is the length of the path defined in the usual way as $\sup_{ \{
  n \in \mathbb{N}, 0 = \tau_0 < \tau_1< \dots< \tau_n = s \} } \sum_{i = 1}^n \Vert
\sigma(\tau_i) - \sigma(\tau_{i-1}) \Vert$. Given two paths in $X$, $\sigma_1: [0, s_1] \to X$, and
$\sigma_2: [0, s_2] \to X$, with $\sigma_1(s_1) = \sigma_2(0)$, their concatenation is denoted by
$\sigma_1 \vert \sigma_2$. More precisely, $\sigma = \sigma_1 \vert \sigma_2 : [0, s_1 + s_2] \to X$
is defined as $\sigma(s) = \sigma_1(s)$ for all $s \in [0, s_1]$, and $\sigma(s) = \sigma_2(s-s_1)$
for all $s \in [s_1, s_1+s_2]$. More generally, a concatenation of multiple paths $\sigma_1 \vert
\sigma_2 \vert \cdots \vert \sigma_n$ is defined as $( \cdots ((\sigma_1 \vert \sigma_2) \vert
\sigma_3 ) \vert \cdots \vert \sigma_n)$. The set of all paths in $X$ with nonzero length is denoted
by $\Sigma_X$. The straight continuous path between two points $x_1, x_2 \in
\mathbb{R}^d$ is denoted by ${\tt Line}(x_1, x_2)$.

Given a probability space $(\Omega, {\cal F}, \PP)$, where $\Omega$ is a sample space, ${\cal F}
\subseteq 2^\Omega$ is a $\sigma-$algebra, and $\PP$ is a probability measure, an event $A$ is an
element of ${\cal F}$. The complement of an event $A$ is denoted by $A^c$. Given a sequence of
events $\{ A_i \}_{i \in \mathbb{N}}$, the event $\cap_{j = 1}^\infty \cup_{i = j}^\infty A_i$ is
denoted by $\limsup_{i \to \infty} A_i$ (also called the event that $A_i$ occurs infinitely often).
A (real) random variable is a measurable function that maps $\Omega$ into $\mathbb{R}$. An extended
(real) random variable can also take the values $\pm\infty$. A sequence of random variables $\{
{\cal Y}_i \}_{i \in \mathbb{N}}$ is said to converge surely to a random variable ${\cal Y}$ if
$\lim_{i \to \infty} {\cal Y}_i (\omega) = {\cal Y} (\omega)$ for all $\omega \in \Omega$; the
sequence is said to converge almost surely if $\PP (\{ \lim_{i \to \infty} {\cal Y}_i = {\cal Y}\})
= 1$. Finally, if $\varphi (\omega)$ is a property that is either true or false for a given $\omega
\in \Omega$, the event that denotes the set of all samples $\omega$ for which $\varphi(\omega)$
holds, i.e., $\{ \omega \in \Omega \,\vert\, \varphi(\omega) \mbox{ holds} \}$, is written as
$\{\varphi\}$, e.g., $\{\omega \in \Omega \,\vert\, \lim_{i \to \infty} {\cal Y}_i (\omega) = {\cal
  Y}(\omega)\}$ is simply written as $\{\lim_{i \to \infty}{\cal Y}_{i} = {\cal Y}\}$.

Let $f(n)$ and $g(n)$ be two functions with domain and range $\mathbb{N}$.
The function $f(n)$ is said to be $O(g(n))$, denoted as $f(n) \in O(g(n))$, if there exists two
constants $M$ and $n_0$ such that $f(n) \le M g(n)$ for all $n \ge n_0$. The function $f(n)$ is said
to be $\Omega(g(n))$, denoted as $f(n) \in \Omega(g(n))$, if there exists constants $M$ and $n_0$
such that $f(n) \ge M g(n)$ for all $n \ge n_0$. The function $f(n)$ is said to be $\Theta(g(n))$,
denoted as $f(n) \in \Theta(g(n))$, if $f(n) \in O(g(n))$ and $f(n) \in \Omega(g(n))$.

Let $X$ be a subset of $\mathbb{R}^d$. A (directed) graph $G = (V,E)$ on $X$ is composed of a vertex
set $V$ and an edge set $E$, such that $V$ is a finite subset of $X$, and $E$ is a subset of $V
\times V$. The set of all directed graphs on $X$ is denoted by ${\cal G}_X$. A directed path on $G$
is a sequence $(v_1, v_2, \dots, v_n)$ of vertices such that $(v_i,v_{i+1}) \in E$ for all $1 \le
i\le n-1$. Given a vertex $v \in V$, the sets $\{ u \in V \,\vert\, (u,v) \in E\}$ and $\{u \in V
\,\vert\, (v,u) \in E\}$ are said to be its incoming neighbors and outgoing neighbors, respectively.
A (directed) tree is a directed graph, in which each vertex but one has a unique incoming neighbor;
the vertex with no incoming neighbor is called the root vertex. Vertices of a tree are often also
called nodes.

\subsection{Problem Formulation} \label{section:problem}

In this section, two variants of the path planning problem are presented. First, the feasibility
problem in path planning is formalized, then the optimality problem is introduced.

Let $X$ be a bounded connected open subset of $\mathbb{R}^d$, where $d \in
\mathbb{N}$, $d \ge 2$. 
Let $X_\mathrm{obs}$ and $X_\mathrm{goal}$, called the {\em obstacle region} and the {\em goal
  region}, respectively, be open subsets of $X$. Let us denote the {\em obstacle-free space}, i.e.,
$X \setminus X_\mathrm{obs}$, as $X_\mathrm{free}$. Let the {\em initial state}, $x_\mathrm{init}$,
be an element of $X_\mathrm{free}$.
In the sequel, a path in $X_\mathrm{free}$ is said to be a {\em collision-free path}. A
collision-free path that starts at $x_\mathrm{init}$ and ends in the goal region is said to be a
{\em feasible path}, i.e., a collision-free path $\sigma: [0,s] \to X_\mathrm{free}$ is feasible if
and only if $\sigma(0) = x_\mathrm{init}$ and $\sigma(s) \in X_\mathrm{goal}$.

The {\em feasibility problem} of path planning is to find a feasible path, if one exists, and report
failure otherwise. The problem can be formalized as follows.
\begin{problem}[Feasible planning] \label{problem:feasibility} %
  Given a bounded connected open set $X \subset \mathbb{R}^d$, an obstacle space $X_\mathrm{obs}
  \subset X$, an initial state $x_\mathrm{init} \in X_\mathrm{free}$, and a goal region
  $X_\mathrm{goal} \subset X_\mathrm{free}$, find a path $\sigma: [0, s] \to X_\mathrm{free}$ such
  that $\sigma(0) = x_\mathrm{init}$ and $\sigma(s) \in X_\mathrm{goal}$, if one exists. If no such
  path exists, then report failure.
\end{problem}

Let $c : \Sigma_{X_\mathrm{free}} \to \mathrm{R}_{> 0}$ be a function, called the {\em cost
  function}, which assigns a non-negative cost to all nontrivial collision-free paths. The {\em
  optimality problem} of path planning asks for finding a feasible path with minimal cost,
formalized as follows.

\begin{problem}[Optimal planning] \label{problem:optimality} %
  Given a bounded connected open set $X$, an obstacle space $X_\mathrm{obs}$, an initial state
  $x_\mathrm{init}$, and a goal region $X_\mathrm{goal}$, find a path $\sigma^*: [0, s] \to
  \Cl(X_\mathrm{free})$ such that (i) $\sigma^*(0) = x_\mathrm{init}$ and $\sigma^*(s) \in
  X_\mathrm{goal}$, and (ii) $c(\sigma^*) = \min_{\sigma \in \Sigma_{\Cl(X_\mathrm{free})}}
  c(\sigma)$. If no such path exists, then report failure.
\end{problem}

\section{Algorithms} \label{section:algorithms} 

In this section, two incremental sampling-based motion planning algorithms, namely the RRT and the
RRG algorithms, are introduced.
Before formalizing the algorithms, let us note the primitive procedures that they rely on.

{\em Sampling}: The function $\mathtt{Sample}: \mathbb{N} \to X_\mathrm{free}$ returns independent
identically distributed (i.i.d.) samples from $X_\mathrm{free}$.

{\em Steering:} Given two points $x, y \in X$, the function ${\tt Steer} : (x,y) \mapsto z$ returns
a point $z \in \mathbb{R}^d$ such that $z$ is ``closer'' to $y$ than $x$ is. Throughout the paper,
the point $z$ returned by the function ${\tt Steer}$ will be such that $z$ minimizes $\Vert z - y
\Vert$ while at the same time maintaining $\Vert z - x \Vert \le \eta$, for a prespecified $\eta >
0$,\footnote{This steering procedure is used widely in the robotics literature, since its
  introduction in~\cite{kuffner.lavalle.icra00}. Our results also extend to the Rapidly-exploring
  Random Dense Trees (see, e.g., \cite{lavalle.book06}), which are slightly modified versions of the
  RRTs that do not require tuning any prespecified parameters such as $\eta$ in this case.} i.e.,
$$
{\tt Steer} (x, y) = \displaystyle \mathrm{argmin}_{z \in \mathbb{R}^d, \Vert z - x \Vert \le \eta}
\Vert z - y \Vert.
$$

{\em Nearest Neighbor:} Given a graph $G = (V,E) \in {\cal G}_{X_\mathrm{free}}$ and a point $x \in
X_\mathrm{free}$ , the function ${\tt Nearest} : (G, x) \mapsto v$ returns a vertex $v \in V$ that is
``closest'' to $x$ in terms of a given distance function. In this paper, we will use
Euclidean distance (see, e.g.,~\cite{lavalle.kuffner.ijrr01} for alternative choices), i.e.,
$${\tt Nearest} (G = (V,E), x) = \mathrm{argmin}_{v \in V} \Vert x-v \Vert.
$$

{\em Near Vertices:} Given a graph $G = (V, E) \in {\cal G}_{X_\mathrm{free}}$, a point $x \in
X_\mathrm{free}$, and a number $n \in \mathbb{N}$, the function ${\tt \NearNodes}: (G, x, n) \mapsto
V'$ returns a set $V'$ of vertices such that $V' \subseteq V$. The ${\tt \NearNodes}$ procedure can
be thought of as a generalization of the nearest neighbor procedure in the sense that the former
returns a collection of vertices that are close to $x$, whereas the latter returns only one such
vertex that is the closest. Just like the ${\tt Nearest}$ procedure, there are many ways to define
the ${\tt \NearNodes}$ procedure, each of which leads to different algorithmic properties. For
technical reasons to become clear later, we define ${\tt \NearNodes}(G,x,n)$ to be the set of all
vertices within the closed ball of radius $r_n$ centered at $x$, where
$$r_n = \min\left\{  \left(\frac{\gamma}{\zeta_d}\,\frac{\log n}{n} \right)^{1/d}, \eta\right\},
$$
and $\gamma$ is a constant. Hence, the volume of this ball is $\min\{ \gamma \frac{\log n}{n},
\zeta_d \, \eta^d \}$.

{\em Collision Test:} Given two points $x,x' \in X_\mathrm{free}$, the Boolean function ${\tt
  ObstacleFree} (x,x')$ returns $\True$ iff the line segment between $x$ and $x'$ lies in
$X_\mathrm{free}$, i.e., $[x, x'] \subset X_\mathrm{free}$.

Both the RRT and the RRG algorithms are similar to most other incremental sampling-based planning
algorithms (see Algorithm~\ref{algorithm:motionplanning}). Initially, the algorithms start with the
graph that includes the initial state as its single vertex and no edges; then, they incrementally
grow a graph on $X_\mathrm{free}$ by sampling a state $x_\mathrm{rand}$ from $X_\mathrm{free}$ at
random and extending the graph towards $x_\mathrm{rand}$.
In the sequel, every such step of sampling followed by extensions
(Lines~\ref{line:iteration_start}-\ref{line:iteration_end} of
Algorithm~\ref{algorithm:motionplanning}) is called a single {\em iteration} of the incremental
sampling-based algorithm.

Hence, the body of both algorithms, given in Algorithm~\ref{algorithm:motionplanning}, is the same.
However, RRGs and RRTs differ in the choice of the vertices to be extended. The $\tt Extend$
procedures for the RRT and the RRG algorithms are provided in Algorithms~\ref{algorithm:rrt} and
\ref{algorithm:rrg}, respectively. Informally speaking, the RRT algorithm extends the nearest vertex
towards the sample. The RRG algorithm first extends the nearest vertex, and if such extension is
successful, it also extends all vertices returned by the ${\tt \NearNodes}$ procedure. In both
cases, all extensions resulting in collision-free trajectories are added to the graph as edges, and
their terminal points as new vertices.

\begin{algorithm}
  $V \leftarrow \{ x_\mathrm{init}\}$; $E \leftarrow \emptyset$; $i\leftarrow 0$\; 
  \While {$i < N$} { \label{line:iteration_start}
    $G \leftarrow (V,E)$\; 
    $x_\mathrm{rand} \leftarrow {\tt Sample}(i)$; 
    $i \leftarrow i + 1$\; 
    $(V, E)\leftarrow {\tt Extend}(G,x_\mathrm{rand})$\;  \label{line:iteration_end}
  }
  \caption{Body of RRT and RRG Algorithms}
  \label{algorithm:motionplanning}
\end{algorithm}

\begin{algorithm}
  $V' \leftarrow V$; $E' \leftarrow E$\;
  $x_\mathrm{nearest} \leftarrow {\tt Nearest} (G,x)$\; 
  $x_\mathrm{new} \leftarrow {\tt Steer} (x_\mathrm{nearest}, x)$\; 
  \If{${\tt ObstacleFree}(x_\mathrm{nearest},x_\mathrm{new})$} {
    $V' \leftarrow V' \cup \{x_\mathrm{new}\}$\; 
    $E' \leftarrow E' \cup \{ (x_\mathrm{nearest}, x_\mathrm{new})\}$\; 
  }
  \Return{$G' = (V', E')$}
  \caption{${\tt Extend}_{RRT}$}
  \label{algorithm:rrt}
\end{algorithm}

\begin{algorithm}
  $V' \leftarrow V$; $E' \leftarrow E$\;
  $x_\mathrm{nearest} \leftarrow {\tt Nearest} (G,x)$\; 
  $x_\mathrm{new} \leftarrow {\tt Steer} (x_\mathrm{nearest}, x)$\; 
  \If { ${\tt ObstacleFree (x_\mathrm{nearest}, x_\mathrm{new})}$} {
    $V' \leftarrow V' \cup \{x_\mathrm{new}\}$\;
    $E' \leftarrow E' \cup \{ (x_\mathrm{nearest}, x_\mathrm{new}), (x_\mathrm{new}, x_\mathrm{nearest})
    \} $\;
    $X_\mathrm{near} \leftarrow {\tt \NearNodes} (G, x_\mathrm{new}, \vert V \vert)$\;
    \For { all $x_\mathrm{near} \in X_\mathrm{near} $ } {
      \If {${\tt ObstacleFree}(x_\mathrm{new}, x_\mathrm{near})$}{
        $E' \leftarrow E' \cup \{ (x_\mathrm{near}, x_\mathrm{new}), (x_\mathrm{new},
        x_\mathrm{near}) \}$\; 
      }
    }
  }
  \Return{$G' = (V', E')$}
  \caption{${\tt Extend}_{RRG}$}
  \label{algorithm:rrg}
\end{algorithm}

Notice that the graph $G$ maintained by the RRT algorithm is a tree, whereas that maintained by the
RRG can, in general, include vertices with more than one incoming neighbor.

\section{Analysis}
\label{section:analysis}
\subsection{Convergence to a Feasible Solution} \label{section:feasibility} 

In this section, the feasibility problem is considered. It is proven that the RRG algorithm inherits
the probabilistic completeness as well as the exponential decay of the probability of failure (as
the number of samples increase) from the RRT. These results imply that the RRT and RRG algorithms
have the same performance in producing a solution to the feasibility problem as the number of
samples increase. Before formalizing these claims, let us introduce some notation and preliminaries.

Note that the only randomness in Algorithms~\ref{algorithm:motionplanning}-\ref{algorithm:rrg} may
arise from the sampling procedure denoted as ${\tt Sample}.$\footnote{We will not address the case
  in which the sampling procedure is deterministic, but refer the reader to
  ~\cite{LaValle.Branicky.ea:04}, which contains an in-depth discussion of the relative merits of
  randomness and determinism in sampling-based motion planning algorithms.} To model this
randomness, we define the sample space $\Omega$ to be the infinite cartesian product
$X_\mathrm{free}^\infty$. Intuitively, a single element $\omega$ of $\Omega$ denotes a sequence
$\{{\tt Sample}(i)\}_{i \in \mathbb{N}}$ of infinitely many samples drawn from $X_\mathrm{free}$,
and $\Omega$ denotes the set of all such sequences.

Sets of vertices and edges of the graphs maintained by the RRT and the RRG algorithms, then, can be
defined as functions from the sample space $\Omega$ to appropriate sets.
More precisely, let $\{{\cal V}^\mathrm{RRT}_i\}_{i \in \mathbb{N}}$ and $\{{\cal
  V}^\mathrm{RRG}_i\}_{i \in \mathbb{N}}$, sequences of functions defined from $\Omega$ into finite
subsets of $X_\mathrm{free}$, be the sets of vertices in the RRT and the RRG, respectively, at the
end of iteration $i$. By convention, we define ${\cal V}^\mathrm{RRT}_0 = {\cal V}^\mathrm{RRG}_0 =
\{x_\mathrm{init}\}$. Similarly, let ${\cal E}^\mathrm{RRT}_i$ and ${\cal E}^\mathrm{RRG}_i$,
defined for all $i \in \mathbb{N}$, denote the set of edges in the RRT and the RRG, respectively, at
the end of iteration $i$. Clearly, ${\cal E}^\mathrm{RRT}_0 = {\cal E}^\mathrm{RRG}_0 = \emptyset$.

An important lemma used for proving the equivalency between the RRT and the RRG algorithms is the
following.
\begin{lemma} \label{lemma:vertexedge} %
  For all $i \in \mathbb{N}$ and all $\omega \in \Omega$, ${\cal V}^\mathrm{RRT}_i(\omega) = {\cal
    V}^\mathrm{RRG}_i (\omega)$ and ${\cal E}^\mathrm{RRT}_i(\omega) \subseteq {\cal
    E}^\mathrm{RRG}_i(\omega)$.
\end{lemma}

Lemma~\ref{lemma:vertexedge} implies that the paths discovered by the RRT algorithm by the end of
iteration $i$ is, essentially, a subset of those discovered by the RRG by the end of the same
iteration.

An algorithm addressing Problem~\ref{problem:feasibility} is said to be {\em probabilistically
  complete} if it finds a feasible path with probability approaching one as the number of iterations
approaches infinity. Note that there exists a collision-free path starting from $x_\mathrm{init}$ to
any vertex in the tree maintained by the RRT, since the RRT maintains a connected graph on
$X_\mathrm{free}$ that necessarily includes $x_\mathrm{init}$. Using this fact, the probabilistic
completeness property of the RRT is stated alternatively as follows.
\begin{theorem}[see~\cite{lavalle.kuffner.ijrr01}] \label{theorem:probabilisticcompletenessrrt} %
  If there exists a feasible solution to Problem~\ref{problem:feasibility}, then $\lim_{i \to
    \infty} \PP \left(\left\{ {\cal V}^\mathrm{RRT}_i \cap X_\mathrm{goal} \neq \emptyset
    \right\}\right) = 1$.
\end{theorem}

Moreover, under certain assumptions on the environment, the probability that the RRT fails to find a
feasible path, even though one exists, decays to zero exponentially (see, e.g.,
\cite{lavalle.kuffner.ijrr01,hsu.kindel.ea.ijrr02}). We state these assumptions here and then
state the theorem in our notation. 

An {\em attraction sequence}~\cite{lavalle.kuffner.ijrr01} is defined as a finite
  sequence ${\cal A} = \{A_1, A_2, \dots, A_k \}$ of sets as follows: (i) $A_0 =
  \{x_\mathrm{init}\}$, and (ii) for each set $A_i$, there exists a set $B_i$, called the {\em
    basin} such that for any $x \in A_{i-1}$, $y \in A_i$, and $z \in X \setminus B_i$, there holds
  $\Vert x - y \Vert \le \Vert x - z \Vert$\footnote{Since we do not consider any differential
    constraints, some of the other assumptions introduced in~\cite{lavalle.kuffner.ijrr01} are
    immediately satisfied.}. Given an attraction sequence ${\cal A}$ of length $k$, let $p_k$ denote
  $\min_{i \in \{1,2, \dots, k\}} \left(\frac{\mu(A_i)}{\mu(X_\mathrm{free})}\right)$.

The following theorem states that the probability that the RRT algorithm fails to return a solution,
when one exists, decays to zero exponentially fast.
\begin{theorem}[see~\cite{lavalle.kuffner.ijrr01}] \label{theorem:exponentialdecayrrt} %
  If there exists an attraction sequence ${\cal A}$ of length $k$, then $\PP \left( \left\{ {\cal
        V}^\mathrm{RRT}_i \cap X_\mathrm{goal} = \emptyset\right\}\right) \le e^{-\frac{1}{2} (i \,
    p_k - 2k)}$.
\end{theorem}

As noted in~\cite{lavalle.kuffner.ijrr01}, the convergence gets faster as the length of the
attraction sequence gets smaller and the volume of the minimum volume set in the attraction sequence
gets larger. Such properties are achieved when the environment does not involve narrow passages and,
essentially, has good ``visibility'' properties. (See, e.g.,~\cite{hsu.latombe.ea.ijrr06} for 
dependence of the performance of probabilistic planners such as PRMs on the visibility properties of
the environment).

With Lemma~\ref{lemma:vertexedge} and Theorems~\ref{theorem:probabilisticcompletenessrrt} and
\ref{theorem:exponentialdecayrrt}, the following theorem is immediate.
\begin{theorem} \label{theorem:rrtrrgequivalency} %
  If there exists a feasible solution to Problem~\ref{problem:feasibility}, then $\lim_{i \to
    \infty} \PP \left( \left\{ {\cal V}^\mathrm{RRG}_i \cap X_\mathrm{goal} \neq \emptyset \right\}
  \right) = 1$. Moreover, if an attraction sequence ${\cal A}$ of length $k$ exists, then $\PP
  \left( \left\{ {\cal V}^\mathrm{RRG}_i \cap X_\mathrm{goal} = \emptyset \right\} \right) \le
  e^{-\frac{1}{2}(i\,p_k - 2\,k)}$.
\end{theorem}

\subsection{Asymptotic Optimality} \label{section:optimality} 

This section is devoted to the investigation of optimality properties of the RRT and the RRG
algorithms. First, under some mild technical assumptions, we show that the probability that the RRT
converges to an optimal solution is zero. However, the convergence of this random variable is
guaranteed, which implies that the RRT converges to a {\em non}-optimum solution with probability
one. On the contrary, we subsequently show that the RRG algorithm converges to an optimum solution
almost-surely.

Let $\{{\cal Y}^\mathrm{RRT}_i\}_{i \in \mathbb{N}}$ be a sequence of extended random variables that denote
the cost of a minimum-cost path contained within the tree maintained by the RRT algorithm at the end
of iteration $i$. The extended random variable ${\cal Y}^\mathrm{RRG}_i$ is defined similarly. Let $c^*$
denote the cost of a minimum-cost path in $\Cl(X_\mathrm{free})$, i.e., the cost of a path that solves
Problem~\ref{problem:optimality}.

Let us note that the limits of these two extended random variable sequences as $i$ approaches
infinity exist. More formally, notice that ${\cal Y}^\mathrm{RRT}_{i+1} (\omega) \le {\cal
  Y}^\mathrm{RRT}_i (\omega)$ holds for all $i \in \mathbb{N}$ and all $\omega \in \Omega$.
Moreover, we have that ${\cal Y}^\mathrm{RRT}_i (\omega) \ge c^*$ for all $i \in \mathbb{N}$ and all
$\omega \in \Omega$, by optimality of $c^*$. Hence, $\{ {\cal Y}^\mathrm{RRT}_i\}_{i \in
  \mathbb{N}}$ is a surely non-increasing sequence of random variables that is surely lower-bounded
by $c^*$. Thus, for all $\omega \in \Omega$, the limit $\lim_{i \to \infty} {\cal Y}^\mathrm{RRT}_i
(\omega)$ exists. The same argument also holds for the sequence $\{{\cal Y}^\mathrm{RRG}_i\}_{i \in
  \mathbb{N}}$.

\subsubsection{Almost Sure Suboptimality of the RRT} 

Let us note the following assumptions, which will be required to show the almost-sure sub-optimality
of the RRT.
Let $\Sigma^*$ denote the set of all optimal paths, i.e., the set of all paths that solve
Problem~\ref{problem:optimality}, and $X_\mathrm{opt}$ denote the set of states that an optimal path
in $\Sigma^*$ passes through, i.e., 
$$
X_\mathrm{opt} = \cup_{\sigma^* \in \Sigma^*} \cup_{\tau \in [0, s^*]} \{ \sigma^*(\tau) \}
$$
\begin{assumption}[Zero-measure Optimal Paths] \label{assumption:zeromeasureoptimal} %
  The set of all points in the state-space that an optimal trajectory passes through has measure
  zero, i.e., $ \mu \left( X_\mathrm{opt} \right) = 0$.
\end{assumption}
\begin{assumption}[Sampling Procedure] \label{assumption:sampling} The sampling procedure is such
  that the samples $\{{\tt Sample}(i)\}_{i \in \mathbb{N}}$ are drawn from an absolutely continuous
  distribution with a continuous density function $f(x)$ bounded away from zero on
  $X_\mathrm{free}$.
\end{assumption}
\begin{assumption}[Monotonicity of the Cost Function] \label{assumption:monotonecost}
  For all $\sigma_1, \sigma_2 \in \Sigma_{X_\mathrm{free}}$, the cost function $c$ satisfies the
  following: $c(\sigma_1) \le c(\sigma_1 \vert \sigma_2)$.
\end{assumption}

Assumptions~\ref{assumption:zeromeasureoptimal} rules out trivial cases, in which the RRT algorithm
can sample exactly an optimal path with non-zero probability. Note that this assumption is placed on
the problem instance rather than the algorithm. Most cost functions and problem instances of
interest satisfy this assumption, including, e.g., the Euclidean length of the path. Let us note
that this assumption does not imply that there is a single optimal path; indeed, there are problem
instances with uncountably many optimal paths, for which
Assumption~\ref{assumption:zeromeasureoptimal} holds. Assumption~\ref{assumption:sampling} also
ensures that the sampling procedure can not be tuned to construct the optimal path exactly.
Finally, Assumption~\ref{assumption:monotonecost} merely states that extending a path to produce a
longer path can not decrease its cost.

Recall that $d$ denotes the dimensionality of the state space. The negative result of this section
is formalized as follows.
\begin{theorem} \label{theorem:rrtoptimality} %
  Let Assumptions~\ref{assumption:zeromeasureoptimal}, \ref{assumption:sampling}, and
  \ref{assumption:monotonecost} hold. Then, the probability that the cost of the minimum-cost path
  in the RRT converges to the optimal cost is zero, i.e.,
  $$
  \PP \left(\left\{\lim_{i \to \infty} {\cal Y}^\mathrm{RRT}_i = c^*\right\}\right) = 0,
  $$
   whenever $d \ge 2$.
\end{theorem}
As noted before, the limit $\lim_{i \to \infty} {\cal Y}^\mathrm{RRT}_i(\omega)$ exists and is a
random variable. However, Theorem~\ref{theorem:rrtoptimality} directly implies that this limit is
strictly greater than $c^*$ with probability one, i.e., $\PP \left(\{\lim_{i \to \infty} {\cal
  Y}^\mathrm{RRT}_i > c^* \}\right) = 1$. In other words, we establish, as a corollary, that the RRT
algorithm converges to a suboptimal solution with probability one.

\begin{remark}[Effectiveness of Multiple RRTs]Since the cost of the best path returned by the RRT
  algorithm converges to a random variable, say $\mathcal{Y}^\mathrm{RRT}_\infty$,
  Theorem~\ref{theorem:rrtoptimality} provides new insight explaining the effectiveness of
  approaches as in~\cite{ferguson.stentz.iros06}. In fact, running multiple instances of the RRT
  algorithm amounts to drawing multiple samples of $\mathcal{Y}^{RRT}_\infty$.
  \end{remark}

\subsubsection{Almost Sure Optimality of the RRG}

Let us note the following set of assumptions, which will be required to show the asymptotic
optimality of the RRG.
\begin{assumption}[Additivity of the Cost Function] \label{assumption:additivecost} %
  For all $\sigma_1, \sigma_2 \in \Sigma_{X_\mathrm{free}}$, the cost function $c$ satisfies the
  following: $c(\sigma_1 \vert \sigma_2) = c(\sigma_1) + c(\sigma_2)$. 
\end{assumption}
\begin{assumption}[Continuity of the Cost Function] \label{assumption:continuouscost} %
  The cost function $c$ is Lipschitz continuous in the following sense: there exists some constant
  $\kappa$ such that for any two paths $\sigma_1 : [0,s_1] \to X_\mathrm{free}$ and $\sigma_2 :
  [0,s_2] \to X_\mathrm{free}$, $$\vert c(\sigma_1) - c(\sigma_2) \vert \le \kappa \,
  \sup_{\tau \in [0,1]} \Vert \sigma_1(\tau \, s_1) - \sigma_2(\tau \, s_2) \Vert.$$
\end{assumption}
\begin{assumption}[Obstacle Spacing] \label{assumption:obstacles} %
  There exists a constant $\delta \in \mathbb{R}_+$ such that for any point $x \in X_\mathrm{free}$,
  there exists $x' \in X_\mathrm{free}$, such that (i) the $\delta$-ball centered at $x'$ lies inside
  $X_\mathrm{free}$, i.e., ${\cal B}_{x', \delta} \subset X_\mathrm{free}$, and (ii) $x$ lies inside the
  $\delta$-ball centered at $x'$, i.e., $x \in {\cal B}_{x', \delta}$.
\end{assumption}

Assumption~\ref{assumption:continuouscost} ensures that two paths that are very close to each other
have similar costs. Let us note that several cost functions of practical interest satisfy
Assumptions~\ref{assumption:additivecost} and \ref{assumption:continuouscost}. An example that is
widely used, e.g., in optimal control, is the line integral of a continuous function
$g:X_\mathrm{free} \to \mathbb{R}_{>0}$ over trajectories of bounded length, i.e., $c(\sigma) =
\int_{0}^s g(\sigma(\tau)) d\tau$. Assumption~\ref{assumption:obstacles} is a rather technical
assumption, which ensures existence of some free space around the optimal trajectories to allow
convergence.
We also assume for simplicity that the sampling is uniform, although our results can be easily
extended to more general sampling procedures.

Recall that $d$ is the dimensionality of the state-space $X$, and $\gamma$ is the constant defined
in the ${\tt \NearNodes}$ procedure. The positive result that states the asymptotic optimality of
the RRG algorithm can be formalized as follows.
\begin{theorem} \label{theorem:rrgoptimality} %
  Let Assumptions~\ref{assumption:additivecost}, \ref{assumption:continuouscost}, and
  \ref{assumption:obstacles} hold, and assume that Problem~\ref{problem:feasibility} admits a
  feasible solution. Then, the cost of the minimum-cost path in the RRG converges to $c^*$ almost
  surely, i.e.,
  $$
  \PP \left( \left\{ \lim_{i \to \infty} {\cal Y}^\mathrm{RRG}_i = c^* \right\}\right) = 1,
  $$
  whenever $d \ge 2$ and $\gamma > \gamma_L := 2^d (1 + 1/d) \mu (X_\mathrm{free})$.
\end{theorem}

\subsection{Computational Complexity} \label{section:complexity} 

The objective of this section is to compare the computational complexity of RRTs and RRGs. It is
shown that these algorithms share essentially the same asymptotic computational complexity in terms
of the number of calls to {\em simple operations} such as comparisons, additions, and
multiplications. Simple operations should not be confused with {\em primitive procedures} outlined
in Section~\ref{section:algorithms}, which are higher level functions that may have different
asymptotic computational complexities in terms of the number of simple operations that they perform.
Note that the objective of this section is not the evaluation of the computational complexity in
absolute terms, but to compare the relative complexity of the two algorithms, hence the input of
interest is the number of iterations, as opposed to parameters describing the problem instance.

Let us first investigate the computational complexity of the RRT and the RRG algorithms in terms of
the number of calls to the primitive procedures. Notice that, in every iteration, the number of
calls to ${\tt Sample}$, ${\tt Steer}$, and ${\tt Nearest}$ procedures are the same in both
algorithms. However, number of calls to ${\tt \NearNodes}$ and ${\tt ObstacleFree}$ procedures
differ: the former is never called by the RRT and is called at most once by the RRG, whereas the
latter is called exactly once by the RRT and at least once by the RRG. Informally speaking, we first
show that the expected number of calls to the ${\tt ObstacleFree}$ procedure by the RRG algorithm is
$O(\log n)$, where $n$ is the number of vertices in the graph.

Let ${\cal O}^\mathrm{RRG}_i$ be a random variable that denotes the number of calls to the ${\tt
  ObstacleFree}$ procedure by the RRG algorithm in iteration $i$. Notice that, as an immediate
corollary of Lemma~\ref{lemma:vertexedge}, the number of vertices in the RRT and RRG algorithms is
the same at any given iteration. Let ${\cal N}_i$ be the number of vertices in these algorithms at
the end of iteration $i$. The following theorem establishes that the expected number of calls to the
${\tt ObstacleFree}$ procedure in iteration $i$ by the RRG algorithm scales logarithmically with the
number of vertices in the graph as $i$ approaches infinity.

\begin{lemma} \label{lemma:complexity_of_obstaclefree} %
  In the limit as $i$ approaches infinity, the random variable ${\cal O}^\mathrm{RRG}_i/\log({\cal N}_i)$
  is no more than a constant in expectation, i.e.,
  $$
  \limsup_{i \to \infty} \EE \left[ \frac{{\cal O}^\mathrm{RRG}_i}{\log ({\cal N}_i)} \right] \le \phi,
  $$
  where $\phi \in \mathbb{R}_{>0}$ is a constant that depends only on the problem
    instance.
\end{lemma}

Hence, informally speaking, we have established that the RRG algorithm has an overhead of order
$\log n$ calls to the ${\tt ObstacleFree}$ procedure and one call to the ${\tt \NearNodes}$
procedure, where $n$ is the number of vertices in the RRG; otherwise, both algorithms scale the same in
terms of number of calls to the primitive procedures.

However, some primitive procedures clearly take more computation time to execute than others. For a
more fair comparison, we also evaluate the time required to execute each primitive procedure in
terms of the number of simple operations (also called {\em steps} in this paper) that they perform.
This analysis shows that the expected number of simple operations performed by the RRG is
asymptotically within a constant factor of that performed by the RRT, which establishes that the RRT
and the RRG algorithms have the same asymptotic computational complexity in terms of number of calls
to the simple operations.

First, notice that ${\tt Sample}$, ${\tt Steer}$, and ${\tt ObstacleFree}$ procedures can be
performed in a constant number of steps, i.e., independent of the number of vertices in the graph. (Although all these procedures clearly depend on parameters such as the dimensionality of the state-space and the number of obstacles, these are fixed for each problem instance.)

Second, let us consider the computational complexity of the ${\tt Nearest}$ procedure. Indeed, the
nearest neighbor search problem has been widely studied in the literature, since it has many
applications in, e.g., computer graphics, database systems, image processing, data mining, patern
recognition, etc.~\cite{ samet.book89b, samet.book89a}. Let us assume that the distance computation
between two points can be done in $O(d)$ time (note that for the purposes of comparing distances,
the Euclidean distance can be evaluated without computing any roots). Clearly, a brute-force
algorithm that examines every vertex runs in $O(d \, n)$ time and requires $O(1)$ space. However, in
many online real-time applications such as robotics, it is highly desirable to reduce the
computation time of each iteration under sublinear bounds, especially for anytime algorithms, which
provide better solutions as the number of iterations increase.

Fortunately, it is possible to design nearest neighbor computation algorithms that run in sublinear
time. Generally, such efficient computations of nearest neighbors are based on constructing a tree
structure online so as to reduce the query time~\cite{samet.book89a}.
Indeed, using $k$-d trees, nearest neighbor search can be performed within sublinear time bounds, in
the worst case~\cite{samet.book89a}. In fact, let us note that the $k$-d tree algorithm was extended
to arbitrary topological manifolds and used for nearest neighbor computation in
RRTs~\cite{atramentov.lavalle.icra02}.

Worst case logarithmic time, however, is hard to achieve in many cases of practical interest. Note
that when $d = 1$, a binary tree~\cite{cohen.leiserson.book90} will allow answering nearest neighbor
queries in $O(\log n)$ time using $O(n)$ space, in the worst case. When $d = 2$, similar logarithmic
time bounds can also be achieved by, for instance, computing a Voronoi
diagram~\cite{edelsbrunner.book87}. However, the time complexity of computing Voronoi diagrams is
known to be $O(n^{\lceil d / 2 \rceil})$ when $ d > 2$. Several algorithms that achieve a worst-case
query time that is logarithmic in the number of vertices and polynomial in the number of dimensions
were designed (see, e.g.,~\cite{clarkson.sjc88}). However, these algorithms use space that is
exponential in $d$, which may be impractical. Unfortunately, no exact nearest neighbor algorithm
that does not scale exponentially with the number of dimensions and runs queries in logarithmic time
using roughly linear space is known~\cite{arya.mount.ea.jacm99}.

Fortunately, computing an ``approximate'' nearest neighbor instead of an exact one is
computationally easier. In the sequel, a vertex $y$ is said to be an $\varepsilon$-approximate
nearest neighbor of a point $x$ if $\Vert y - x \Vert \le (1 + \varepsilon) \, \Vert z - x \Vert $,
where $z$ is the true nearest neighbor of $x$. An approximate nearest neighbor can be computed using
balanced-box decomposition (BBD) trees, which achieves $O(c_{d,\varepsilon} \log n)$ query time
using $O (d \, n)$ space~\cite{arya.mount.ea.jacm99}, where $c_{d, \varepsilon} \le d \lceil 1 +
6d/\varepsilon \rceil^d$. This algorithm is computationally optimal in fixed dimensions, since it
closely matches a lower bound for algorithms that use a tree structure stored in roughly linear
space~\cite{arya.mount.ea.jacm99}. Let us note that using approximate nearest neighbor computation
in the context of both PRMs and RRTs was also discussed very recently~\cite{plaku.kavraki.wafr08}.

Assuming that the ${\tt Nearest}$ procedure computes an approximate nearest neighbor using the
algorithm given in~\cite{arya.mount.ea.jacm99}, in fixed dimensions the ${\tt NearestNeighbor}$
algorithm has to run in $\Omega(\log n)$ time as formalized in the following lemma. Let ${\cal
  M}^\mathrm{RRT}_i$ be the random variable that denotes the number of steps taken by the RRT
algorithm in iteration $i$.
\begin{lemma} \label{lemma:rrtcomplexity} %
  Assuming that ${\tt Nearest}$ is implemented using the algorithm given
  in~\cite{arya.mount.ea.jacm99}, which is optimal in fixed dimensions, the number of steps executed
  by the RRT algorithm at each iteration is at least order $\log ({\cal N}_i)$ in expectation in the
  limit, i.e., there exists a constant $\phi_{RRT} \in \mathbb{R}_{>0}$ such that
  $$
  \liminf_{i \to \infty} \EE \left[ \frac{ {\cal M}^\mathrm{RRT}_i }{\log ({\cal N}_i)} \right] \ge
  \phi_{RRT}.
  $$
\end{lemma}

Likewise, problems similar to that solved by the ${\tt \NearNodes}$ procedure are also widely-studied
in the literature, generally under the name of {\em range search problems}, as they have many
applications in, for instance, computer graphics and spatial database systems~\cite{samet.book89a}.
In the worst case and in fixed dimensions, computing the exact set of vertices that reside in a ball of
radius $r_n$ centered at a query point $x$ takes $O(n^{1 - 1/d} + m)$ time using $k$-d
trees~\cite{lee.wong.acta_informatica77}, where $m$ is the number of vertices returned by the search.
In~\cite{chanzy.devroye.ea.acta_informatica01}, a thorough analysis of the average case performance
is also provided under certain conditions.

Similar to the nearest neighbor search, computing approximate solutions to the range search problem
is computationally easier. A range search algorithm is said to be $\varepsilon$-approximate if it
returns all vertices that reside in the ball of size $r_n$ and no vertices outside a ball of radius $(1 +
\varepsilon)\,r_n$, but may or may not return the vertices that lie outside the former ball and inside
the latter ball. Computing $\varepsilon$-approximate solutions using BBD-trees requires $O(2^d\log n
+ d^2(3\sqrt{d}/\varepsilon)^{d-1})$ time when using $O(d \, n)$ space, in the worst
case~\cite{arya.mount.comp_geo00}. Thus, in fixed dimensions, the complexity of this algorithm is
$O(\log n + (1/ \varepsilon)^{d-1})$, which is known to be optimal, closely matching a lower
bound~\cite{arya.mount.comp_geo00}. More recently, algorithms that can provide trade-offs between
time and space were also proposed~\cite{arya.malamatos.ea.symp_dis_alg05}.

Note that the ${\tt \NearNodes}$ procedure can be implemented as an approximate range search while
maintaining the asymptotic optimality guarantee. Notice that the expected number of vertices
returned by the ${\tt \NearNodes}$ procedure also does not change, except by a constant factor.
Hence, the ${\tt \NearNodes}$ procedure can be implemented to run in order $\log n$ expected time in
the limit\ and linear space in fixed dimensions.

Let ${\cal M}^\mathrm{RRG}_i$ denote the number of steps performed by the RRG algorithm in iteration
$i$. Then, together with Lemma~\ref{lemma:complexity_of_obstaclefree}, the discussion above implies
the following lemma.
\begin{lemma} \label{lemma:rrgcomplexity} %
  The number of steps executed by the RRG algorithm at each iteration is at most order $\log ({\cal
    N}_i)$ in expectation in the limit, i.e., there exists a constant $\phi_{RRG} \in
  \mathbb{R}_{>0}$ such that
  $$
  \limsup_{i \to \infty} \EE \left[ \frac{{\cal M}^\mathrm{RRG}_i}{ \log ({\cal N}_i)} \right] \le
  \phi_{RRG}.
  $$
\end{lemma}

Finally, by Lemmas~\ref{lemma:rrtcomplexity} and \ref{lemma:rrgcomplexity}, we conclude that the RRT
and the RRG algorithms have the same asymptotic computational complexity as stated in the following
theorem.
\begin{theorem} \label{theorem:complexitymain} %
  Under the assumptions of Lemmas~\ref{lemma:rrtcomplexity} and \ref{lemma:rrgcomplexity}, there
  exists a constant $\phi \in \mathbb{R}_{>0}$ such that
  $$
  \limsup_{i \to \infty} \EE \left[ \frac{{\cal M}^\mathrm{RRG}_i}{{\cal M}^\mathrm{RRT}_i} \right]
  \le \phi.
  $$
\end{theorem}

In section~\ref{section:simulations}, we also provide some experimental evidence supporting
Theorem~\ref{theorem:complexitymain}.

\subsection{On the efficient construction of Probabilistic RoadMaps}
At this point, let us note that our results also imply an efficient PRM planning algorithm. The PRM
algorithm samples a set of $n$ vertices from the state-space, and checks connectivity of each pair
of samples via the ${\tt Steer}$ and ${\tt ObstacleFree}$ procedures to build a graph (called the
roadmap) as follows. Any two vertices $x_1$ and $x_2$ are connected by an edge, if the path ${\tt
  Steer}(x_1,x_2)$ that connects these two samples lies within the obstacle-free space\footnote{In
  practice, some variations of the PRM algorithm attempt to connect every vertex to the vertices
  that lie within a fixed distance. However, this variation does not change the theoretical
  properties such as probabilistic completeness or the asymptotic computational complexity.}. It is
known that this algorithm is probabilistically complete and, under certain assumptions, the
probability of failure, if a solution exists, decays to zero exponentially.
However, notice that, from an asymptotic computational complexity point of view, creating this graph
requires $O(n^2)$ calls to the ${\tt Steer}$ and ${\tt ObstacleFree}$ procedures.
On the contrary, our results imply that in a modified version of this algorithm, if each vertex is
attempted connection to vertices within a ball of volume $\gamma_L \frac{\log n}{n}$, then,
theoretically speaking, probabilistic completeness can be achieved, while incurring $O(n \log n)$
expected computational cost. Pursuing this direction is outside the scope of this work, and is left
to future work.

\section{A Tree Version of the RRG Algorithm} \label{section:optrrt} %
Maintaining a tree structure rather than a graph may be advantageous in some applications, due to,
for instance, relatively easy extensions to motion planning problems with differential constraints,
or to cope with modeling errors. The RRG algorithm can also be slightly modified to maintain a tree
structure, while preserving the asymptotic optimality properties as well the computational
efficiency. In this section a tree version of the RRG algorithm, called RRT$^*$, is introduced and
analyzed.

\subsection{The RRT$^*$ Algorithm}

Given two points $x, x' \in X_\mathrm{free}$, recall that $ {\tt Line}(x,x') : [0,s] \to
X_\mathrm{free}$ denotes the path defined by $\sigma(\tau) = \tau x + (s - \tau) x'$ for all $\tau
\in [0,s]$, where $s = \Vert x' - x \Vert$. Given a tree $G = (V,E)$ and a vertex $v \in V$, let
${\tt Parent}$ be a function that maps $v$ to the unique vertex $v' \in V$ such that $(v',v) \in E$.

The RRT$^*$ algorithm differs from the RRT and the RRG algorithms only in the way that it handles
the ${\tt Extend}$ procedure. The body of the RRT$^*$ algorithm is presented in
Algorithm~\ref{algorithm:motionplanning} and the ${\tt Extend}$ procedure for the RRT$^*$ is given
in Algorithm~\ref{algorithm:optrrt}. In the description of the RRT$^*$ algorithm, the cost of the
unique path from $x_\mathrm{init}$ to a vertex $v \in V$ is denoted by ${\tt Cost}(v)$. Initially,
${\tt Cost} (x_\mathrm{init})$ is set to zero.

\begin{algorithm}
  $V' \leftarrow V$; $E' \leftarrow E$\;
  $x_\mathrm{nearest} \leftarrow {\tt Nearest} (G, x)$\; \label{rrtstar:extend_nearest_start}
  $x_\mathrm{new} \leftarrow {\tt Steer} (x_\mathrm{nearest}, x)$\;    \label{rrtstar:extend_nearest_end}
  \If {${\tt ObstacleFree} (x_\mathrm{nearest}, x_\mathrm{new})$} {
    $V' \leftarrow V' \cup \{x_\mathrm{new}\}$\;
    $x_\mathrm{min} \leftarrow x_\mathrm{nearest}$\;
    \label{rrtstar:connect_to_close_node_start}
    $X_\mathrm{near} \leftarrow {\tt \NearNodes} (G, x_\mathrm{new}, \vert V \vert)$\; 
    \For {all $x_\mathrm{near} \in X_\mathrm{near}$} {
      \If {${\tt ObstacleFree} (x_\mathrm{near}, x_\mathrm{new})$} {
        $c' \leftarrow {\tt Cost} (x_\mathrm{near}) + c({\tt Line} (x_\mathrm{near}, x_\mathrm{new}))$\;
        \If {$c' < {\tt Cost} (x_\mathrm{new})$} {
          $x_\mathrm{min} \leftarrow x_\mathrm{near}$\;
        }
      }
    }
    $E' \leftarrow E' \cup \{(x_\mathrm{min}, x_\mathrm{new})\}$\;
    \label{rrtstar:connect_to_close_node_end}
    \For {all $x_{near} \in X_\mathrm{near} \setminus \{x_\mathrm{min} \}$
    \label{rrtstar:extend_back_to_tree_start}}{ 
      \If{${\tt ObstacleFree}(x_\mathrm{new}, x_\mathrm{near})$ and ${\tt Cost}(x_\mathrm{near}) >
        {\tt Cost}(x_\mathrm{new}) + c ({\tt Line}(x_\mathrm{new}, x_\mathrm{near}))$}{
        $x_\mathrm{parent} \leftarrow {\tt Parent} (x_\mathrm{near})$\;
        $E' \leftarrow E' \setminus \{ (x_\mathrm{parent}, x_\mathrm{near})\}$;
        $E' \leftarrow E' \cup \{(x_\mathrm{new}, x_\mathrm{near}) \}$\;  
        \label{rrtstar:extend_back_to_tree_end}
      } 
    } 
  }
  \Return{$G' = (V', E')$}
  \caption{${\tt Extend}_{RRT^*}$}
  \label{algorithm:optrrt}
\end{algorithm}

Similarly to the RRT and RRG, the RRT$^*$ algorithm first extends the nearest neighbor towards the
sample (Lines~\ref{rrtstar:extend_nearest_start}-\ref{rrtstar:extend_nearest_end}). However, it
connects the new vertex, $x_\mathrm{new}$, to the vertex that incurs the minimum accumulated cost up
until $x_\mathrm{new}$ and lies within the set $X_\mathrm{near}$ of vertices returned by the ${\tt
  \NearNodes}$ procedure
(Lines~\ref{rrtstar:connect_to_close_node_start}-\ref{rrtstar:connect_to_close_node_end}). RRT$^*$
also extends the new vertex to the vertices in $X_\mathrm{near}$ in order to ``rewire'' the vertices that
can be accessed through $x_\mathrm{new}$ with smaller cost
(Lines~\ref{rrtstar:extend_back_to_tree_start}-\ref{rrtstar:extend_back_to_tree_end}).

\subsection{Convergence to a Feasible Solution}

For all $i \in \mathbb{N}$, let ${\cal V}^{\mathrm{RRT}^*}_i$ and ${\cal E}^{\mathrm{RRT}^*}_i$
denote the set of vertices and the set of edges of the graph maintained by the RRT$^*$ algorithm, at
the end of iteration $i$. Let ${\cal V}^{\mathrm{RRT}^*}_0 (\omega) = \{x_\mathrm{init}\}$ and
${\cal E}^{\mathrm{RRT}^*}_0 (\omega) = \emptyset$ for all $\omega \in \Omega$.

The following lemma is the equivalent of Lemma~\ref{lemma:vertexedge}.
\begin{lemma} \label{lemma:optrrtvertexedgeequivalence} For all $i \in \mathbb{N}$ and all $\omega
  \in \Omega$, ${\cal V}^{\mathrm{RRT}^*}_i (\omega) = {\cal V}^\mathrm{RRG}_i (\omega)$, and ${\cal
    E}^{\mathrm{RRT}^*}_i (\omega) \subseteq {\cal E}^\mathrm{RRG}_i (\omega)$.
\end{lemma}
From Lemma~\ref{lemma:optrrtvertexedgeequivalence} and Theorem~\ref{theorem:rrtrrgequivalency}, the
following theorem, which asserts the probabilistic completeness and the exponential decay of failure
probability of the RRT$^*$ algorithm, is immediate.
\begin{theorem}
  If there exists a feasible solution to Problem~\ref{problem:feasibility}, then $\lim_{i \to
    \infty} \PP \left( \left\{ {\cal V}^{\mathrm{RRT}^*}_i \cap X_\mathrm{goal} \neq \emptyset
    \right\} \right) = 1$. Moreover, if an attraction sequence ${\cal A}$ of length $k$ exists, then
  $\PP \left( \left\{ {\cal V}^{\mathrm{RRT}^*}_i \cap X_\mathrm{goal} = \emptyset \right\} \right)
  \le e^{-\frac{1}{2}(i\,p_k - 2\,k)}$.
\end{theorem}

\subsection{Asymptotic Optimality}

Let ${\cal Y}^{\mathrm{RRT}^*}_i$ be a random variable that denotes the cost of a minimum cost path in the
tree maintained by the RRT$^*$ algorithm, at the end of iteration $i$. The following theorem ensures
the asymptotic optimality of the RRT$^*$ algorithm.
\begin{theorem}
  Let Assumptions~\ref{assumption:additivecost}, \ref{assumption:continuouscost}, and
  \ref{assumption:obstacles} hold. Then, the cost of the minimum cost path in the RRT$^*$ converges
  to $c^*$ almost surely, i.e., $\PP \left(\{\lim_{i \to \infty} {\cal Y}^{\mathrm{RRT}^*}_i =
    c^*\}\right) = 1$.
\end{theorem}

\subsection{Computational Complexity}

Let ${\cal M}^{\mathrm{RRT}^*}_i$ be the number of steps performed by the RRT$^*$ algorithm in iteration
$i$. The following theorem follows from Lemma~\ref{lemma:optrrtvertexedgeequivalence} and
Theorem~\ref{theorem:complexitymain}.
\begin{theorem}  %
  Under the assumptions of Theorem~\ref{theorem:complexitymain}, there exists a constant $\phi \in
  \mathbb{R}_{>0}$ such that
  $$
  \limsup_{i \to \infty} \EE \left[ \frac{{\cal M}^{\mathrm{RRT}^*}_i}{{\cal M}^\mathrm{RRT}_i}
  \right] \le \phi.
  $$
\end{theorem}

\section{Simulations} \label{section:simulations}

This section is devoted to an experimental study of the algorithms. Three different problem
instances are considered and the RRT and RRT$^*$ algorithms are compared with respect to their running
time and cost of the solution achieved.
Both algorithms were implemented in C and run on a computer with 2.66 GHz processor and 4GB RAM
running the Linux operating system.

We consider three problem instances. In the first two, the cost function is the Euclidean path
length. The first scenario includes no obstacles. Both algorithms are run in a square environment.
The trees maintained by the algorithms are shown in Figure~\ref{figure:sim1} at several stages. The
figure illustrates that, in this case, the RRT algorithm does not improve the feasible solution to
converge to an optimum solution. On the other hand, running the RRT$^*$ algorithm further improves
the paths in the tree to lower cost ones. The convergence properties of the two algorithms are also
investigated in Monte-Carlo runs. Both algorithms were run for 20,000 iterations 500 times and the
cost of the best path in the trees were averaged for each iteration. The results are shown in
Figure~\ref{figure:sim1cost}, which shows that in the limit the RRT algorithm has cost very close to
a $\sqrt{2}$ factor the optimal solution (see~\cite{lavalle.kuffner.tech_rep09} for a similar result
in a deterministic setting), whereas the RRT$^*$ converges to the optimal solution. Moreover, the
variance over different RRT runs approaches 2.5, while that of the RRT$^*$ approaches zero. Hence,
almost all RRT$^*$ runs have the property of convergence to an optimal solution, as expected.

In the second scenario, both algorithms are run in an environment in presence of obstacles. In
Figure~\ref{figure:sim2}, the trees maintained by the algorithms are shown after 20,000 iterations.
The tree maintained by the RRT$^*$ algorithm is also shown in Figure~\ref{figure:sim2optrrt} in
different stages. It can be observed that the RRT$^*$ first rapidly explores the state space just
like the RRT. Moreover, as the number of samples increase, the RRT$^*$ improves its tree to include
paths with smaller cost and eventually discovers a path in a different homotopy class, which reduces
the cost of reaching the target considerably. Results of a Monte-Carlo study for this scenario is
presented in Figure~\ref{figure:sim2cost}. Both algorithms were run alongside up until 20,000
iterations 500 times and cost of the best path in the trees were averaged for each iteration. The
figures illustrate that all runs of the RRT$^*$ algorithm converges to the optimum, whereas the RRT
algorithm is about 1.5 of the optimal solution on average. The high variance in solutions returned
by the RRT algorithm stems from the fact that there are two different homotopy classes of paths that
reach the goal. If the RRT luckily converges to a path of the homotopy class that contains an
optimum solution, then the resulting path is relatively closer to the optimum than it is on average.
If, on the other hand, the RRT first explores a path of the second homotopy class, which is often
the case for this particular scenario, then the solution that RRT converges is generally around
twice the optimum.

Finally, in the third scenario, where no obstacles are present, the cost function is selected to be
the line integral of a function, which evaluates to 2 in the high cost region, 1/2 in the low cost
region, and 1 everywhere else. The tree maintained by the RRT$^*$ algorithm is shown after 20,000
iterations in Figure~\ref{figure:sim3}. Notice that the tree either avoids the high cost region or
crosses it quickly, and vice-versa for the low-cost region. (Incidentally, this behavior corresponds
to the well known Snell-Descartes law for refraction of light, see~\cite{Rowe.Alexander.00} for a
path-planning application.)

To compare the running time, both algorithms were run alongside in an environment with no obstacles
up until one million iterations. Figure~\ref{figure:sim0time}, shows the ratio of the running time
of RRT$^*$ and that of RRT versus the number of iterations averaged over 50 runs. As expected from
the complexity analysis of Section~\ref{section:complexity}, this ratio converges to a constant
value. The same figure is produced for the second scenario and provided in
Figure~\ref{figure:sim1time}.

\begin{figure*}[ht]
  \begin{center}
    \mbox{ \subfigure[]{\scalebox{0.25}{\includegraphics[trim = 3.1cm 1.8cm 2.5cm 1.4cm, clip =
          true]{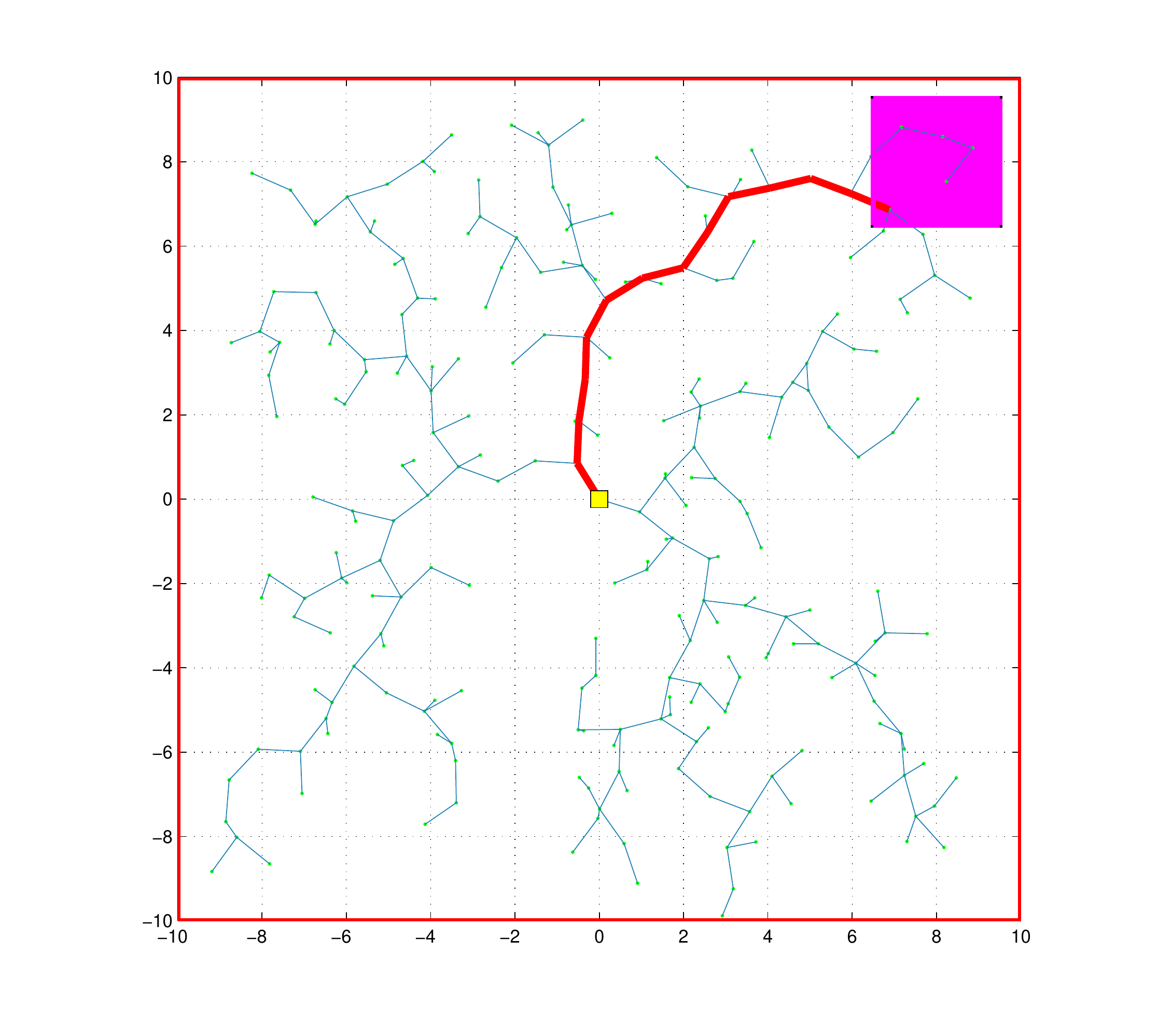}} 
        \label{sim1_rrt_10}}
      \subfigure[ ]{\scalebox{0.25}{\includegraphics[trim = 3.1cm 1.8cm 2.5cm 1.4cm, clip =
          true]{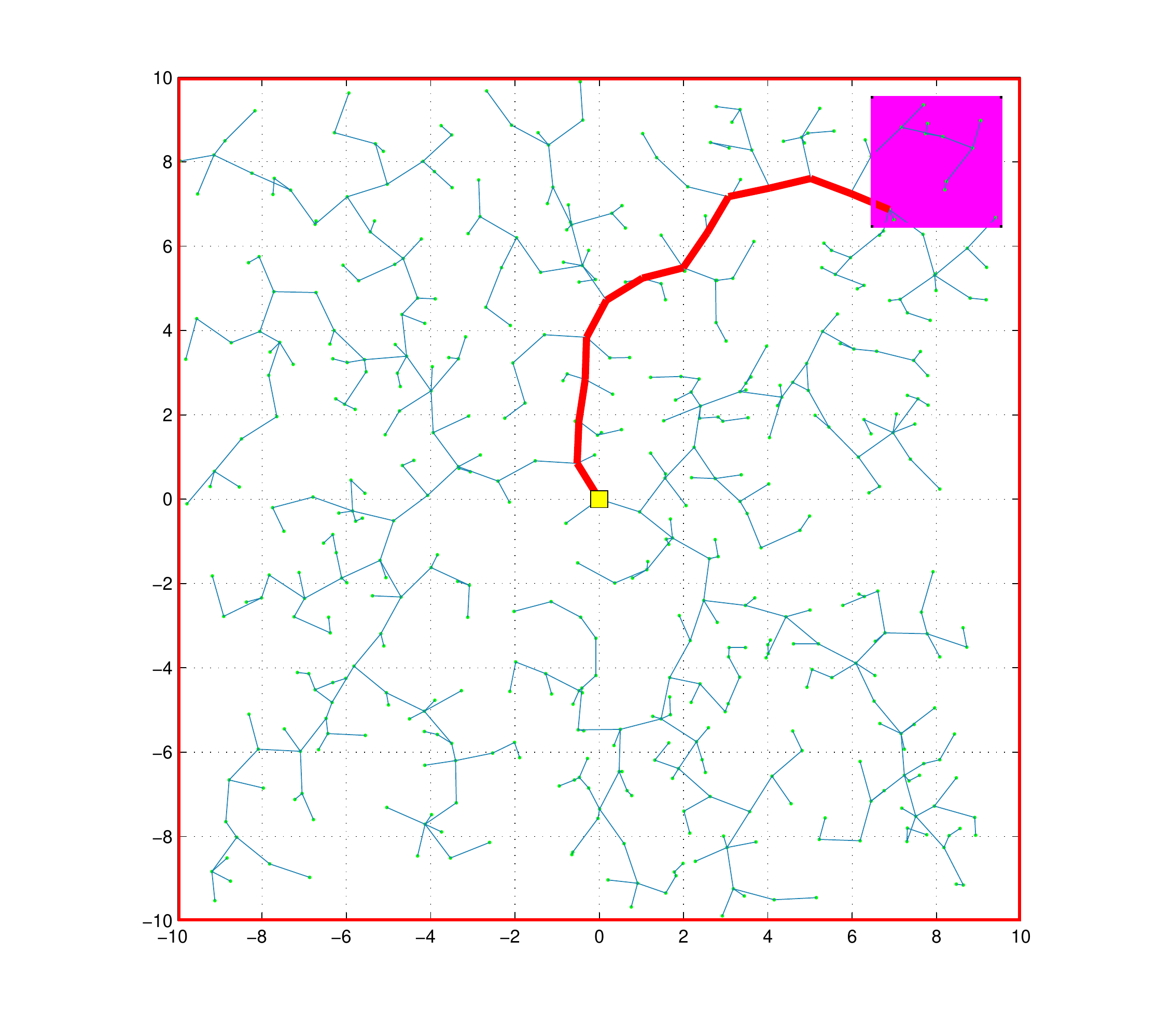}} 
        \label{sim1_optrrt_10}}
      \subfigure[ ]{\scalebox{0.25}{\includegraphics[trim = 3.1cm 1.8cm 2.5cm 1.4cm, clip =
          true]{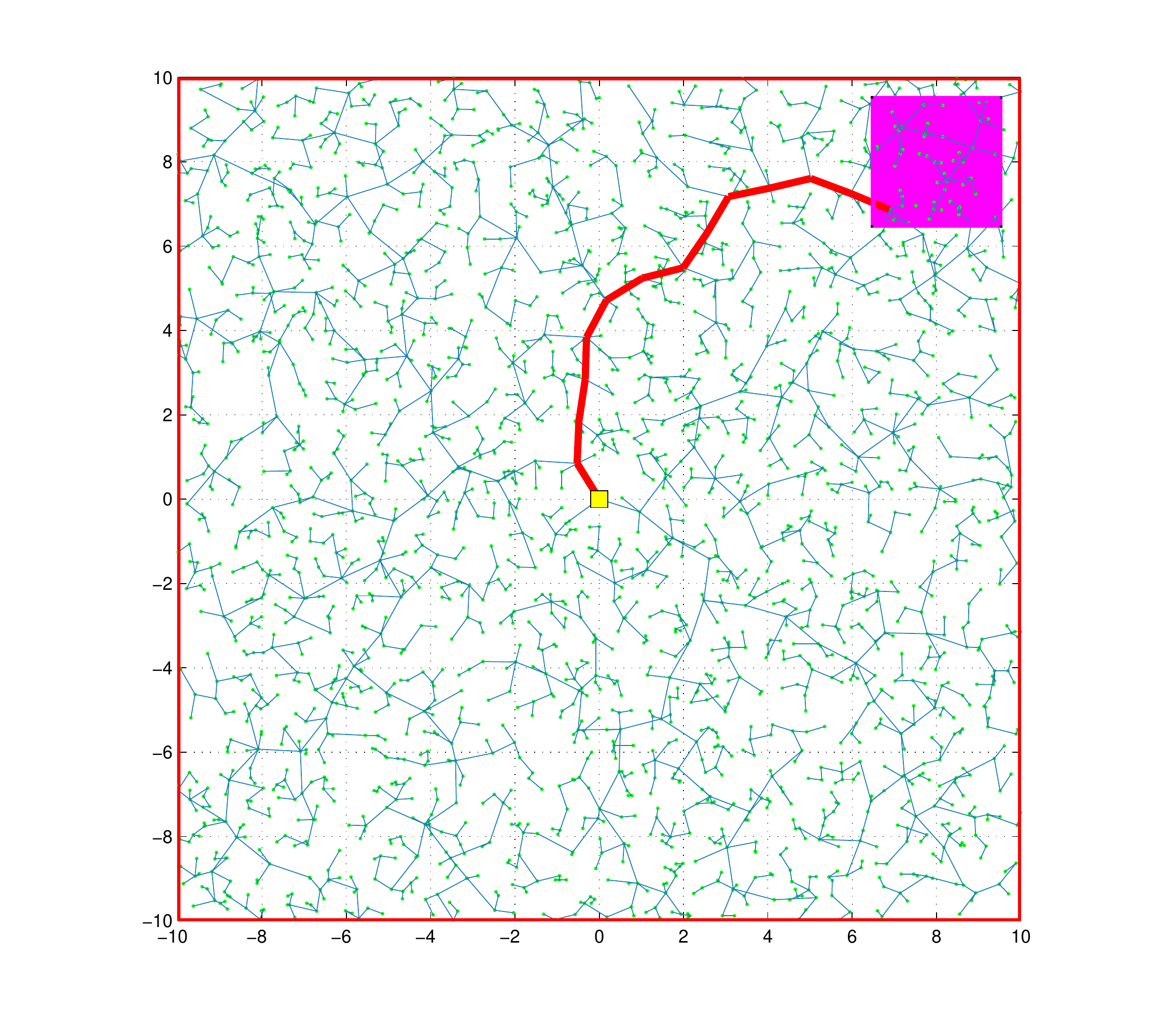}} 
        \label{sim1_optrrt_10}}
      \subfigure[ ]{\scalebox{0.25}{\includegraphics[trim = 3.1cm 1.8cm 2.5cm 1.4cm, clip =
          true]{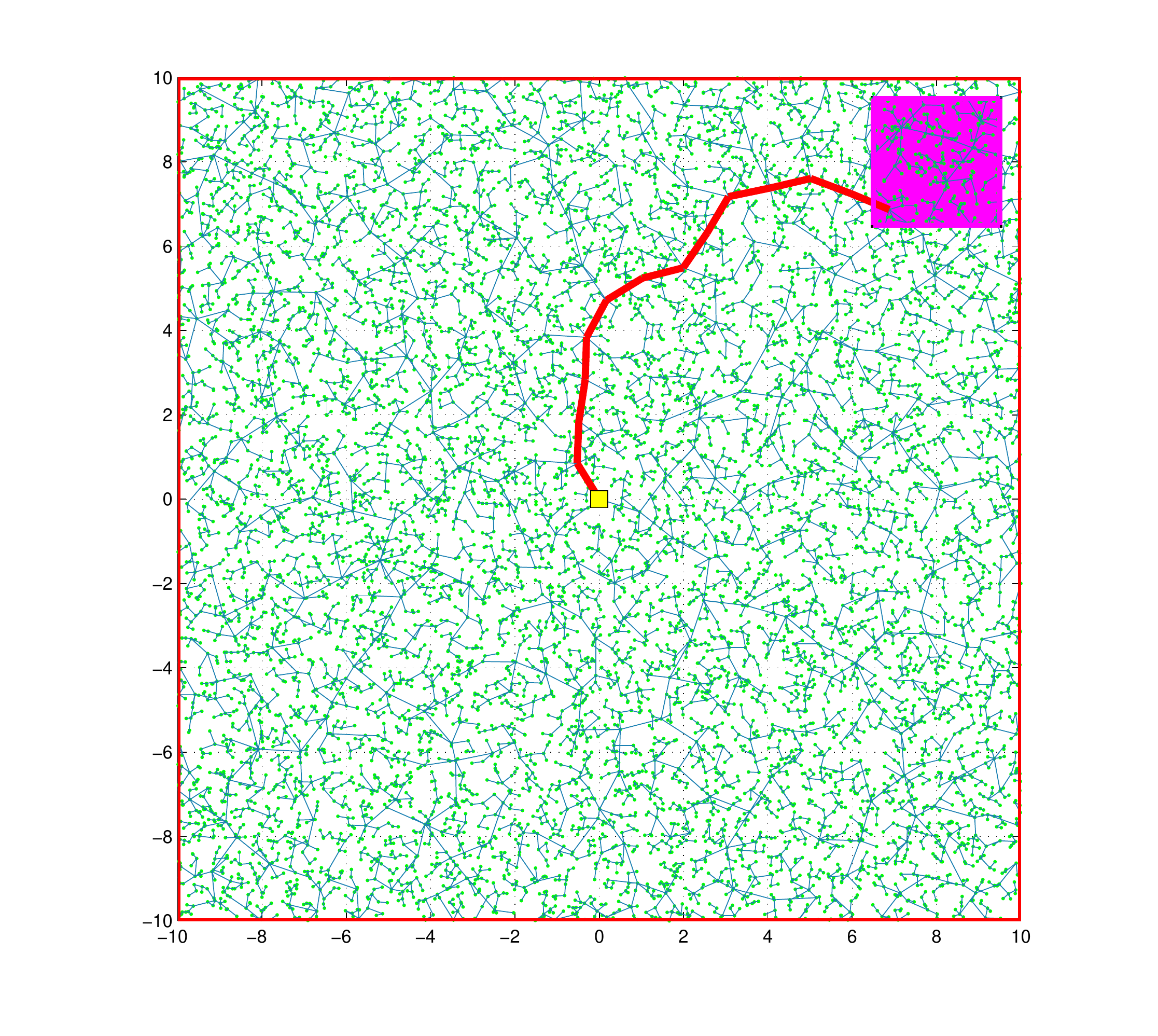}} 
        \label{sim1_optrrt_10}}}
    \mbox{ \subfigure[]{\scalebox{0.25}{\includegraphics[trim = 3.1cm 1.8cm 2.5cm 1.4cm, clip =
          true]{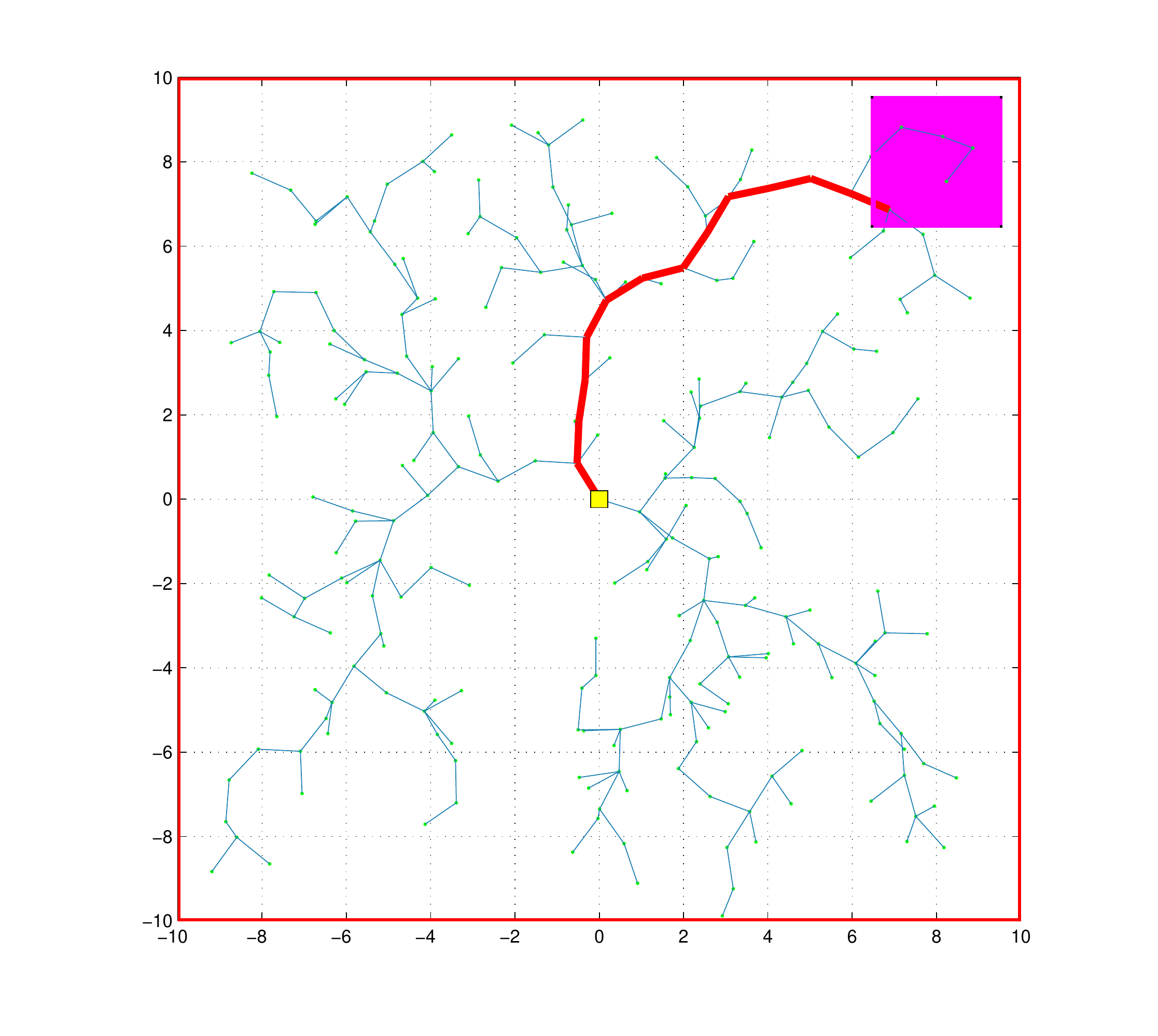}}
        \label{sim1_rrt_10}} \subfigure[ ]{\scalebox{0.25}{\includegraphics[trim = 3.1cm 1.8cm
          2.5cm 1.4cm, clip =
          true]{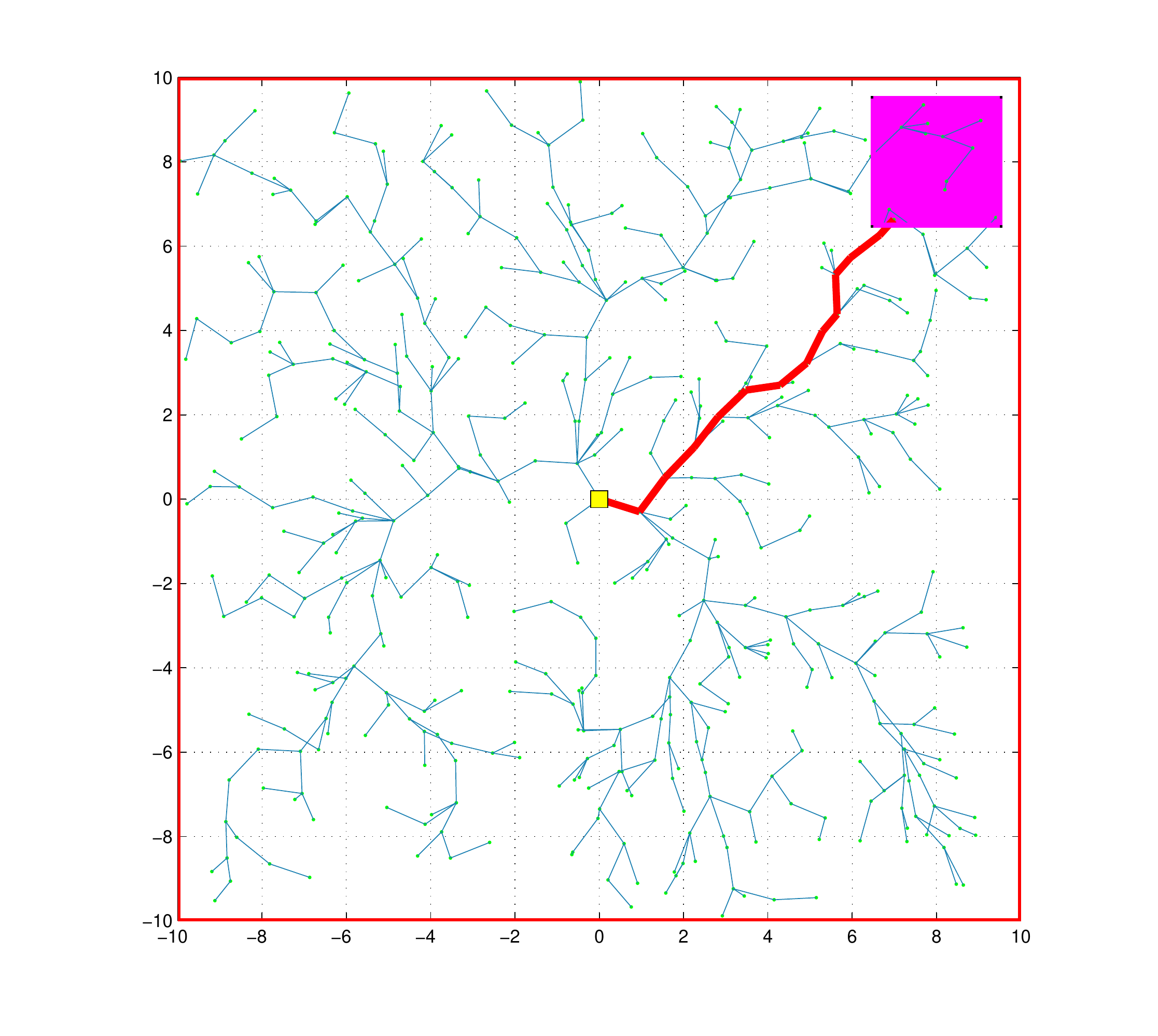}}
        \label{sim1_optrrt_10}} \subfigure[ ]{\scalebox{0.25}{\includegraphics[trim = 3.1cm 1.8cm
          2.5cm 1.4cm, clip =
          true]{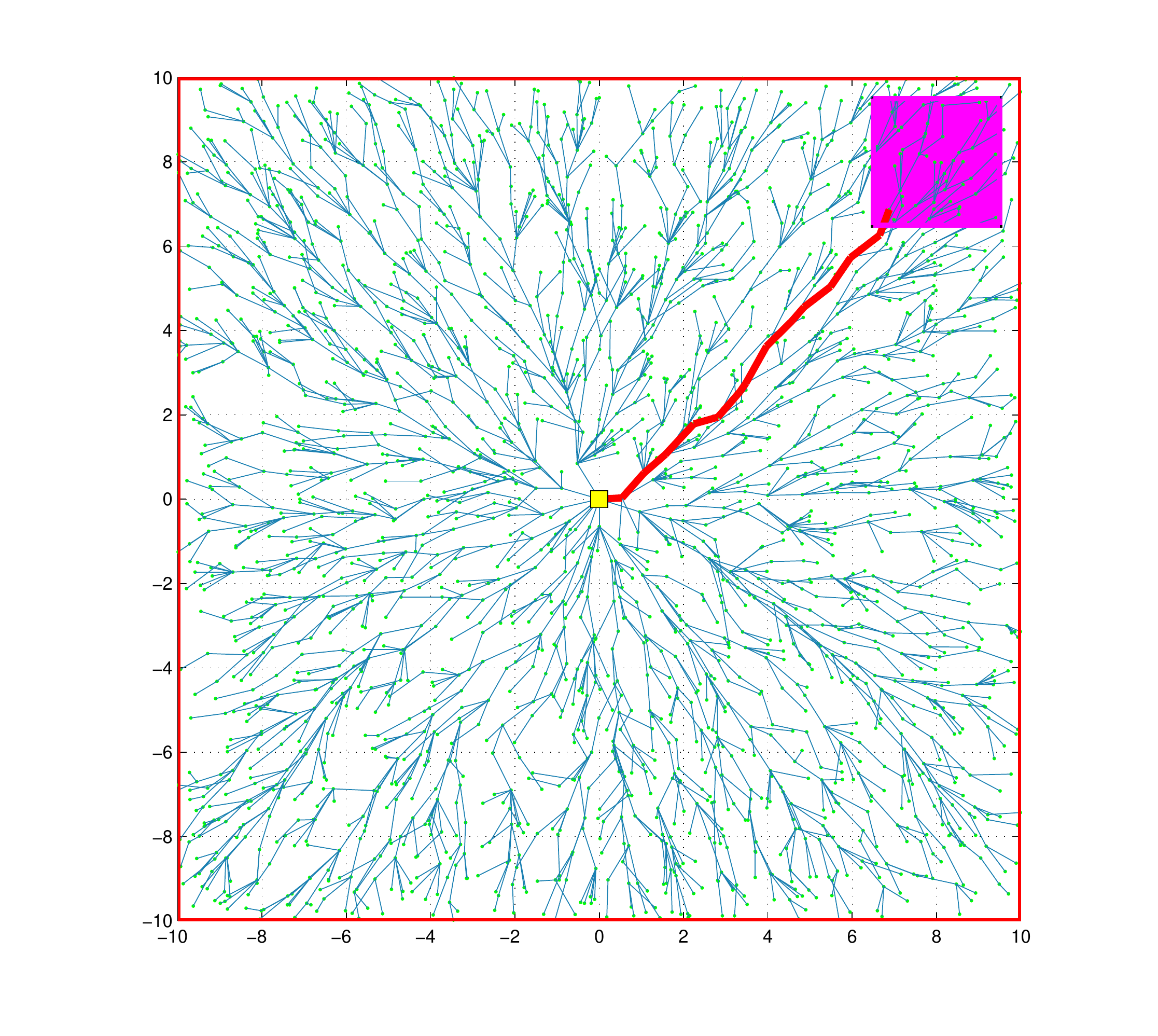}}
        \label{sim1_optrrt_10}} \subfigure[ ]{\scalebox{0.25}{\includegraphics[trim = 3.1cm 1.8cm
          2.5cm 1.4cm, clip =
          true]{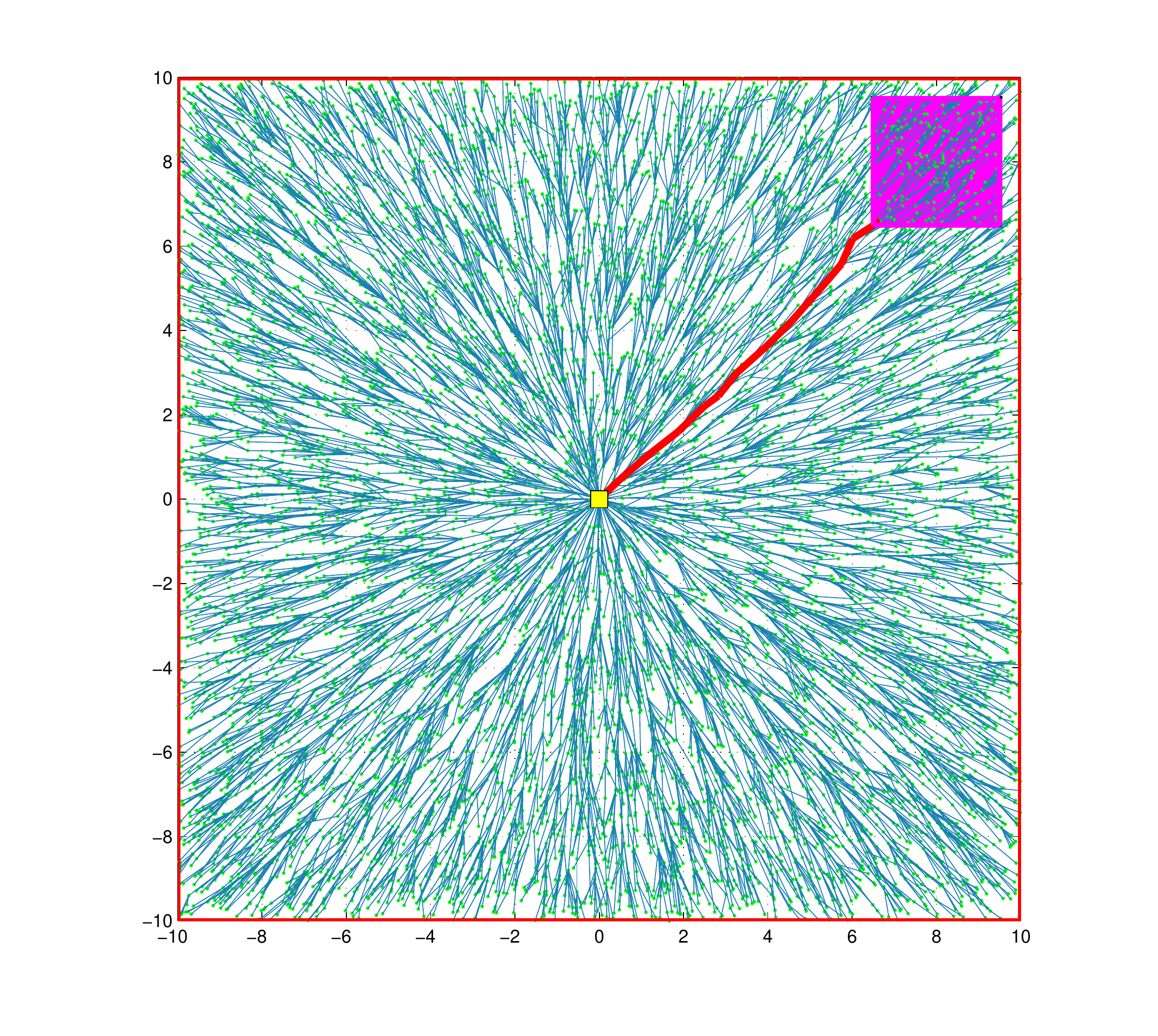}}
        \label{sim1_optrrt_10}}} \mbox{ \subfigure[]{\scalebox{0.51}{\includegraphics[trim = 3.1cm
          1.8cm 2.5cm 1.4cm, clip =
          true]{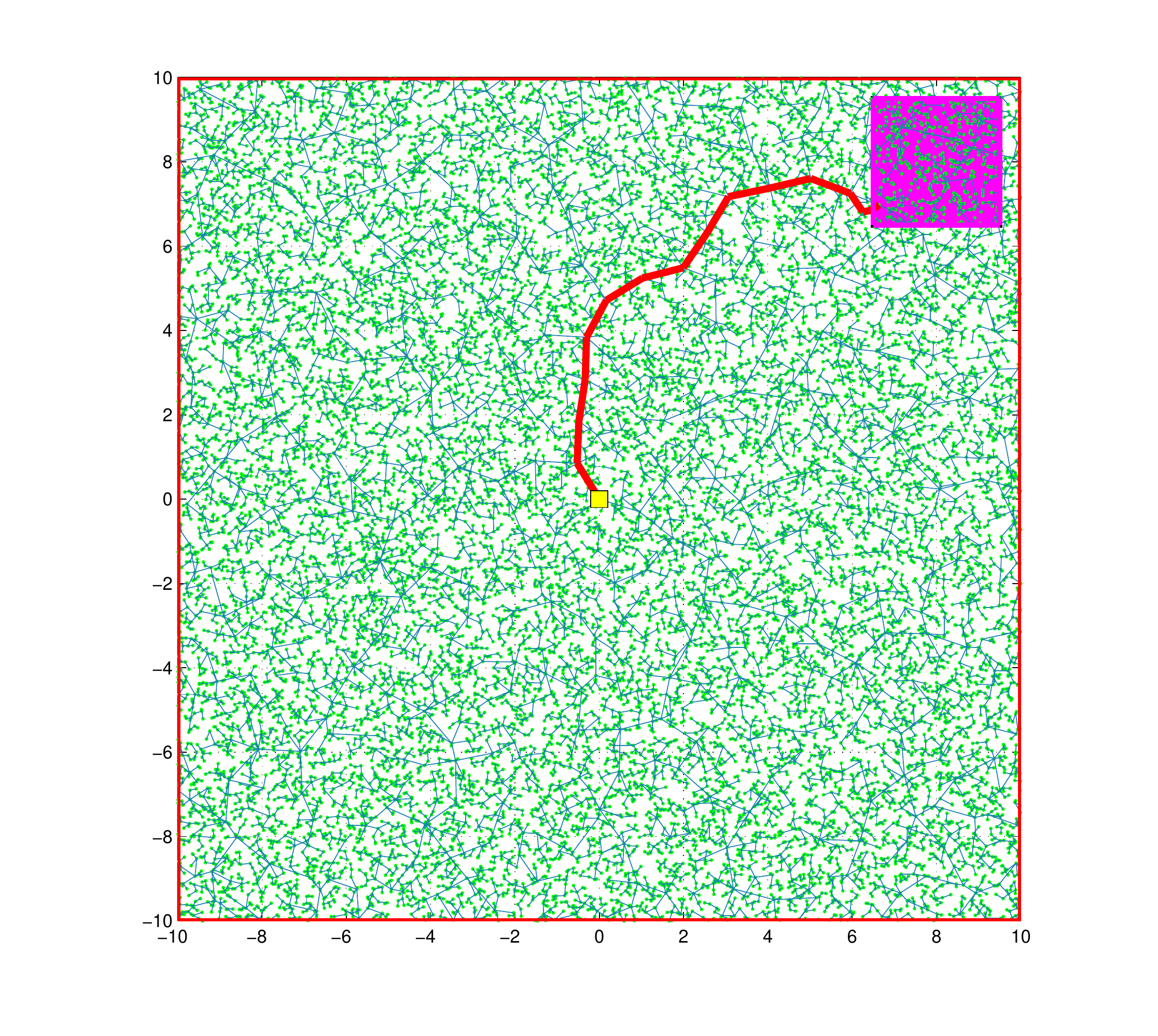}} \label{sim1_rrt_20}}
      \subfigure[ ]{\scalebox{0.51}{\includegraphics[trim = 3.1cm 1.8cm 2.5cm 1.4cm, clip =
          true]{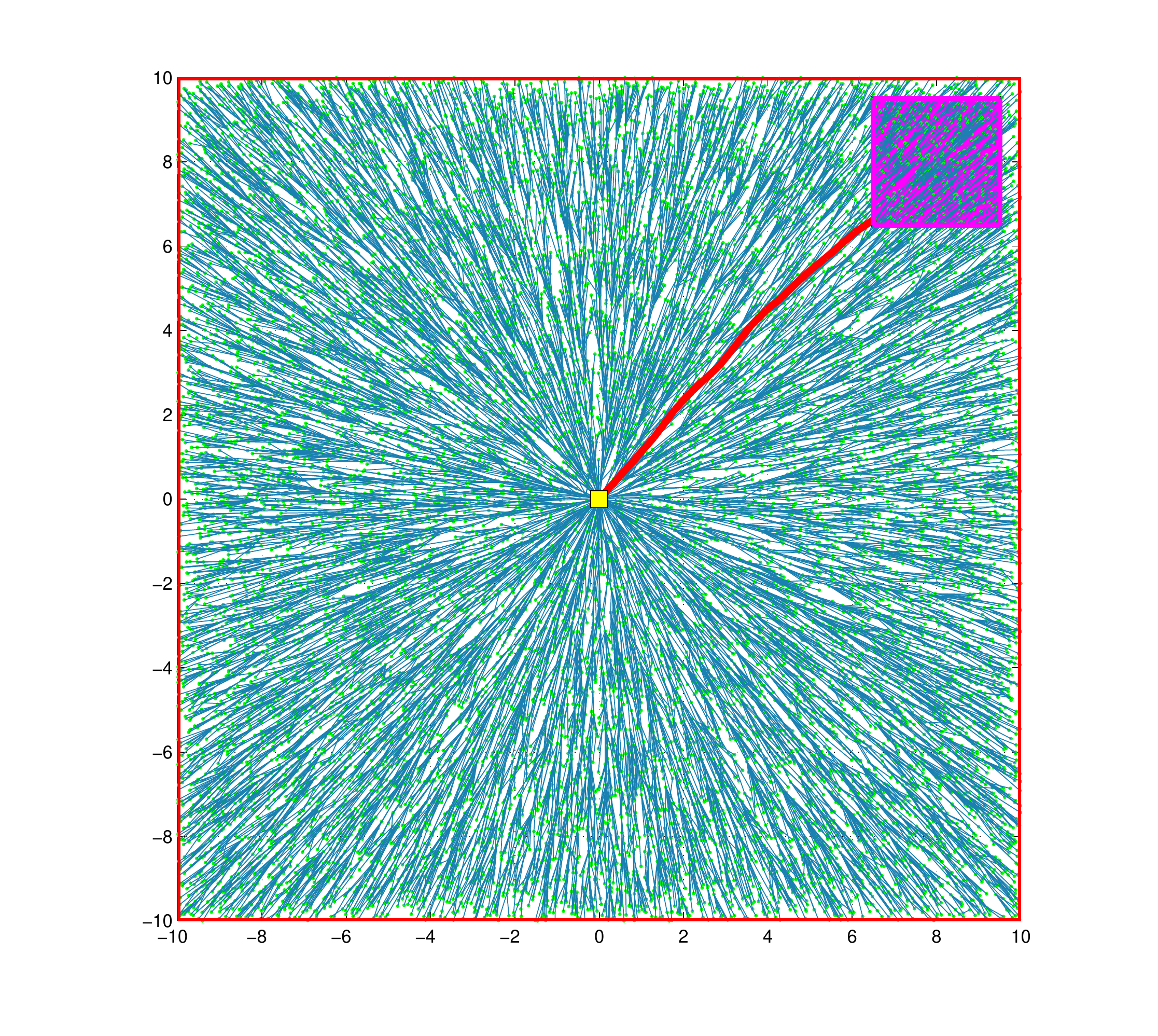}} 
        \label{sim1_optrrt_20}}
    }
    \caption{A Comparison of the RRT$^*$ and RRT algorithms on a simulation example with no
      obstacles. Both algorithms were run with the same sample sequence. Consequently, in this case,
      the vertices of the trees at a given iteration number are the same for both of the algorithms;
      only the edges differ. The edges formed by the RRT$^*$ algorithm are shown in (a)-(d) and (j),
      whereas those formed by the RRT algorithm are shown in (e)-(h) and (i). The tree snapshots
      (a), (e) contain 250 vertices, (b), (f) 500 vertices, (c), (g) 2500 vertices, (d), (h) 10,000 vertices and
      (i), (j) 20,000 vertices. The goal regions are shown in magenta (in upper right). The best paths
      that reach the target in all the trees are highlighted with red.}
    \label{figure:sim1}
  \end{center}
\end{figure*}

\begin{figure}[ht]
  \begin{center}
    \mbox{ \subfigure[]{\scalebox{0.4}{\includegraphics[trim = 2.25cm 0.75cm 1.55cm 0.4cm, clip =
          true]{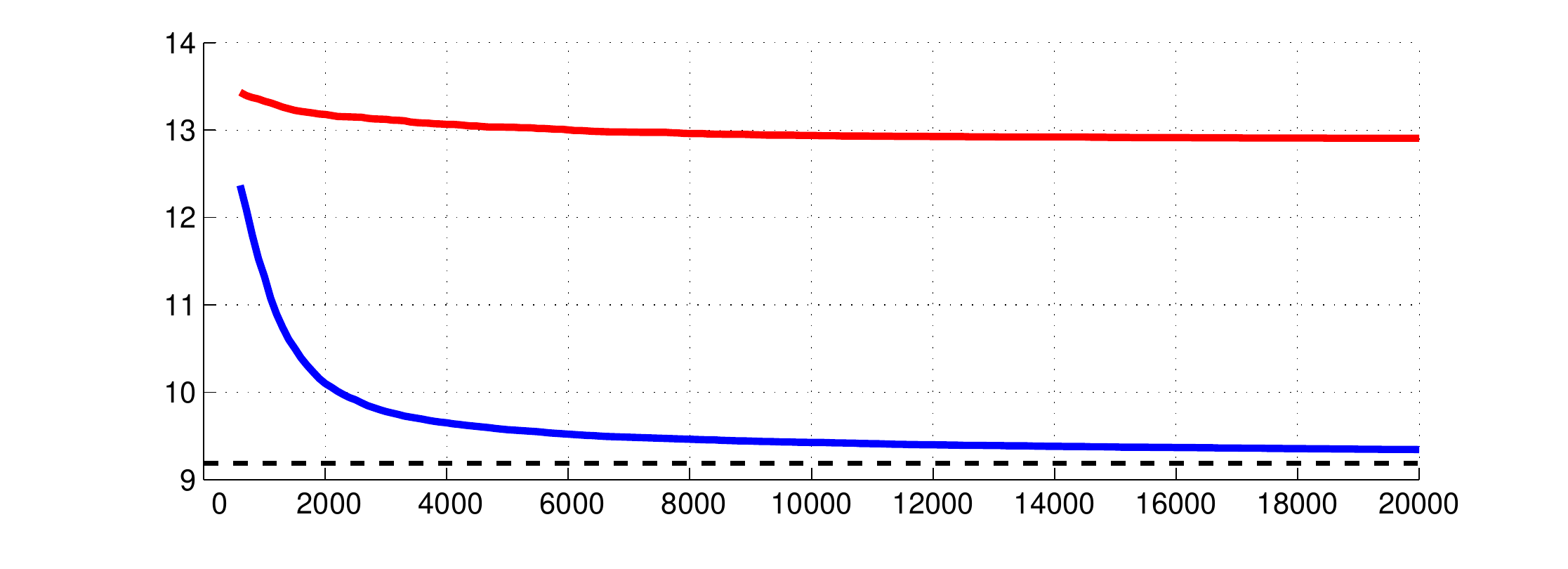}} \label{sim0_mc_cost}}}
    \mbox{ \subfigure[ ]{\scalebox{0.4}{\includegraphics[trim = 2.25cm 0.75cm 1.55cm 0.4cm, clip
          = true
          ]{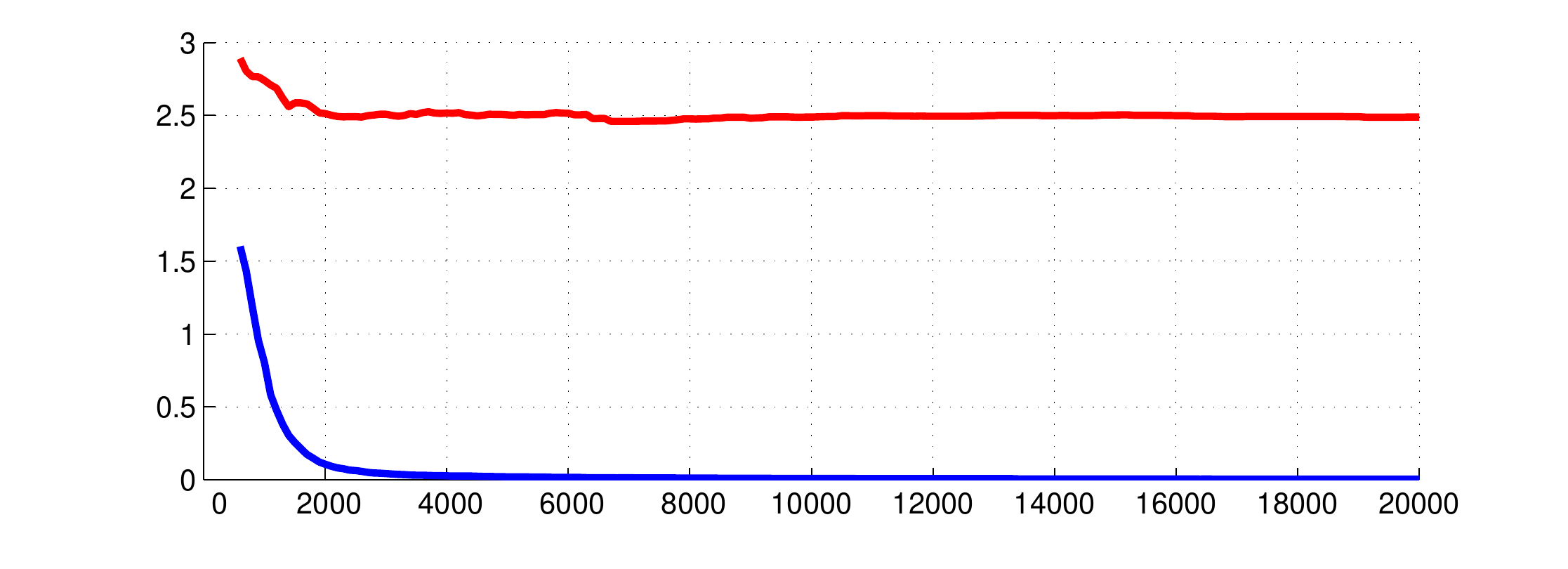}} \label{sim0_mc_cost_var}}}
    \caption{The cost of the best paths in the RRT (shown in red) and the RRT$^*$ (shown in blue)
      plotted against iterations averaged over 500 trials in (a). The optimal cost is shown in
      black. The variance of the trials is shown in (b).}
    \label{figure:sim1cost}
  \end{center}
\end{figure}

\begin{figure*}[ht]
  \begin{center}
    \mbox{ \subfigure[]{\scalebox{0.65}{\includegraphics[trim = 4.55cm 7.6cm 4.05cm 7.35cm, clip =
          true]{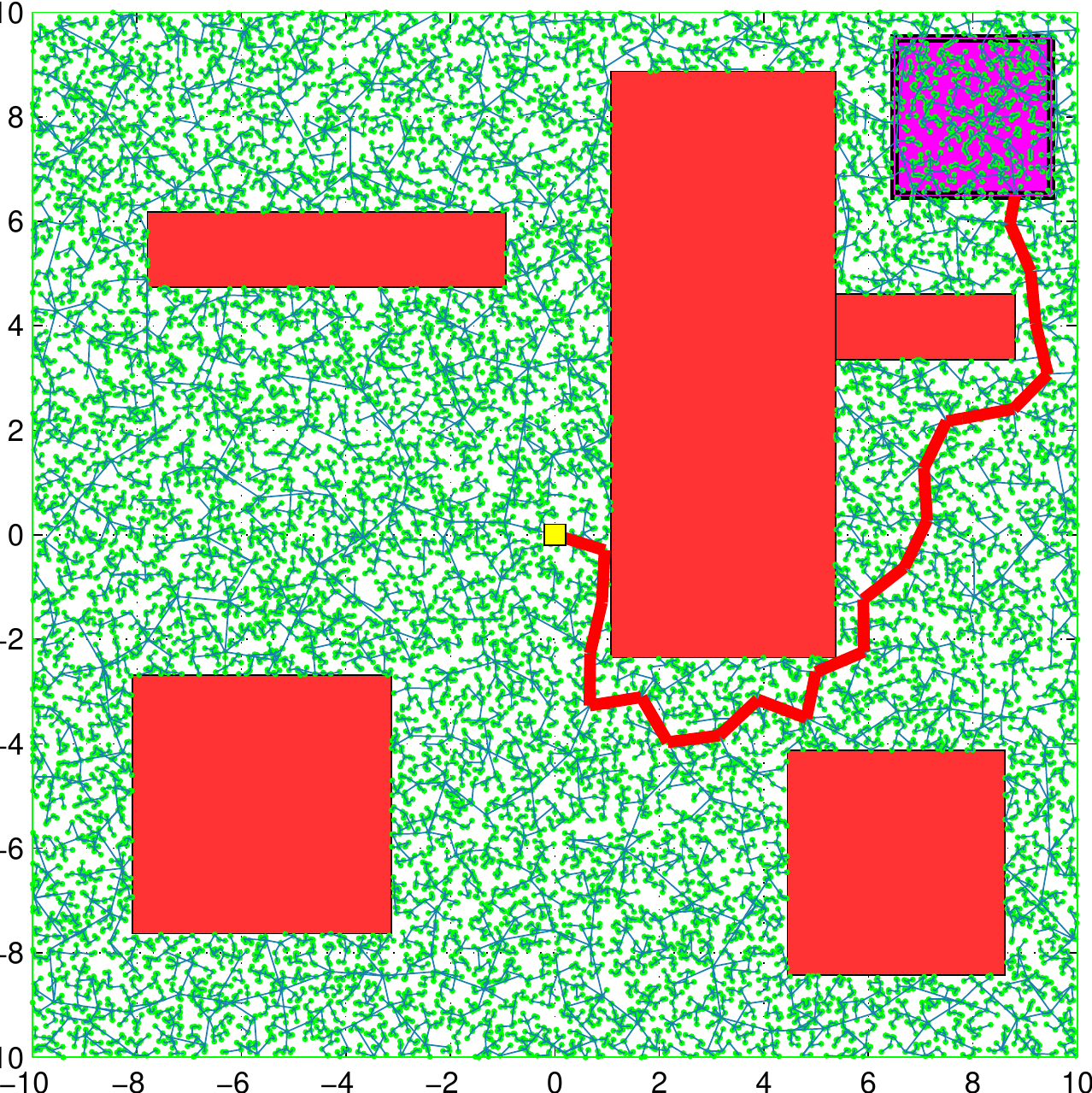}} \label{sim2_rrt_10}} \subfigure[
      ]{\scalebox{0.65}{\includegraphics[trim = 4.55cm 7.6cm 4.05cm 7.35cm, clip =
          true]{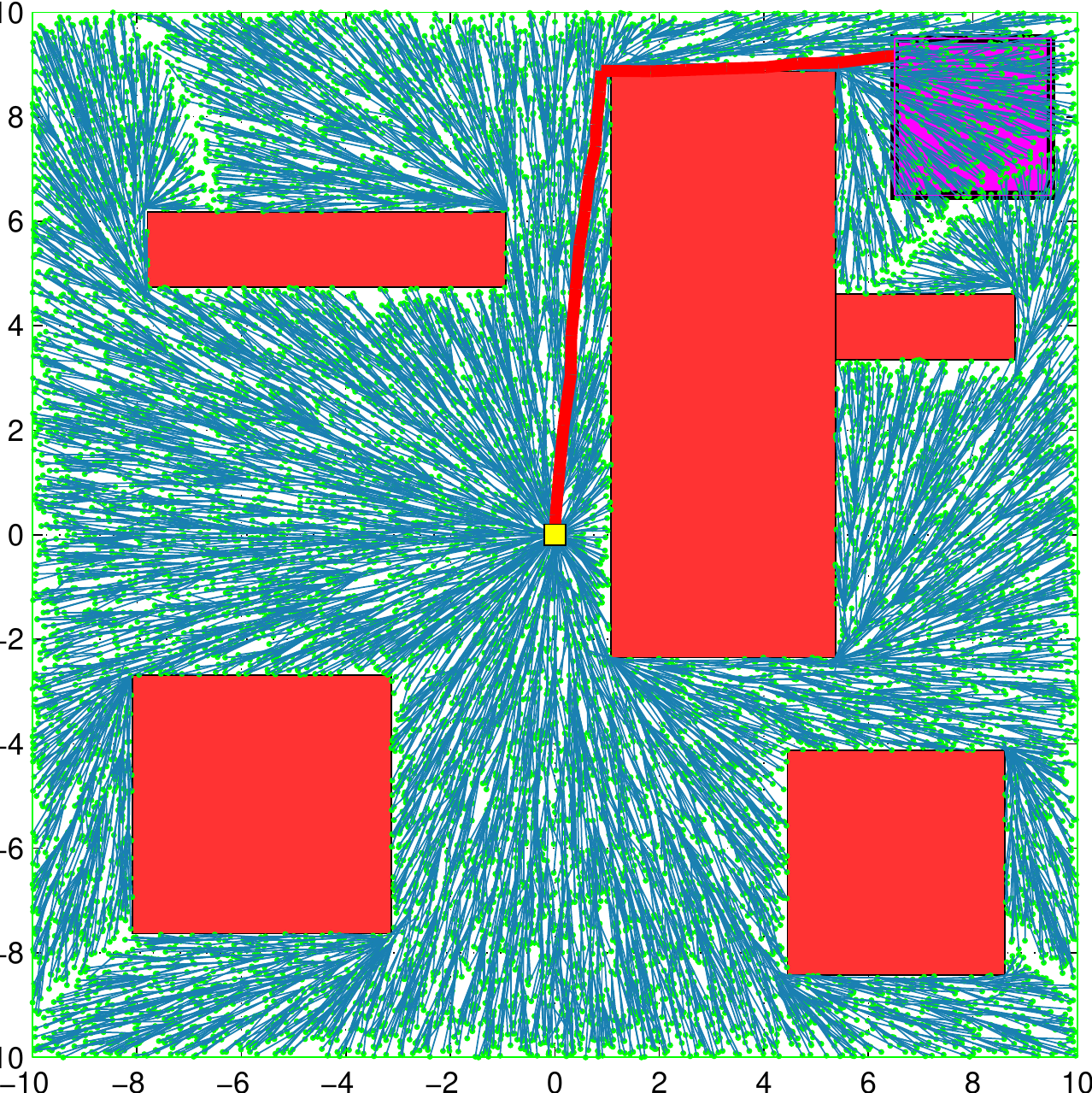}} \label{sim2_optrrt_10}} }
    \caption{A Comparison of the RRT (shown in (a)) and RRT$^*$ (shown in (b)) algorithms on a
      simulation example with obstacles. Both algorithms were run with the same sample sequence
      for 20,000 samples. The cost of best path in the RRT and the RRG were 21.02 and 14.51,
      respectively.}
    \label{figure:sim2}
  \end{center}
\end{figure*}

\begin{figure*}[ht]
  \begin{center}
    \mbox{ \subfigure[]{\scalebox{0.45}{\includegraphics[trim = 4.55cm 7.6cm 4.05cm 7.35cm, clip =
          true]{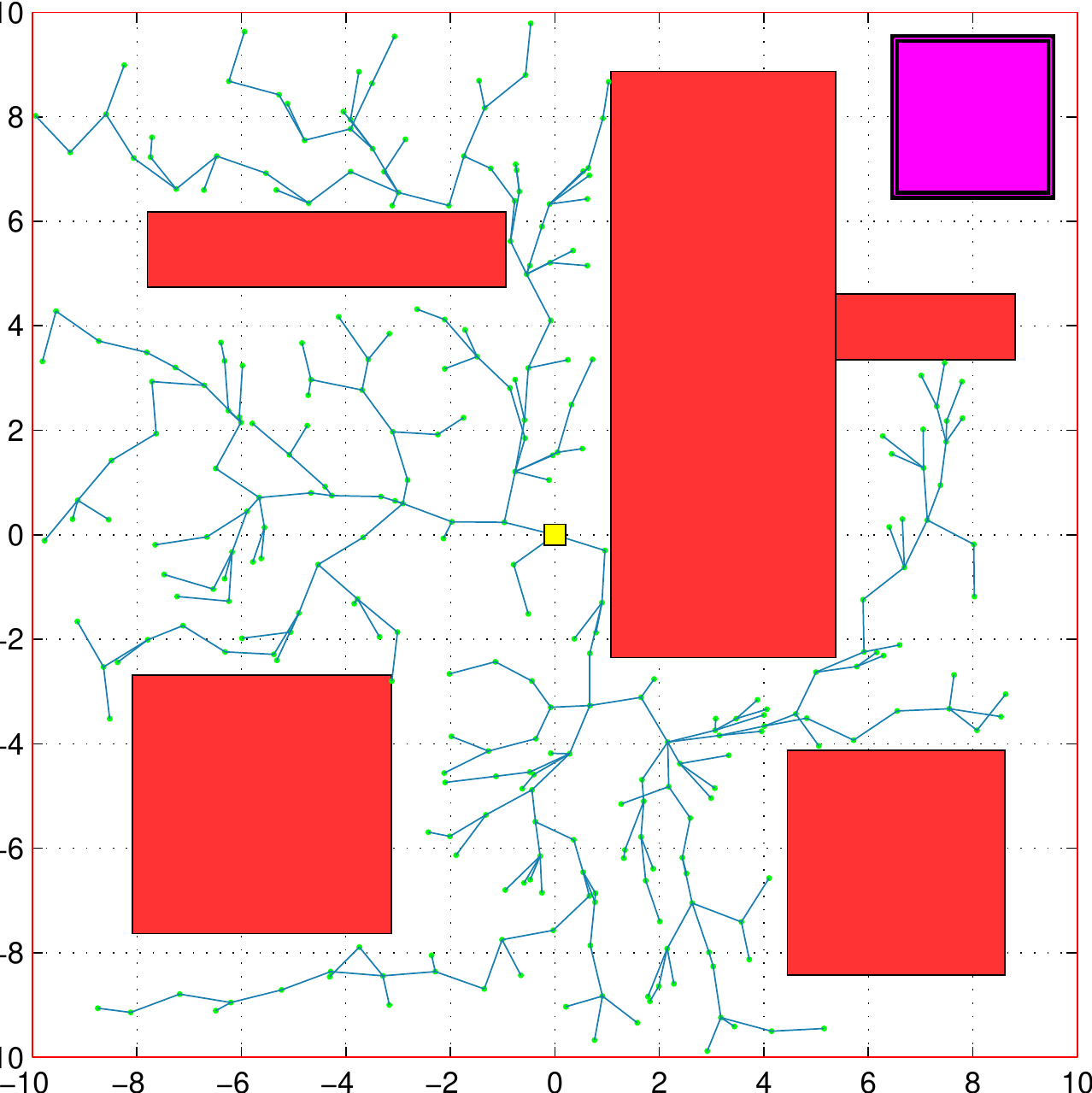}} \label{sim1_rrt_10}} \subfigure[
      ]{\scalebox{0.45}{\includegraphics[trim = 4.55cm 7.6cm 4.05cm 7.35cm, clip = true
          ]{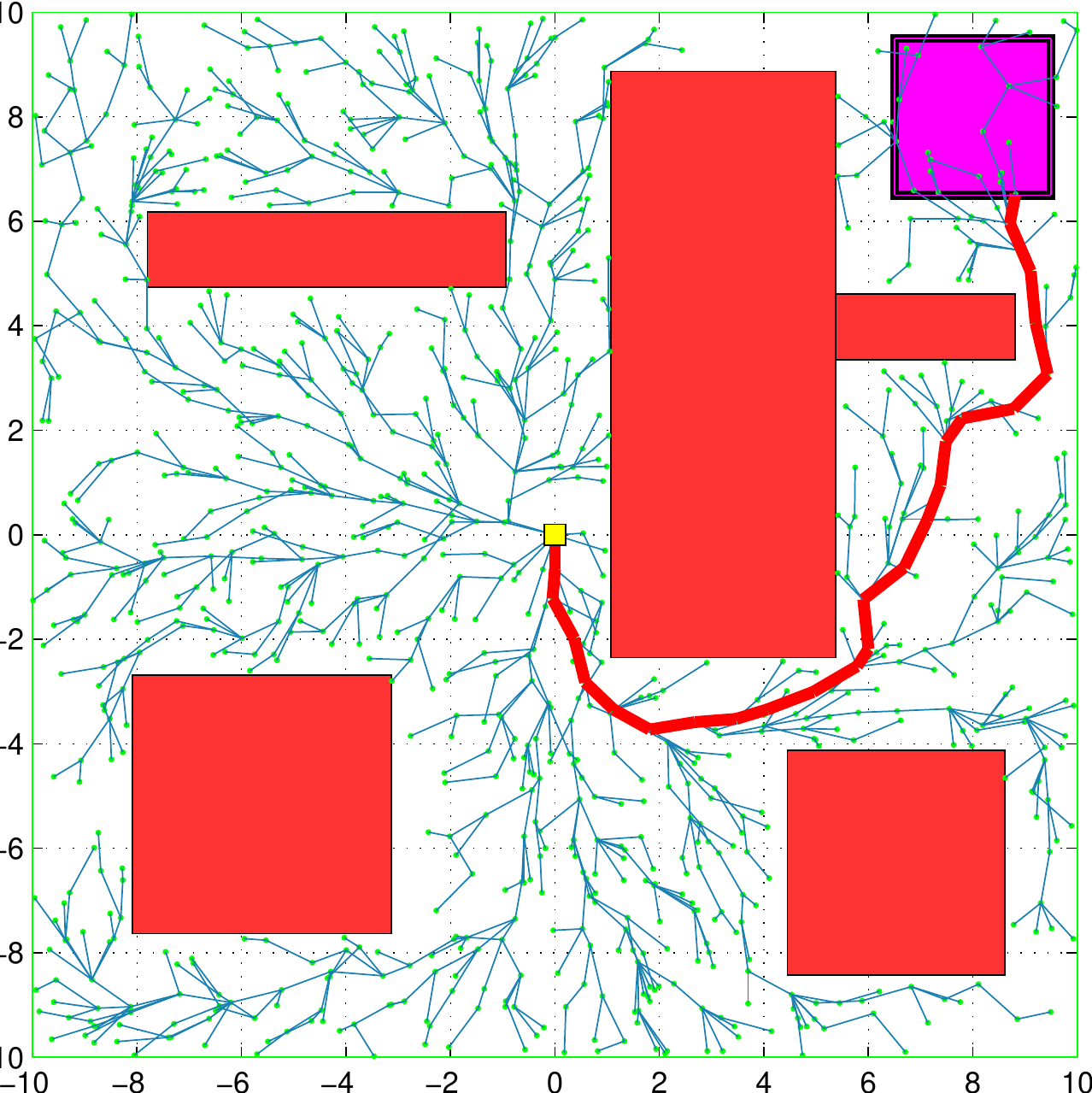}} \label{sim1_optrrt_10}}
      \subfigure[]{\scalebox{0.45}{\includegraphics[trim = 4.55cm 7.6cm 4.05cm 7.35cm, clip = true
          ]{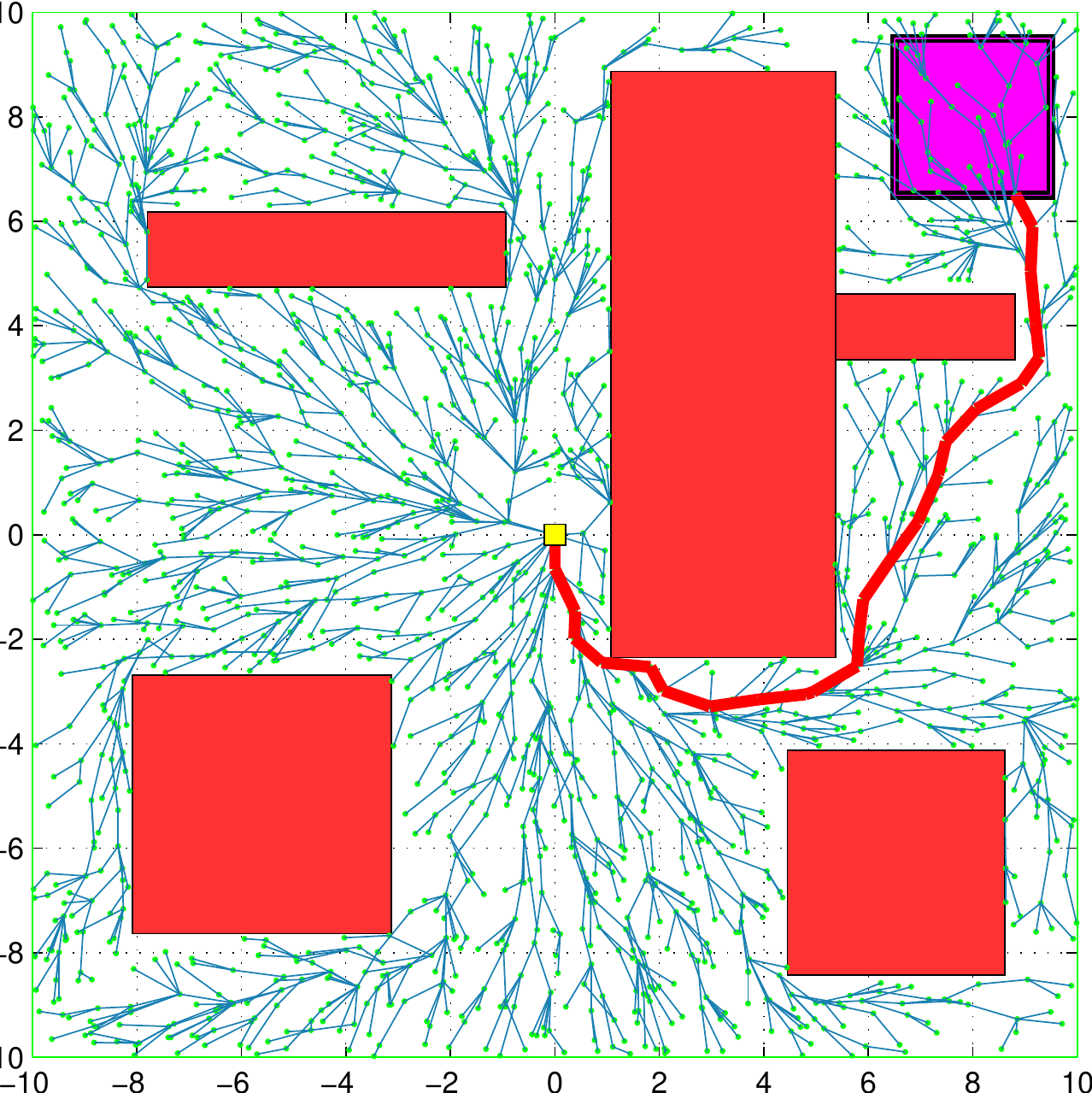}} \label{sim1_rrt_20}} } \mbox{ \subfigure[
      ]{\scalebox{0.45}{\includegraphics[trim = 4.55cm 7.6cm 4.05cm 7.35cm, clip = true
          ]{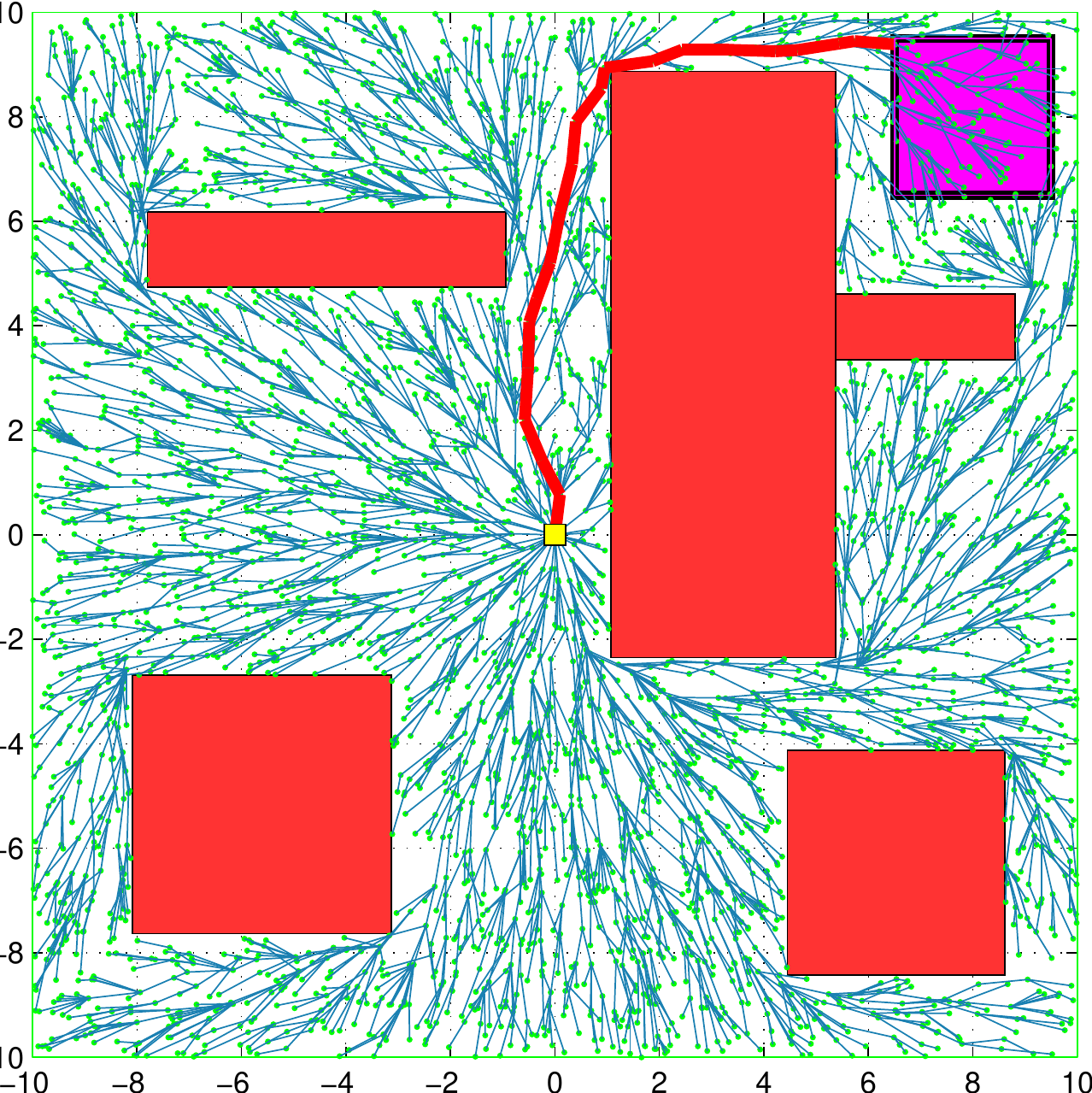}} \label{sim1_optrrt_20}}
      \subfigure[]{\scalebox{0.45}{\includegraphics[trim = 4.55cm 7.6cm 4.05cm 7.35cm, clip = true
          ]{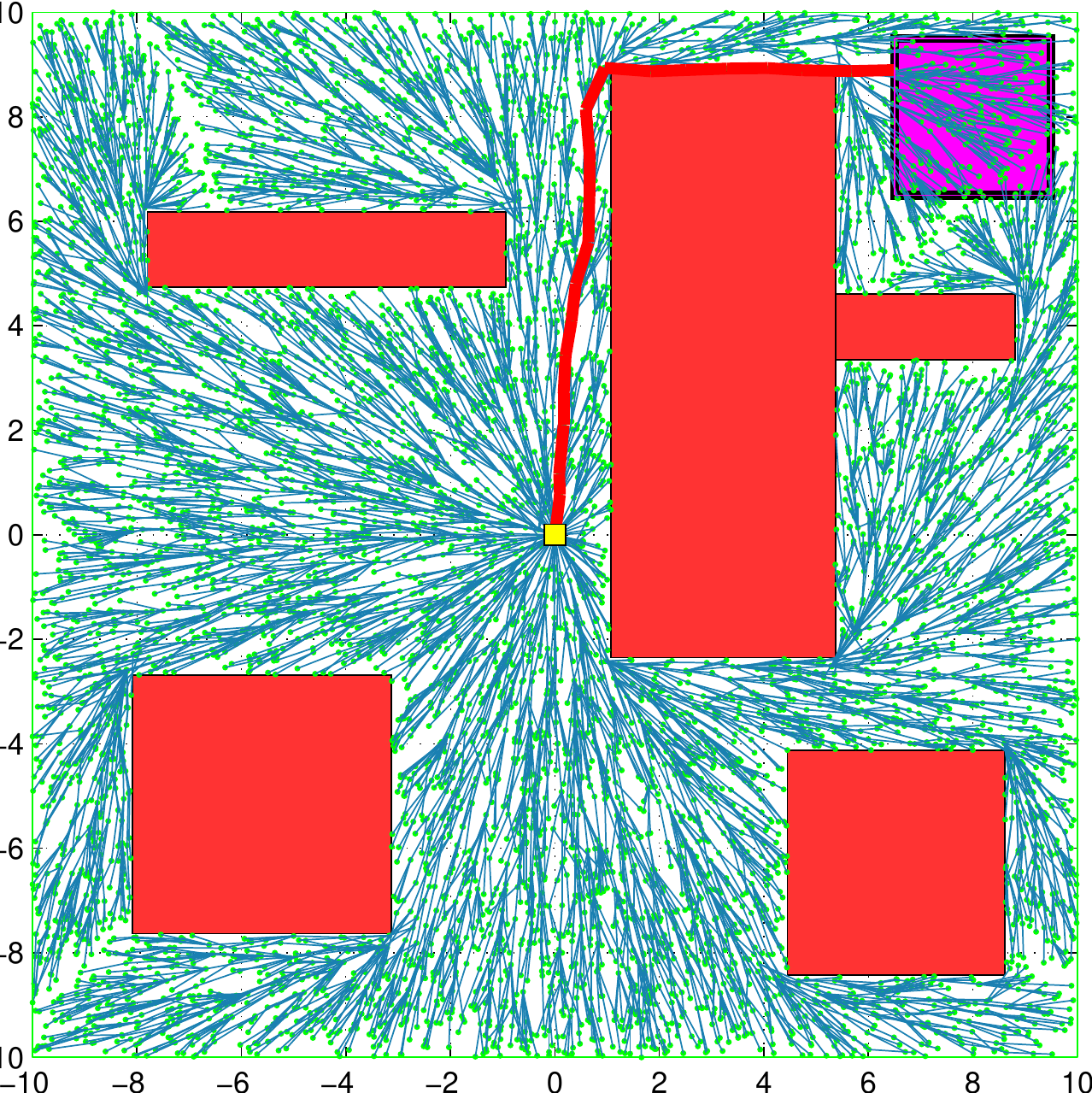}} \label{sim1_rrt_20}} \subfigure[
      ]{\scalebox{0.45}{\includegraphics[trim = 4.55cm 7.6cm 4.05cm 7.35cm, clip = true
          ]{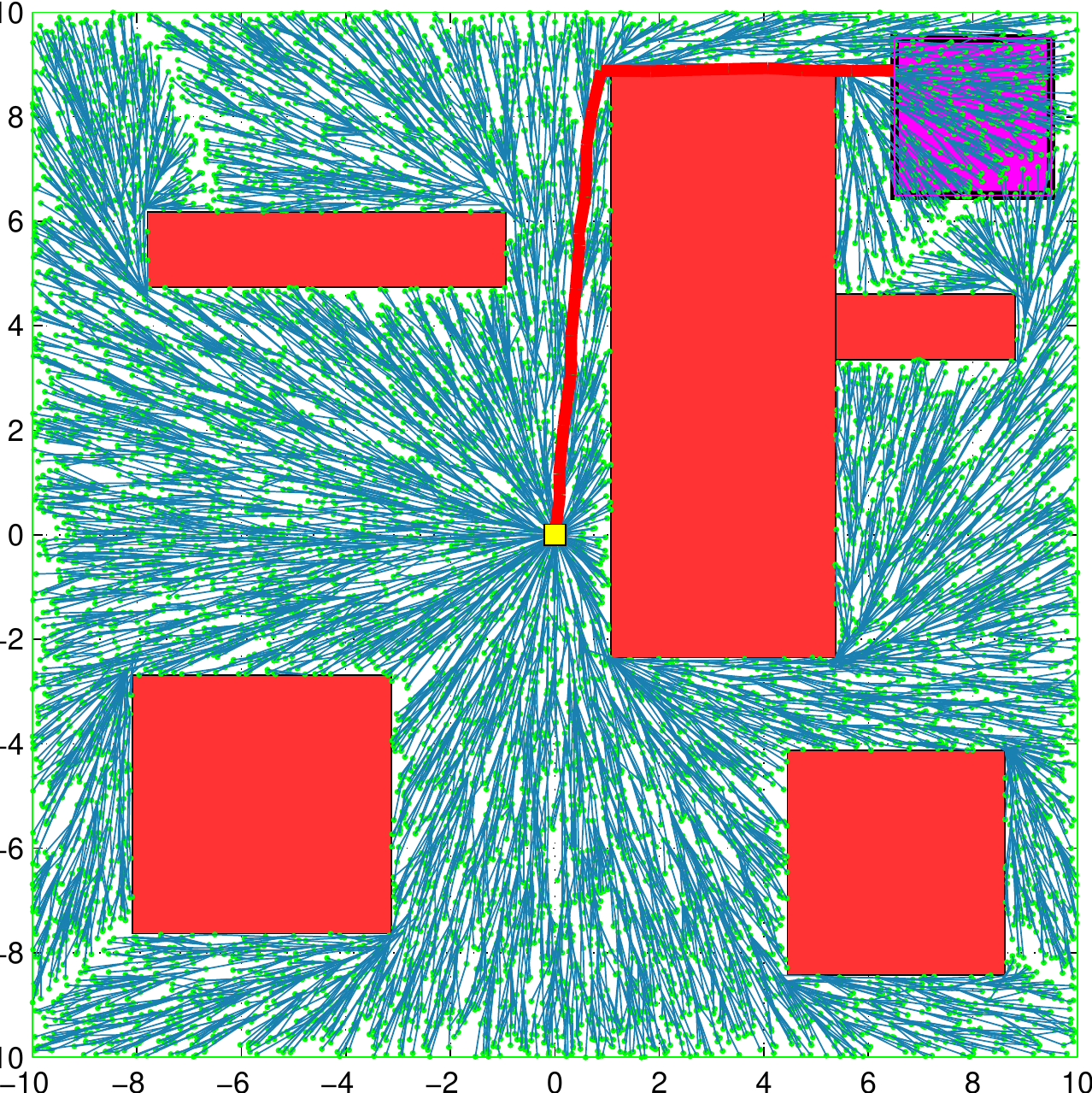}} \label{sim1_optrrt_20}} }
    \caption{RRT$^*$ algorithm shown after 500 (a), 1,500 (b), 2,500 (c), 5,000 (d), 10,000 (e), 15,000
      (f) iterations.}
    \label{figure:sim2optrrt}
  \end{center}
\end{figure*}

\begin{figure}[ht]
  \begin{center}
    \mbox{
      \subfigure[]{\scalebox{0.35}{\includegraphics{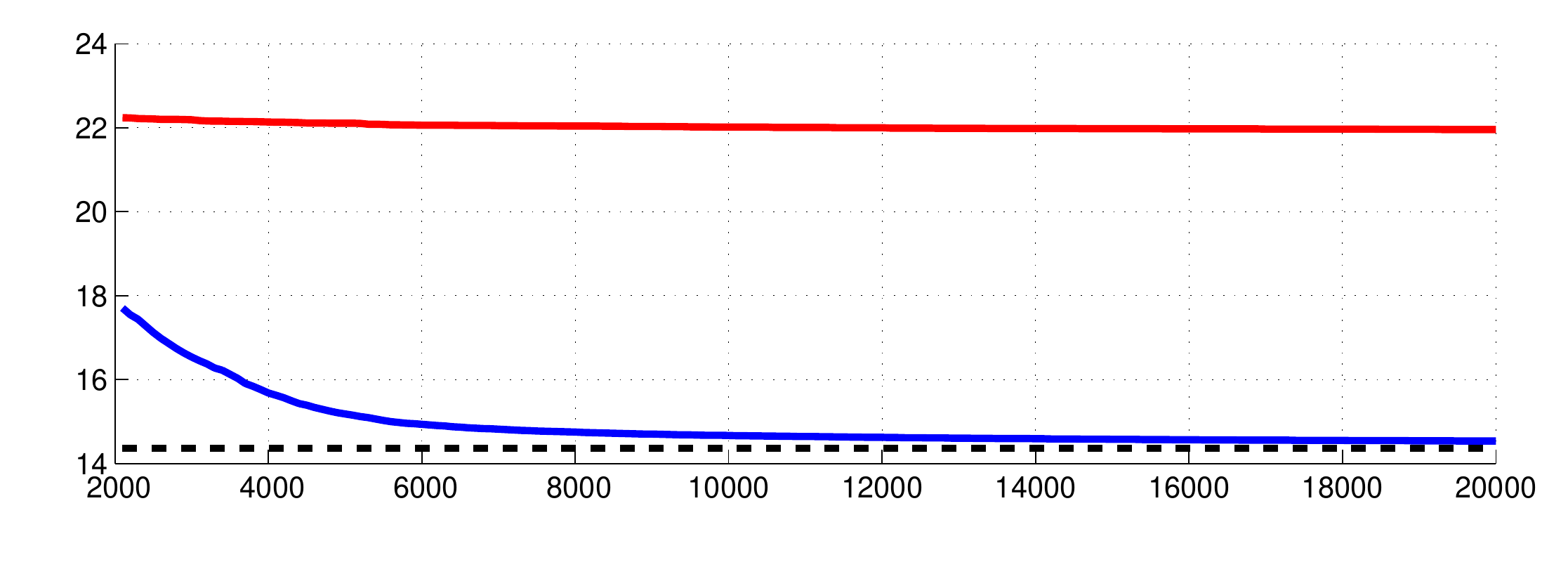}} 
        \label{sim0_mc_cost}}}
    \mbox{ 
      \subfigure[]{\scalebox{0.35}{\includegraphics{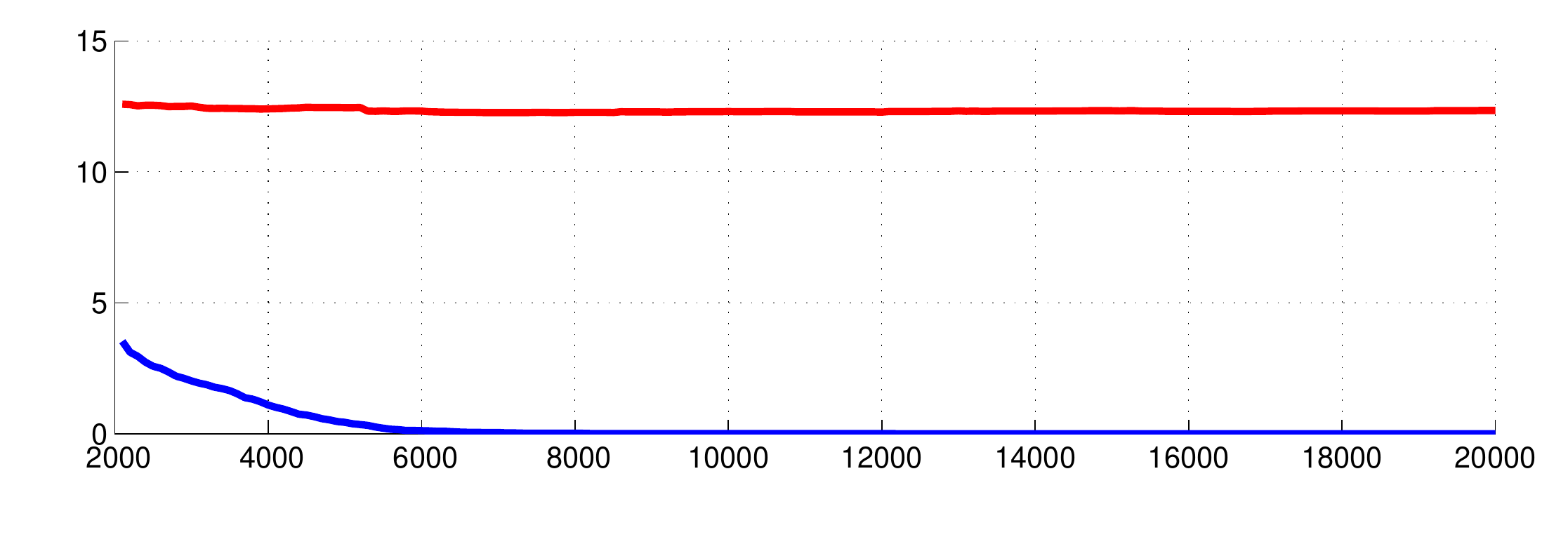}} 
        \label{sim0_mc_cost_var}}}
    \caption{An environment cluttered with obstacles is considered. The cost of the best paths in
      the RRT (shown in red) and the RRT$^*$ (shown in blue) plotted against iterations averaged
      over 500 trials in (a). The optimal cost is shown in black. The variance of the trials is
      shown in (b).}
    \label{figure:sim2cost}
  \end{center}
\end{figure}

\begin{figure}[ht]
  \begin{center}
    \includegraphics[trim = 2.9cm 1.8cm 2cm 1.4cm, clip = true, height =
    8.8cm]{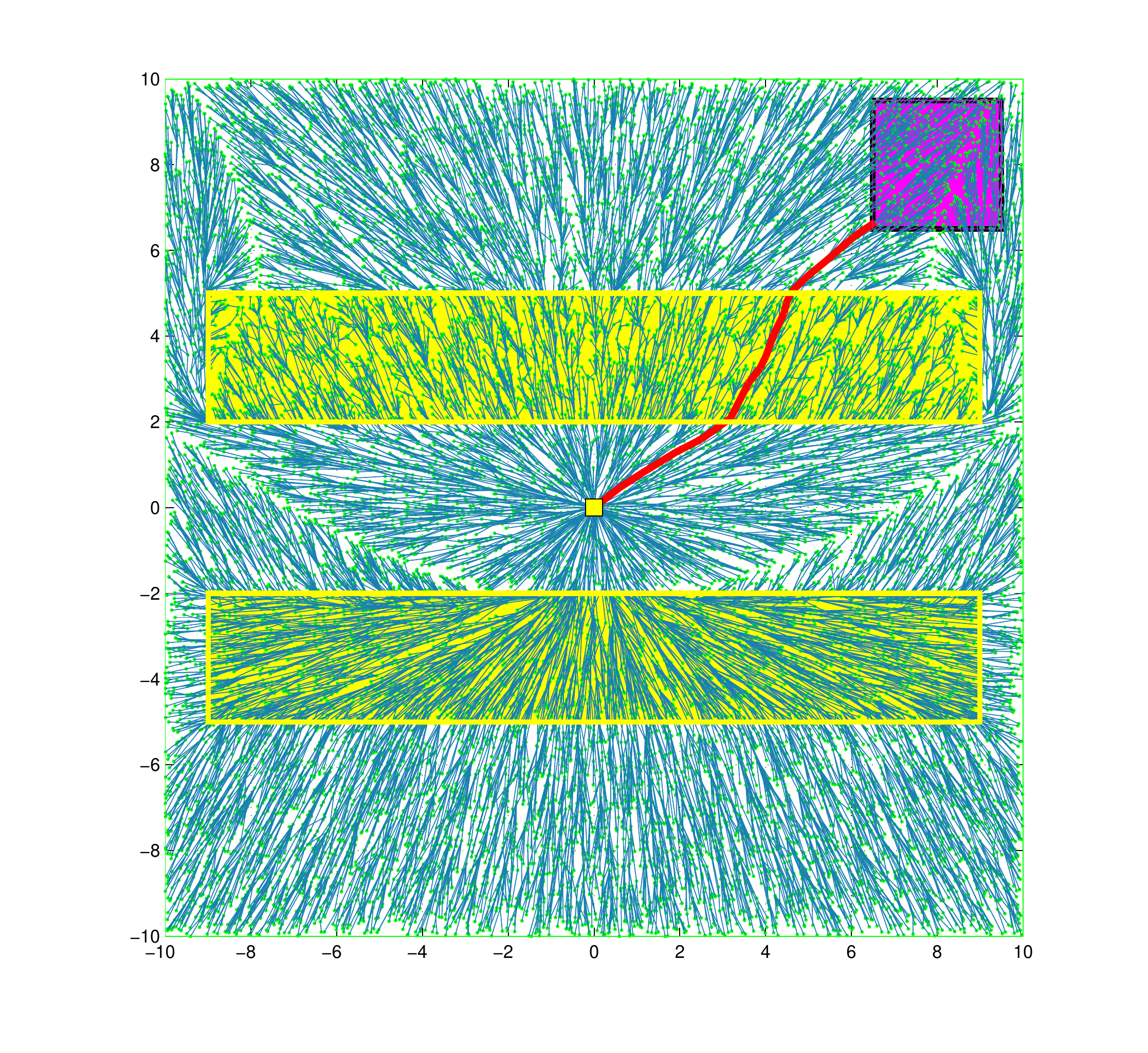}
    \caption{RRT$^*$ algorithm at the end of iteration 20,000 in an environment with no obstacles.
      The upper yellow region is the high-cost region, whereas the lower yellow region is low-cost.}
    \label{figure:sim3}
  \end{center}
\end{figure}

\begin{figure}[ht]
  \begin{center}
    \includegraphics[height = 3.8cm]{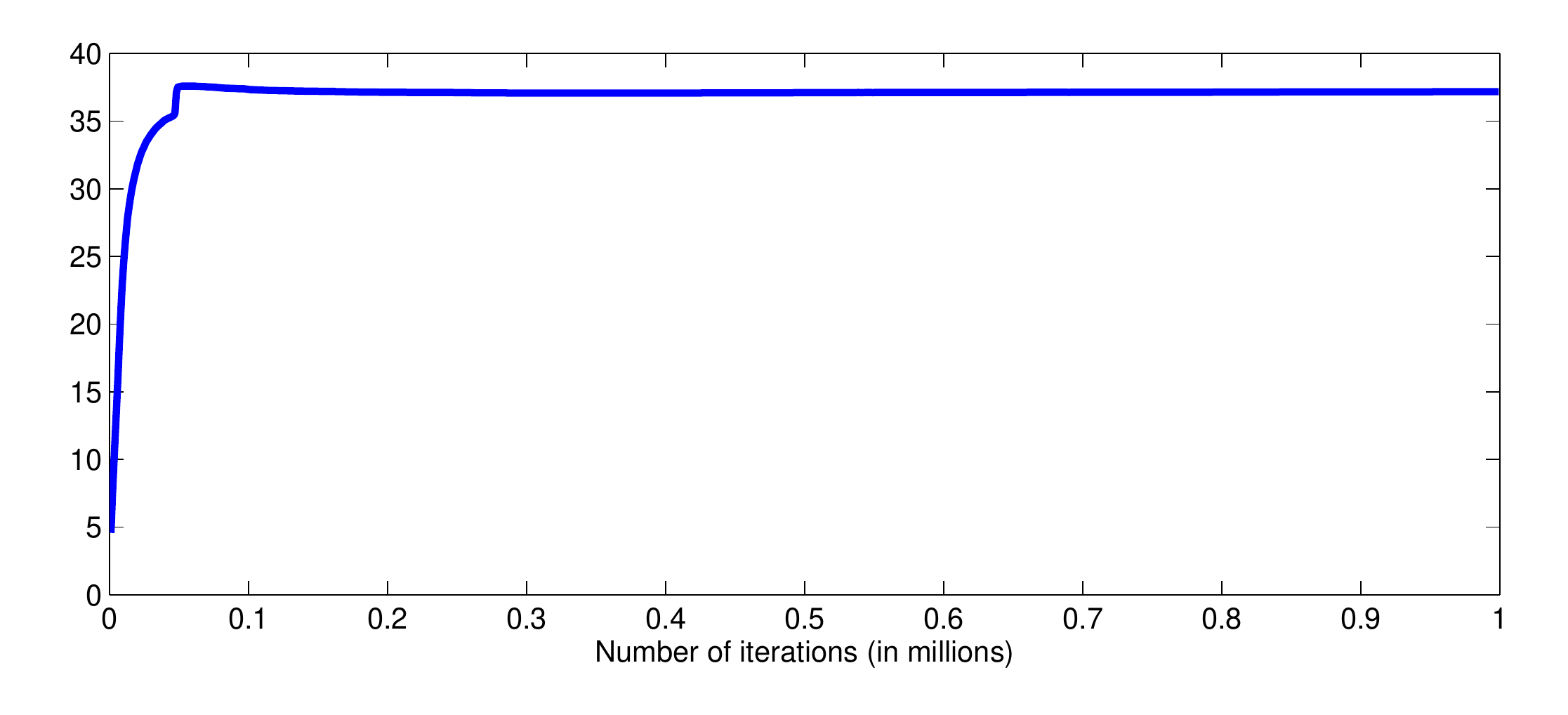}
    \caption{A comparison of the running time of the RRT$^*$ and the RRT algorithms. The ratio of
      the running time of the RRT$^*$ over that of the RRT up until each iteration is plotted versus
      the number of iterations.}
    \label{figure:sim0time}
  \end{center}
\end{figure}

\begin{figure}[ht]
  \begin{center}
    \includegraphics[height = 3cm]{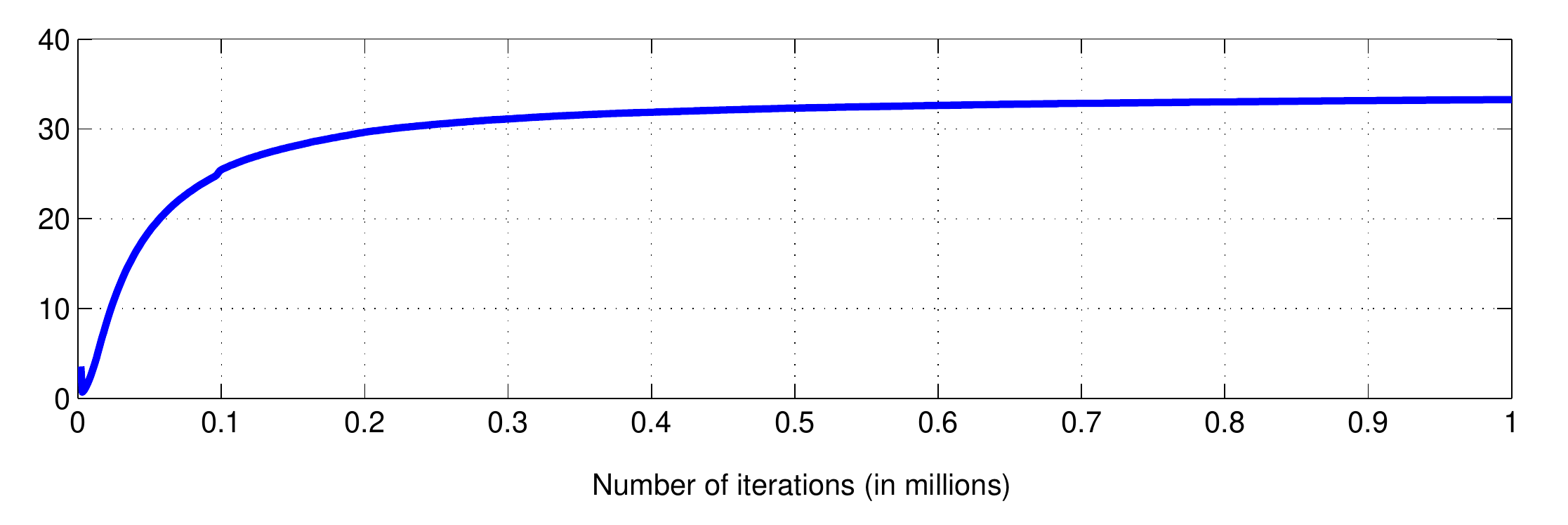}
    \caption{A comparison of the running time of the RRT$^*$ and the RRT algorithms in an
      environment with obstacles. The ratio of the running time of the RRT$^*$ over that of the RRT
      up until each iteration is plotted versus the number of iterations.}
    \label{figure:sim1time}
  \end{center}
\end{figure}

\section{Conclusion} \label{section:conclusion}

This paper presented the results of a thorough analysis of the RRT and RRG algorithms for optimal
motion planning. It is shown that, as the number of samples increases, the RRT algorithm converges
to a sub-optimal solution almost surely. On the other hand, it is proven that the RRG algorithm has
the asymptotic optimality property, i.e., almost sure convergence to an optimum solution, which the
RRT algorithm lacked.
The paper also proposed a novel algorithm called the RRT$^*$, which inherits the asymptotic
optimality property of the RRG, while maintaining a tree structure rather than a graph.
The RRG and the RRT$^*$ were shown to have no significant overhead when compared to the RRT
algorithm in terms of asymptotic computational complexity.
Experimental evidence that demonstrate the effectiveness of the algorithms proposed and support the
theoretical claims were also provided.

The results reported in this paper can be extended in a number of directions, and applied to other
sampling-based algorithms other than RRT. First of all, the proposed approach, building on the
theory of random graphs to adjust the length of new connections can enhance the computational
efficiency of PRM-based algorithms. Second, the algorithms and the analysis should be modified to
address motion planning problems in the presence of differential constraints, also known as
kino-dynamic planning problems. A third direction is the optimal planning problem in the presence of
temporal/logic constraints on the trajectories, e.g., expressed using formal specification languages
such as Linear Temporal Logic, or the $\mu$-calculus. Such constraints correspond to, e.g., rules of
the road constraints for autonomous ground vehicles, mission specifications for autonomous robots,
and rules of engagement in military applications. Ultimately, incremental sampling-based algorithms
with asymptotic optimality properties may provide the basic elements for the on-line solution of
differential games, as those arising when planning in the presence of dynamic obstacles.

Finally, it is noted that the proposed algorithms may have applications outside of the robotic
motion planning domain. In fact, the class of incremental sampling algorithm described in this paper
can be readily extended to deal with problems described by partial differential equations, such as
the eikonal equation and the Hamilton-Jacobi-Bellman equation.

\section*{Acknowledgments}
The authors are grateful to Professors M.S. Branicky and G.J. Gordon for their insightful comments
on a draft version of this paper. This research was supported in part by the Michigan/AFRL
Collaborative Center on Control Sciences, AFOSR grant no. FA 8650-07-2-3744. Any opinions, findings,
and conclusions or recommendations expressed in this publication are those of the authors and do not
necessarily reflect the views of the supporting organizations.

\bibliography{arxiv} \bibliographystyle{unsrt}

\begin{thebibliography}{10}

\bibitem{latombe.ijrr99}
J.~Latombe.
\newblock Motion planning: A journey of robots, molecules, digital actors, and
  other artifacts.
\newblock {\em International Journal of Robotics Research}, 18(11):1119--1128,
  1999.

\bibitem{bhatia.frazzoli.hscc04}
A.~Bhatia and E.~Frazzoli.
\newblock Incremental search methods for reachability analysis of continuous
  and hybrid systems.
\newblock In R.~Alur and G.J. Pappas, editors, {\em Hybrid Systems: Computation
  and Control}, number 2993 in Lecture Notes in Computer Science, pages
  142--156. Springer-Verlag, Philadelphia, PA, March 2004.

\bibitem{branicky.curtis.ea.ieeeproc06}
M.~S. Branicky, M.~M. Curtis, J.~Levine, and S.~Morgan.
\newblock Sampling-based planning, control, and verification of hybrid systems.
\newblock {\em IEEE Proc. Control Theory and Applications}, 153(5):575--590,
  Sept. 2006.

\bibitem{cortes.jailet.ea.icra07}
J.~Cortes, L.~Jailet, and T.~Simeon.
\newblock Molecular disassembly with {RRT}-like algorithms.
\newblock In {\em IEEE International Conference on Robotics and Automation
  (ICRA)}, 2007.

\bibitem{liu.badler.comp_anim_conf03}
Y.~Liu and N.I. Badler.
\newblock Real-time reach planning for animated characters using hardware
  acceleration.
\newblock In {\em IEEE International Conference on Computer Animation and
  Social Characters}, pages 86--93, 2003.

\bibitem{finn.kavraki.algorithmica99}
P.W. Finn and L.E. Kavraki.
\newblock Computational approaches to drug design.
\newblock {\em Algorithmica}, 25:347--371, 1999.

\bibitem{lozanoperez.wesley.comm_acm79}
T.~Lozano-Perez and M.~A. Wesley.
\newblock An algorithm for planning collision-free paths among polyhedral
  obstacles.
\newblock {\em Communications of the ACM}, 22(10):560--570, 1979.

\bibitem{schwartz.sharir.adv_app_math83}
J.~T. Schwartz and M.~Sharir.
\newblock On the `piano movers' problem: {II}. general techniques for computing
  topological properties of real algebraic manifolds.
\newblock {\em Advances in Applied Mathematics}, 4:298--351, 1983.

\bibitem{reif.sym_foun_com_sci79}
J.H. Reif.
\newblock Complexity of the mover's problem and generalizations.
\newblock In {\em Proceedings of the IEEE Symposium on Foundations of Computer
  Science}, 1979.

\bibitem{brooks.lozanoperez.icai83}
R.~Brooks and T.~Lozano-Perez.
\newblock A subdivision algorithm in configuration space for findpath with
  rotation.
\newblock In {\em International Joint Conference on Artificial Intelligence},
  1983.

\bibitem{barraquand.latombe.ijrr93}
J.~Barraquand and J.~C. Latombe.
\newblock Robot motion planning: A distributed representation approach.
\newblock {\em International Journal of Robotics Research}, 10(6):628--649,
  1993.

\bibitem{khatib.ijrr86}
O.~Khatib.
\newblock Real-time obstacle avoidance for manipulators and mobile robots.
\newblock {\em International Journal of Robotics Research}, 5(1):90--98, 1986.

\bibitem{canny.book88}
J.~Canny.
\newblock {\em The Complexity of Robot Motion Planning}.
\newblock MIT Press, 1988.

\bibitem{ge.cui.autorobo02}
S.~S. Ge and Y.J. Cui.
\newblock Dynamic motion planning for mobile robots using potential field
  method.
\newblock {\em Autonomous Robots}, 13(3):207--222, 2002.

\bibitem{koren.borenstein.icra91}
Y.~Koren and J.~Borenstein.
\newblock Potential field methods and their inherent limitations for mobile
  robot navigation.
\newblock In {\em IEEE Conference on Robotics and Automation}, 1991.

\bibitem{kavraki.latombe.icra94}
L.~Kavraki and J.~Latombe.
\newblock Randomized preprocessing of configuration space for fast path
  planning.
\newblock In {\em IEEE International Conference on Robotics and Automation},
  1994.

\bibitem{kavraki.svetska.ea.tro96}
L.E. Kavraki, P.~Svestka, J~Latombe, and M.H. Overmars.
\newblock Probabilistic roadmaps for path planning in high-dimensional
  configuration spaces.
\newblock {\em IEEE Transactions on Robotics and Automation}, 12(4):566--580,
  1996.

\bibitem{lavalle.kuffner.ijrr01}
S.~M. LaValle and J.~J. Kuffner.
\newblock Randomized kinodynamic planning.
\newblock {\em International Journal of Robotics Research}, 20(5):378--400, May
  2001.

\bibitem{prentice.roy.ijrr09}
S.~Prentice and N.~Roy.
\newblock The belief roadmap: Efficient planning in blief space by factoring
  the covariance.
\newblock {\em International Journal of Robotics Research},
  28(11--12):1448--1465, 2009.

\bibitem{tedrake.manchester.ea.ijrr}
R.~Tedrake, I.~R. Manchester, M.~M. Tobekin, and J.~W. Roberts.
\newblock {LQR}-trees: Feedback motion planning via sums of squares
  verification.
\newblock {\em International Journal of Robotics Research (to appear)}, 2010.

\bibitem{luders.karaman.ea.acc10}
B.~Luders, S.~Karaman, E.~Frazzoli, and J.~P. How.
\newblock Bounds on tracking error using closed-loop rapidly-exploring random
  trees.
\newblock In {\em American Control Conference}, 2010.

\bibitem{berenson.kuffner.ea.icra08}
D.~Berenson, J.~Kuffner, and H.~Choset.
\newblock An optimization approach to planning for mobile manipulation.
\newblock In {\em IEEE International Conference on Robotics and Automation},
  2008.

\bibitem{yershova.lavalle.rep08}
A.~Yershova and S.~Lavalle.
\newblock Motion planning in highly constrained spaces.
\newblock Technical report, University of Illinois at Urbana-Champaign, 2008.

\bibitem{stilman.schamburek.ea.icra07}
M.~Stilman, J.~Schamburek, J.~Kuffner, and T.~Asfour.
\newblock Manipulation planning among movable obstacles.
\newblock In {\em IEEE International Conference on Robotics and Automation},
  2007.

\bibitem{koyuncu.ure.ea.j_intell_robot_syst10}
E.~Koyuncu, N.K. Ure, and G.~Inalhan.
\newblock Integration of path/manuever planning in complex environments for
  agile maneuvering {UCAV}s.
\newblock {\em Jounal of Intelligent and Robotic Systems}, 57(1--4):143--170,
  2010.

\bibitem{barraquand.kavraki.ea.ijrr97}
J.~Barraquand, L.~Kavraki, J.~Latombe, T.~Li, R.~Motwani, and P.~Raghavan.
\newblock A random sampling scheme for path planning.
\newblock {\em International Journal of Robotics Research}, 16:759--774, 1997.

\bibitem{hsu.latombe.ea.ijrr06}
D.~Hsu, J.~Latombe, and H.~Kurniawati.
\newblock On the probabilistic foundations of probabilistic roadmap planning.
\newblock {\em International Journal of Robotics Research}, 25:7, 2006.

\bibitem{kavraki.kolountzakis.ea.tro98}
L.~E. Kavraki, M.~N. Kolountzakis, and J.~Latombe.
\newblock Analysis of probabilistic roadmaps for path planning.
\newblock {\em IEEE Transactions on Roborics and Automation}, 14(1):166--171,
  1998.

\bibitem{kuffner.lavalle.icra00}
J.J. Kuffner and S.M. LaValle.
\newblock {RRT}-connect: An efficient approach to single-quert path planning.
\newblock In {\em Proceedings of the IEEE International Conference on Robotics
  and Automation}, 2000.

\bibitem{lavalle.book06}
S.~LaValle.
\newblock {\em Planning Algorithms}.
\newblock Cambridge University Press, 2006.

\bibitem{lindemann.lavalle.symp_rr05}
S.R. Lindemann and S.M. LaValle.
\newblock Current issues in sampling-based motion planning.
\newblock In P.~Dario and R.~Chatila, editors, {\em Eleventh International
  Symposium on Robotics Research}, pages 36--54. Springer, 2005.

\bibitem{branicky.lavalle.ea.icra01}
M.~S. Branicky, S.~M. LaValle, K.~Olson, and L.~Yang.
\newblock Quasi-randomized path planning.
\newblock In {\em IEEE Conference on Robotics and Automation}, 2001.

\bibitem{ladd.kavraki.tro04}
A.~L. Ladd and L.~Kavraki.
\newblock Measure theoretic analysis of probabilistic path planning.
\newblock {\em IEEE Transactions on Robotics and Automation}, 20(2):229--242,
  2004.

\bibitem{hsu.kindel.ea.ijrr02}
D.~Hsu, R.~Kindel, J.~Latombe, and S.~Rock.
\newblock Randomized kinodynamic motion planning with moving obstacles.
\newblock {\em International Journal of Robotics Research}, 21(3):233--255,
  2002.

\bibitem{frazzoli.dahleh.ea.jgcd02}
E.~Frazzoli, M.~Dahleh, and E.~Feron.
\newblock Real-time motion planning for agile autonomous vehicles.
\newblock {\em Journal of Guidance, Control, and Dynamics}, 25(1):116--129,
  2002.

\bibitem{branicky.curtis.ea.cdc03}
M.~S. Branicky, M.~M. Curtis, J.~A. Levine, and S.~B. Morgan.
\newblock {RRT}s for nonlinear, discrete , and hybrid planning and control.
\newblock In {\em IEEE Conference on Decision and Control}, 2003.

\bibitem{zucker.kuffner.ea.icra07}
M.~Zucker, J.~Kuffner, and M.~Branicky.
\newblock Multiple {RRT}s for rapid replanning in dynamic environments.
\newblock In {\em IEEE Conference on Robotics and Automation}, 2007.

\bibitem{bruce.veloso.lncs02}
J.~Bruce and M.M. Veloso.
\newblock {\em Real-Time Randomized Path Planning for Robot Navigation}, volume
  2752 of {\em Lecture Notes in Computer Science}, chapter RoboCup 2002: Robot
  Soccer World Cup VI, pages 288--295.
\newblock Springer, 2003.

\bibitem{kuwata.teo.ea.cst09}
Y.~Kuwata, J.~Teo, G.~Fiore, S.~Karaman, E.~Frazzoli, and J.P. How.
\newblock Real-time motion planning with applications to autonomous urban
  driving.
\newblock {\em IEEE Transactions on Control Systems}, 17(5):1105--1118, 2009.

\bibitem{teller.walter.ea.icra10}
S.~Teller, M.~R. Walter, M.~Antone, A.~Correa, R.~Davis, L.~Fletcher,
  E.~Frazzoli, J.~Glass, J.P. How, A.~S. Huang, J.~Jeon, S.~Karaman, B.~Luders,
  N.~Roy, and T.~Sainath.
\newblock A voice-commandable robotic forklift working alongside humans in
  minimally-prepared outdoor environments.
\newblock In {\em IEEE International Conference on Robotics and Automation},
  2010.

\bibitem{shkolnik.levashov.ea.unpub09}
A.~Shkolnik, M.~Levashov, I.~R. Manchester, and R.~Tedrake.
\newblock Bounding on rough terrain with the {L}ittle{D}og robot.
\newblock Under review.

\bibitem{kuffner.kagami.ea.autorobo02}
J.~J. Kuffner, S.~Kagami, K.~Nishiwaki, M.~Inaba, and H.~Inoue.
\newblock Dynamically-stable motion planning for humanoid robots.
\newblock {\em Autonomous Robots}, 15:105--118, 2002.

\bibitem{karaman.frazzoli.cdc09}
S.~Karaman and E.~Frazzoli.
\newblock Sampling-based motion planning with deterministic $\mu$-calculus
  specifications.
\newblock In {\em IEEE Conference on Decision and Control (CDC)}, 2009.

\bibitem{clarke.grumberg.ea.99}
E.M. Clarke, O.~Grumberg, and D.A. Peled.
\newblock {\em Model Checking}.
\newblock Springer, 1999.

\bibitem{urmson.simmons.iros03}
C.~Urmson and R.~Simmons.
\newblock Approaches for heuristically biasing {RRT} growth.
\newblock In {\em Proceedings of the IEEE/RSJ International Conference on
  Robotics and Systems (IROS)}, 2003.

\bibitem{ferguson.stentz.iros06}
D.~Ferguson and A.~Stentz.
\newblock Anytime {RRT}s.
\newblock In {\em Proceedings of the IEEE/RSJ International Conference on
  Intelligent Robots and Systems (IROS)}, 2006.

\bibitem{wedge.branicky.aaai_ai_conf08}
N.~A. Wedge and M.S. Branicky.
\newblock On heavy-tailed runtimes and restarts in rapidly-exploring random
  trees.
\newblock In {\em Twenty-third AAAI Conference on Artificial Intelligence},
  2008.

\bibitem{likhachev.gordon.ea.nips04}
M.~Likhachev, G.~Gordon, and S.~Thrun.
\newblock Anytime {A}* with provable bounds on sub-optimality.
\newblock In {\em Advances in Neural Information Processing Systems}, 2004.

\bibitem{likhachev.ferguson.ea.aij08}
M.~Likhachev, D.~Ferguson, G.~Gordon, A.~Stentz, and S.~Thrun.
\newblock Anytime search in dynamic graphs.
\newblock {\em Artificial intelligence Journal}, 172(14):1613--1643, 2008.

\bibitem{stentz.ijcai95}
D.~Stentz.
\newblock The focussed {D}* algorithm for real-time replanning.
\newblock In {\em International Joint Conference on Artificial Intelligence},
  1995.

\bibitem{likhachev.ferguson.ijrr09}
M.~Likhachev and D.~Ferguson.
\newblock Planning long dynamically-feasible maneuvers for autonomous vehicles.
\newblock {\em International Journal of Robotics Research}, 28(8):933--945,
  2009.

\bibitem{dolgov.thrun.ea.exp_robotics09}
D.~Dolgov, S.~Thrun, M.~Montemerlo, and J.~Diebel.
\newblock {\em Experimental Robotics}, chapter Path Planning for Autonomous
  Driving in Unknown Environments, pages 55--64.
\newblock Springer, 2009.

\bibitem{penrose.book03}
M.~Penrose.
\newblock {\em Random Geometric Graphs}.
\newblock Oxford University Press, 2003.

\bibitem{dall.christensen.phys_rev_e02}
J.~Dall and M.~Christensen.
\newblock Random geometric graphs.
\newblock {\em Physical Review E}, 66(1):016121, Jul 2002.

\bibitem{LaValle.Branicky.ea:04}
S.M. {LaValle}, M.S. Branicky, and S.R. Lindemann.
\newblock On the relationship between classical grid search and probabilistic
  roadmaps.
\newblock {\em International Journal of Robotics Research}, 23(7--8):673--692,
  2004.

\bibitem{samet.book89b}
H.~Samet.
\newblock {\em Applications of Spatial Data Structures: Computer Graphics,
  Image Processesing and Gis}.
\newblock Addison-Wesley, 1989.

\bibitem{samet.book89a}
H.~Samet.
\newblock {\em Design and Analysis of Spatial Data Structures}.
\newblock Addison-Wesley, 1989.

\bibitem{atramentov.lavalle.icra02}
A.~Atramentov and S.~M. LaValle.
\newblock Efficient nearest neighbor searching for motion planning.
\newblock In {\em IEEE International Conference on Robotics and Automation},
  2002.

\bibitem{cohen.leiserson.book90}
T.~H. Cohen, C.~E. Leiserson, and R.~L. Rivest.
\newblock {\em Introduction to Algorithms}.
\newblock MIT Press, 1990.

\bibitem{edelsbrunner.book87}
H.~Edelsbrunner.
\newblock {\em Algorithms in Computational Geometry}.
\newblock Springer-Verlag, 1987.

\bibitem{clarkson.sjc88}
K.~L. Clarkson.
\newblock A randomized algorithm for closest-point querries.
\newblock {\em SIAM Journal of Computation}, 17:830--847, 1988.

\bibitem{arya.mount.ea.jacm99}
S.~Arya, D.~M. Mount, R.~Silverman, and A.~Y. Wu.
\newblock An optimal algorithm for approximate nearest neighbor search in fixed
  dimensions.
\newblock {\em Journal of the ACM}, 45(6):891--923, November 1999.

\bibitem{plaku.kavraki.wafr08}
E.~Plaku and L.~E. Kavraki.
\newblock Quantitative analysis of nearest-neighbors search in high-dimensional
  sampling-based motion planning.
\newblock In {\em Workshop on Algorithmic Foundations of Robotics (WAFR)},
  2008.

\bibitem{lee.wong.acta_informatica77}
D.~T. Lee and C.~K. Wong.
\newblock Worst-case analysis for region and partial region searches in
  multidimensional binary search trees and quad trees.
\newblock {\em Acta Informatica}, 9:23--29, 1977.

\bibitem{chanzy.devroye.ea.acta_informatica01}
P.~Chanzy, L.~Devroye, and C.~Zamora-Cura.
\newblock Analysis of range search for random k-d trees.
\newblock {\em Acta Informatica}, 37:355--383, 2001.

\bibitem{arya.mount.comp_geo00}
S.~Arya and D.~M. Mount.
\newblock Approximate range searching.
\newblock {\em Computational Geometry: Theory and Applications}, 17:135--163,
  2000.

\bibitem{arya.malamatos.ea.symp_dis_alg05}
S.~Arya, T.~Malamatos, and D.~M. Mount.
\newblock Space-time tradeoffs for approximate speherical range counting.
\newblock In {\em Symposium on Discrete Algorithms}, 2005.

\bibitem{lavalle.kuffner.tech_rep09}
S.~LaValle and J.~Kuffner.
\newblock Space filling trees.
\newblock Technical Report CMU-RI-TR-09-47, Carnegie Mellon University, The
  Robotics Institute, 2009.

\bibitem{Rowe.Alexander.00}
N.C. Rowe and R.S. Alexander.
\newblock Finding optimal-path maps for path planning across weighted regions.
\newblock {\em The International Journal of Robotics Research}, 19:83--95,
  2000.

\bibitem{david.nagaraja.book03}
H.~A. David and H.~N. Nagaraja.
\newblock {\em Order Statistics}.
\newblock Wiley, 2003.

\bibitem{grimmett.stirzaker.book01}
G.~Grimmett and D.~Stirzaker.
\newblock {\em Probability and Random Processes}.
\newblock Oxford University Press, {T}hird edition, 2001.

\bibitem{mitzenmacher.upfal.book05}
M.~Mitzenmacher and E.~Upfal.
\newblock {\em Probability and Computing: Randomized Algorithms and
  Probabilistic Analysis}.
\newblock Cambridge University Press, 2005.

\bibitem{muthukrishnan.pandurangan.symp_disc_alg05}
S.~Muthukrishnan and G.~Pandurangan.
\newblock The bin-covering technique for thresholding random geometric graph
  properties.
\newblock In {\em Proceedings of the sixteenth annual ACM-SIAM symposium on
  discrete algorithms}, 2005.

\end{thebibliography}

\section*{Appendix}
  
\subsection*{Proof of Theorem~\ref{theorem:rrtoptimality}}
Since the feasibility problem is assumed to admit a solution, the cost of the optimal path, $c^*$,
is a finite real number. 
First, consider the following technical lemma.
\begin{lemma} \label{lemma:nonequal_optimal} %
  The probability that the RRT constructs an optimal path at a finite iteration $i \in \mathbb{N}$
  is zero, i.e.,
  $$
  \PP \left( \cup_{i \in \mathbb{N}} \{ {\cal Y}^\mathrm{RRT}_i = c^*\} \right) = 0.
  $$
\end{lemma}
\begin{proof}
  Let $B_i$ denote the event the RRT has a path that has cost exactly equal to $c^*$ at the end of
  iteration $i$, i.e., $B_i = \{{\cal Y}^\mathrm{RRT}_i = c^*\}$. Let $B$ denote the event that the
  RRT finds a path that costs exactly $c^*$ at some finite iteration $i$. Then, $B$ can be written
  as $B = \cup_{i \in \mathbb{N}} B_i$. Notice that $B_i \subseteq B_{i+1}$; thus, we have that
  $\lim_{i \to \infty} \PP (B_i) = \PP(B)$, by monotonocity of measures. Notice also that by
  Assumptions~\ref{assumption:zeromeasureoptimal} and \ref{assumption:sampling},  $\PP
  (B_i) = 0$ for all $i \in \mathbb{N}$, since the probability that the set $\bigcup_{j = 1}^i \{
  {\tt Sample}(j)\}$ of points contains a point from a zero-measure set is zero. Hence,  $\PP (B) = 0$.
\end{proof}

Let $C_i$ denote the event that $x_\mathrm{init}$ is chosen to be extended by the ${\tt Nearest}$
procedure at iteration $i$. The following lemma direcly implies a necessary condition for asymptotic
optimality: with probability one, the initial state, $x_\mathrm{init}$, must be chosen for extension
infinitely often.
\begin{lemma} \label{lemma:necessary_condition}
  The following inequality holds:
  $$
  \PP \left(\left\{ \lim_{i \to \infty} {\cal Y}^\mathrm{RRT}_i = c^* \right\} \right) \le 
  \PP \left(\limsup_{i \to \infty} C_i\right).
  $$
\end{lemma}
\begin{proof}
Denote by $\psi(x)$ a minimum-cost path starting from a given state $x \in X_\mathrm{free}$
and reaching one of the states in the goal region\footnote{If there is no such path, then let
  $\Lambda$ be a symbol for ``unfeasible path,'' and set $\psi(x) = \Lambda$, with ${\tt
    Cost}(\Lambda)= +\infty$.},
and define a sequence $\{c_i\}_{i \in \mathbb{N}}$ of random variables as follows:
\begin{eqnarray*}
 & & \hspace{-0.5in} c_i = \min \big\{ {\tt Cost} ({\tt Line}(x_\mathrm{init}, x) \,\vert\, \psi(x))
 \,\big\vert\,\\ 
& & \hspace{0.5in}  x \in {\cal V}^\mathrm{RRT}_i, (x_\mathrm{init},x) \in {\cal E}^\mathrm{RRT}_i \big\}.
\end{eqnarray*}
Essentially, $c_i$ denotes the minimum cost that is incurred by following the line segment that
connects $x_\mathrm{init}$ to one of its children and following an optimal path afterwards. Notice
that $c_i$ is strictly greater than the optimal cost $c^*$ for all $i \in \mathbb{N}$, unless the
root node, $x_\mathrm{init}$, has a child node that is a part of an optimal path.

Recall, from the proof of Lemma~\ref{lemma:nonequal_optimal}, that $B$ denotes the event $\cup_{i \in
  \mathbb{N}}\{ {\cal Y}^\mathrm{RRT}_i = c^* \}$. Let $A_i$ denote the event that $c_i$ decreases
at iteration $i$, i.e., $A_i = \{ c_i < c_{i-1} \}$. First, notice that, conditioning on the event
$B^c$, we have that the convergence of the cost of the best path in the RRT to the optimum cost
(i.e., the event $\{\lim_{i \to \infty} {\cal Y}^\mathrm{RRT}_i = c^*\}$) implies that $c_i$ must
decrease infinitely often, i.e., $A_i$ must hold infinitely often. More precisely, we have that $\{
\lim_{i \to \infty}{\cal Y}^\mathrm{RRT}_i = c^* \} \cap B^c \subseteq \limsup_{i \to \infty} A_i$.

Noting that $\PP (B^c) = 1$ (by Lemma~\ref{lemma:nonequal_optimal}), it follows that
\begin{eqnarray*}
\PP\left(\left\{ \lim_{i \to \infty} {\cal Y}^\mathrm{RRT}_i = c^*\right\}\right) 
& = &\PP \left(\left\{ \lim_{i \to \infty} {\cal Y}^\mathrm{RRT}_i = c^*\right\} \cap B^c\right) \\
& \le & \PP \left(\limsup_{i \to \infty} A_i\right).
\end{eqnarray*}

Notice that, by Assumption~\ref{assumption:monotonecost}, for $A_i$ to occur infinitely often, $C_i$
must also occur infinitely often, i.e., $\limsup_{i \to \infty} A_i \subseteq \limsup_{i \to \infty}
C_i$. Thus,  $\PP (\limsup_{i \to \infty}) A_i \le \PP (\limsup_{i \to \infty} C_i)$,
which implies the lemma.
\end{proof}

To prove the theorem, it is shown that RRT fails to satisfy this necessary condition:
\begin{lemma} \label{lemma:infinite_often_connection}
  The probability that the RRT algorithm extends its root node infinitely many times is zero, i.e.,  $$\PP \left(\limsup_{i \to \infty} C_i\right) = 0.$$
\end{lemma}
\begin{proof}
Let $\{{\cal C}_1, {\cal C}_2, \dots, {\cal C}_K\}$ be a finite set of cones, such that (i) each cone is
centered at $x_\mathrm{init}$ and has aperture at most $\pi / 3$, and (ii) collectively the cones
cover the ball of radius $\eta$ around $x_\mathrm{init}$. Defined for all $k \in \{1,2, \dots, K\}$
and all $j \in \mathbb{N}$, let $D_{k,j}$ be the event that a node from inside cone ${\cal C}_k$ is
connected to $x_\mathrm{init}$ for the $j$th time. Clearly, for $C_i$ to occur infinitely often,
$D_{k,j}$ must occur for infinitely many $j$, for at least one $k$. More formally, 
$\limsup_{i \to \infty} C_i \subseteq \cup_{k = 1}^K \limsup_{j \to \infty} D_{k,j}$, and that 
\begin{eqnarray*}
\PP (\limsup_{i \to \infty} C_i) 
 & \le & \PP (\cup_{k = 1}^K \limsup_{j \to \infty} D_{k,j}) \\
 & \le & \sum_{k = 1}^K \PP(\limsup_{j \to \infty} D_{k,j}).
\end{eqnarray*} 

The next step is to show that $\PP(\limsup_{j \to \infty} D_{k,j}) = 0$ holds for all $k \in \{1,2, \dots, K \}$,
which implies the lemma.
Let $k$ be any index from $\{1,2, \dots, K\}$. 
Let $V_{k,j}$ denote the set of vertices inside cone ${\cal C}_k$ right after the connection of $j$th
node to $x_\mathrm{init}$ from inside the cone ${\cal C}_k$.
Let $r_{\min, k, j}$ be the distance of the node that is inside $V_{k,j}$ and has minimum distance
to $x_\mathrm{init}$ among all the vertices in $V_{k,j}$, i.e.,
$$
r_{\min, k, j} = \min_{v \in V_{k,j}} \Vert v - x_\mathrm{init} \Vert.
$$
This distance is a random variable for fixed values of $k$ and $j$. Let
$f_{r_\mathrm{min}, k, j}$ denote its probability density function.
Also, given two points $x, y \in X$, let ${\cal D}_{x,y}$ denote the set of all points in $X$ that
are closer to $y$ than they are to $x$, i.e.,
$$
{\cal D}_{x,y} = \big\{ z \in X \,\big\vert\, \Vert x-z \Vert > \Vert y - z \Vert \big\}.
$$

For any $j \in \mathbb{N}_{>0}$, conditioning on $r_{\min, k, j} = z$, the probability that a node
from ${\cal C}_k$ is connected to $x_\mathrm{init}$ for the $j$th time is upper bounded by the
measure of the region ${\cal C}_k \setminus {\cal D}_{x_\mathrm{init}, y}$, where $y$ is any point
with distance $z$ to $x_\mathrm{init}$ (see Figure~\ref{figure:coneinclusion}). Hence, $\PP (D_{k,j}
\given r_{\min, k, j} = z) \le \alpha \mu ({\cal C}_k \setminus {\cal D}_{x_\mathrm{init}, y})$ for
some contant $\alpha \in \mathbb{R}_{>0}$ by Assumption~\ref{assumption:sampling}. Furthermore,
since ${\cal C}_k$ has aperture at most $\pi / 3$, the region ${\cal C}_k \setminus {\cal
  D}_{x_\mathrm{init}, y}$ is included within the ball of radius $z$ centered at $x_\mathrm{init}$
(also illustrated in Figure~\ref{figure:coneinclusion}). Hence,
\begin{eqnarray*}
  \PP (D_{k,j}) & = & \int_{z = 0}^\infty \PP(D_{k,j} \given r_{\min, k, j} = z) f_{r_{\min,k,j}}
  (z) dz \\
  & \le & \int_{z = 0}^\infty \alpha \, \mu({\cal C}_k \setminus {\cal D}_{x_\mathrm{init}, y})
  f_{r_{\min,k,j}}(z) dz \\
  & \le & \int_{z = 0}^\infty \alpha \, \mu({\cal B}_{x_\mathrm{init},z}) f_{r_{\min,k,j}}(z) dz.
\end{eqnarray*}
Note that $\mu ({\cal B}_{x_\mathrm{init}, z}) \le \zeta_d z^d$, where $d$ is, recall, the
dimensionality of the smallest Euclidean space containing $X$, and $\zeta_d$ is the volume of the
unit ball in this space. Hence,
\begin{eqnarray*}
  \PP (D_{k,j}) \le \alpha\, \zeta_d \,  \int_{z = 0}^\infty z^d f_{r_{\min,k,j}}(z) dz 
  = \alpha\, \zeta_d \; \EE [(r_{\min, k, j})^d],
\end{eqnarray*}
where the last equality follows merely by the definition of expectation. By the order statistics of
the minimum distance, we have that $\EE [(r_{\min, k, j})^d]$ evaluates to $\frac{\beta}{j^d}$ for
some constant $\beta \in \mathbb{R}_{>0}$, under Assumption~\ref{assumption:sampling} (see,
e.g.,~\cite{david.nagaraja.book03}). Hence, we have that, for all $k$, $\sum_{j = 1}^\infty \PP
(D_{k,j}) < \infty$, since $d \ge 2$. Then, by the Borel-Cantelli
Lemma~\cite{grimmett.stirzaker.book01}, one can conclude that the probability that $D_{k,j}$ occurs for
infinitely many $j$ is zero, i.e., $\PP(\limsup_{j \to \infty} D_{k,j}) = 0$, which implies the
claim.
\end{proof}

The theorem is immediate from Lemmas~\ref{lemma:necessary_condition} and
\ref{lemma:infinite_often_connection}.

\begin{figure}[ht]
  \begin{center}
    \includegraphics[height = 4cm]{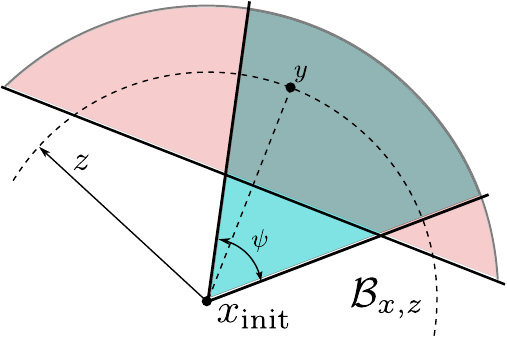}
    \caption{Illustration of intersection of ${\cal C}_k$ (shaded in blue) of aperture $\psi$ at
      least $\pi / 3$ with the region $ {\cal D}_{x_\mathrm{init}, y}$. The ball ${\cal
        B}_{x_\mathrm{init}, z}$ of radius $z$ that includes ${\cal C}_k \setminus {\cal
        D}_{x_\mathrm{init}, y}$ is also shown with dotted lines in the figure.}
    \label{figure:coneinclusion}
  \end{center}
\end{figure}

\subsection*{Proof of Theorem~\ref{theorem:rrgoptimality}}

In order to prove that ${\cal Y}^\mathrm{RRG}_i$ converges to $c^*$ almost surely, it will be shown
that $\sum_{i = 0}^\infty \PP ({\cal Y}^\mathrm{RRG}_i > c^* + \varepsilon)$ is finite for all
$\varepsilon > 0$ (see, for instance,~\cite{grimmett.stirzaker.book01} for justification of this
implication).

Recall that the volume of the unit ball in $d$ dimensions is denoted by $\zeta_d$.
Recall also that $r_i$ denotes the radius of the ball with volume $\gamma \frac{\log i}{i}$ used in
the ${\tt \NearNodes}$ procedure and that we have $r_i = \left(\frac{\gamma}{\zeta_d} \frac{\log
    i}{i} \right)^{1/d}$. 
Let $\sigma^* : [0, s^*] \to \Cl(X_\mathrm{free})$ be an optimal path with length $s^*$.
Recall that $i$ denotes the iteration number (see Algorithm~\ref{algorithm:motionplanning}).
The proof is first outlined in the following section, and then  described in detail. 

\subsubsection*{Outline of the proof}

First, construct a sequence $\{ X_i \}_{i \in \mathbb{N}}$ of subsets of $X_\mathrm{free}$ such that
(i) the sequence is monotonically non-decreasing in the partial subset ordering, i.e., $X_{i}
\subseteq X_{i + 1}$ for all $i \in \mathbb{N}$, (ii) for all large enough $i \in \mathbb{N}$, any
point in $X_i$ is at least a certain fraction of $r_i$ away from obstacles, i.e., $\Vert x - y \Vert
\ge \lambda_i$ for all $x \in X_i$ and all $y \in X_\mathrm{obs}$, where $\lambda_i = \alpha r_i$
for some constant $\alpha$ (see Figures~\ref{figure:obstacles_only} and \ref{figure:X_n}), in
particular, the set $X_i$ is included in $X_\mathrm{free}$, and (iii) the sequence converges to
$X_\mathrm{free}$ in the sense that $\cup_{i \in \mathbb{N}} X_i = X_\mathrm{free}$.

Second, construct a sequence $\{\sigma_i \}_{i \in \mathbb{N}}$ of paths such that (i) for all $i
\in \mathbb{N}$, the path $\sigma_i$ lies completely inside $X_i$, and (ii) the sequence of paths
converges to the optimal path $\sigma^*$, i.e., $\lim_{i \to \infty} \Vert \sigma_i - \sigma^* \Vert
= 0$. This can be done by dividing $\sigma^*$ into segments and constructing an approximation to any segment
that lies outside $X_i$ (see Figure~\ref{figure:segments}). Note that, since for all large enough
$i$, any point in $X_i$ is at least a certain fraction of $r_i$ away from the obstacles, in
particular, any point on the path $\sigma_i$ is at least a certain fraction of $r_i$
away from the obstacles, for all such $i \in \mathbb{N}$.

Third, for any $i \in \mathbb{N}$, construct a set $B_i$ of overlapping balls, each with radius
$q_i$ and centered at a point on the path $\sigma_i$, such that the balls in $B_i$ collectively
``cover'' the path $\sigma_i$ (see Figure~\ref{figure:allballs}). Moreover, the radius $q_i$ is
chosen in such a way that for all large enough $i$, we have that $q_i$ is less than $r_i$. Hence,
for all such large enough $i$, all the balls in $B_i$ completely lie inside the obstacle-free
region.

Finally, compute the probability of the event that at least one ball in $B_i$ contains no node of
the RRG in iteration $i$; this event is denoted by $A_i$. Subsequently, one can show that $A_i$ can not
occur infinitely often as $i$ approaches infinity, which implies that its complement, $A_i^c$, must
occur infinitely often (with probability one). That is, one can conclude that the RRG will include a node
in each ball in the set $B_i$ of balls for infinitely many $i$'s, with probability one.
Joining the vertices in subsequent balls in $B_i$, one can generate a path, whose cost is shown to be close to $c^*$. Moreover, as $i$ approaches infinity, the costs of such paths converge
$c^*$. Appropriately chosing $q_i$ guarantees that the RRG algorithm joins the points in subsequent
balls by an edge, i.e., the path is included in the RRG. Hence, the result follows.

The assumptions on the cost function are mostly used to make the convergence arguments possible. The
assumptions on the obstacles (i.e., environment), on the other hand, are primarily used for ensuring
that the paths in the sequence $\{ \sigma_i \}_{i \in \mathbb{N}}$ are collision-free.

The proof technique is based on the bin-covering technique \cite{mitzenmacher.upfal.book05}, which
is widely used for analyzing almost-sure properties of random geometric graphs (see,
e.g.,~\cite{muthukrishnan.pandurangan.symp_disc_alg05}). Similar methods were also used for
analyzing PRMs (see~\cite{ladd.kavraki.tro04} and the references therein).

\begin{figure}[ht]
  \begin{center}
    \mbox{ \subfigure[]{\scalebox{0.28}{\includegraphics{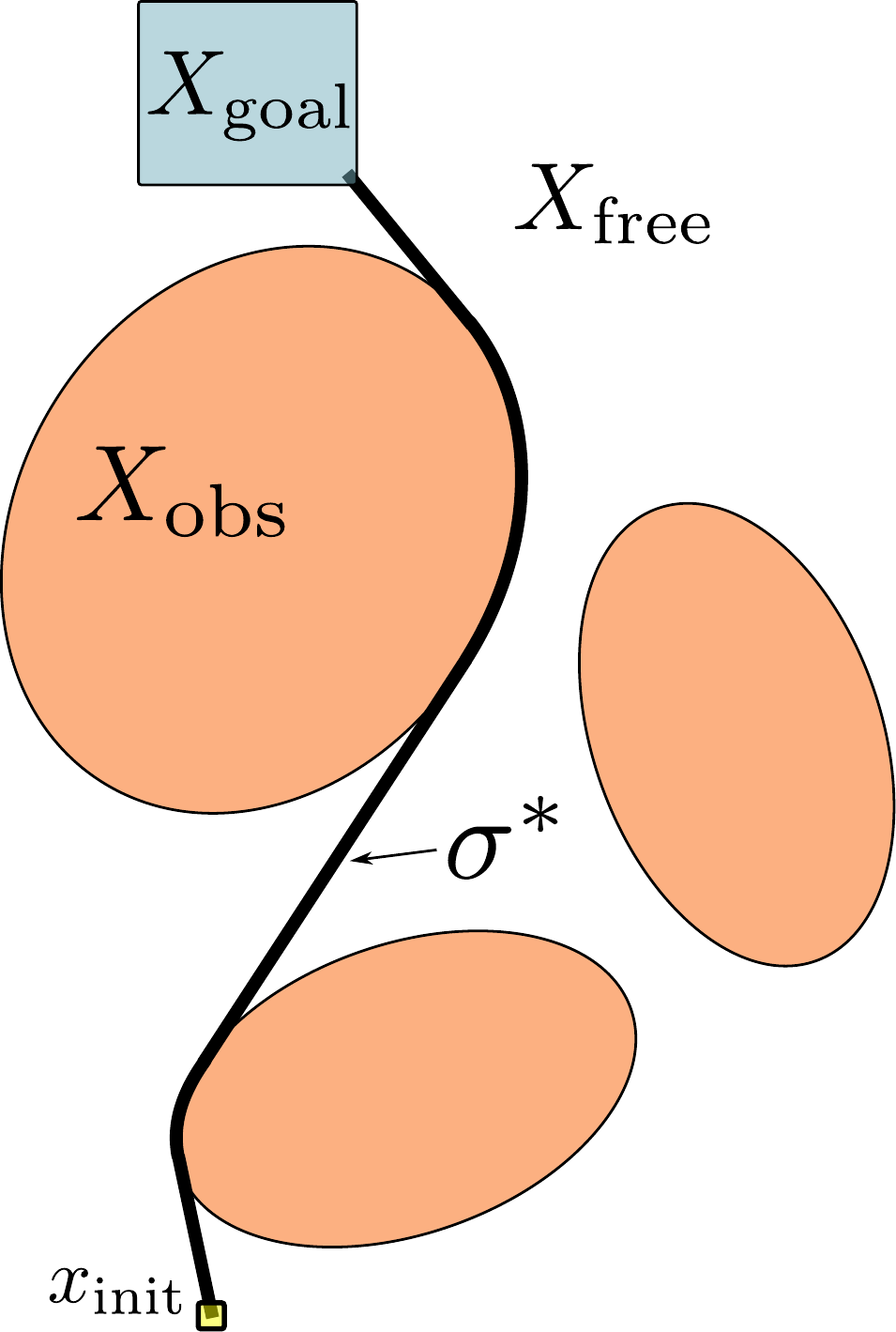}} 
        \label{figure:obstacles_only}} 
      \subfigure[]{\scalebox{0.28}{\includegraphics{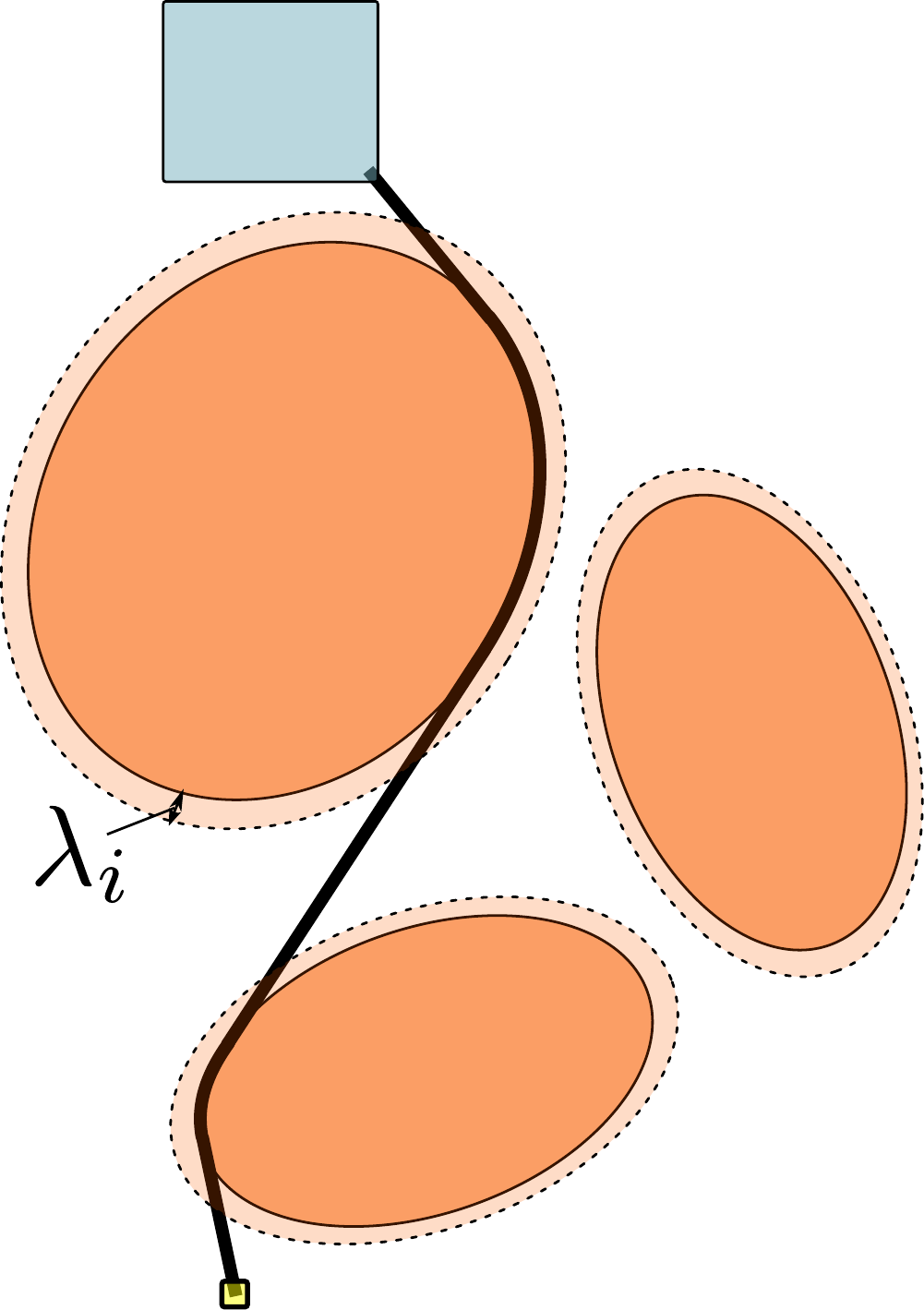}} 
        \label{figure:X_n}}
      \subfigure[]{\scalebox{0.28}{\includegraphics{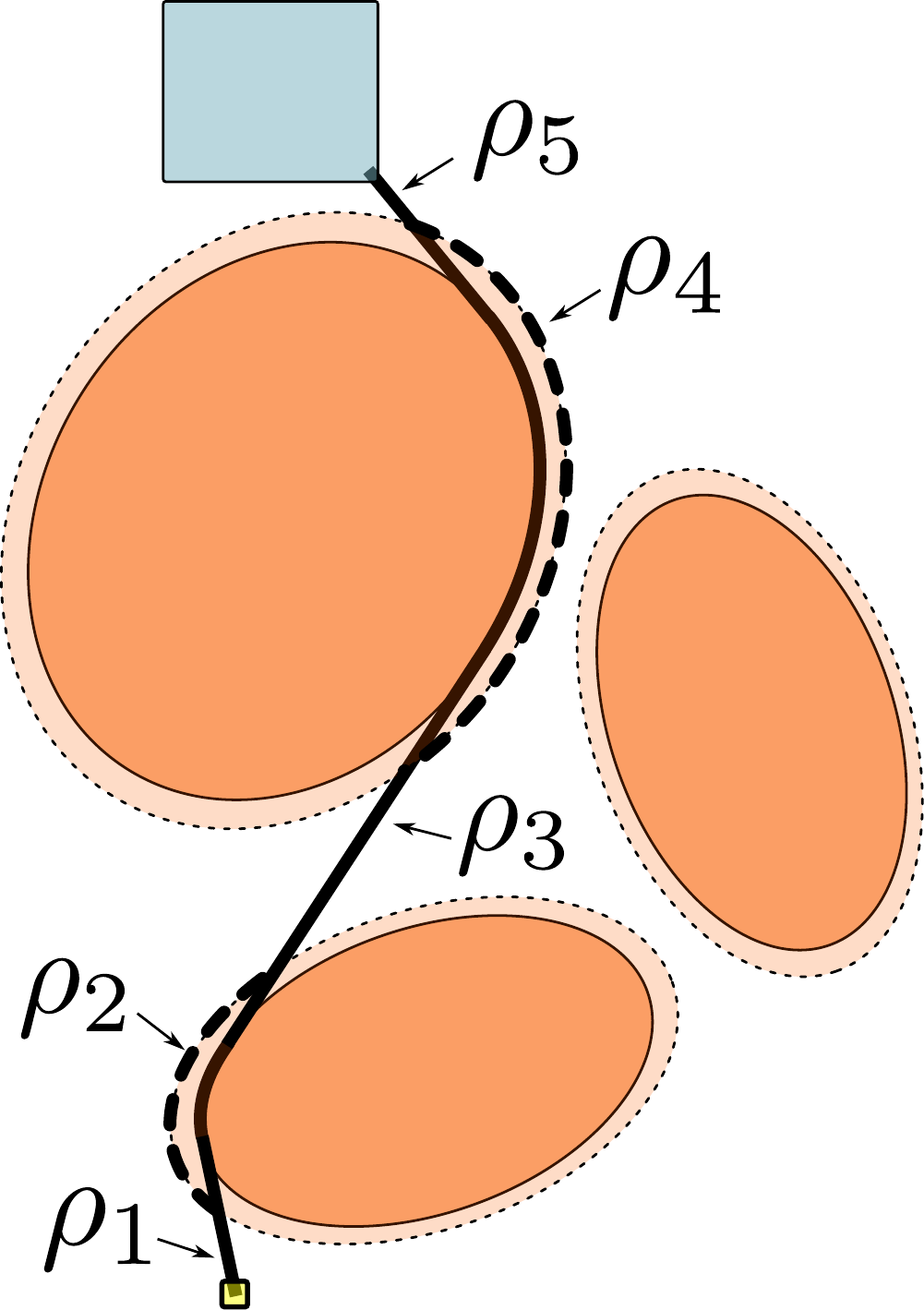}}
        \label{figure:segments}} } 
    \\ \mbox{
      \subfigure[]{\scalebox{0.35}{\includegraphics{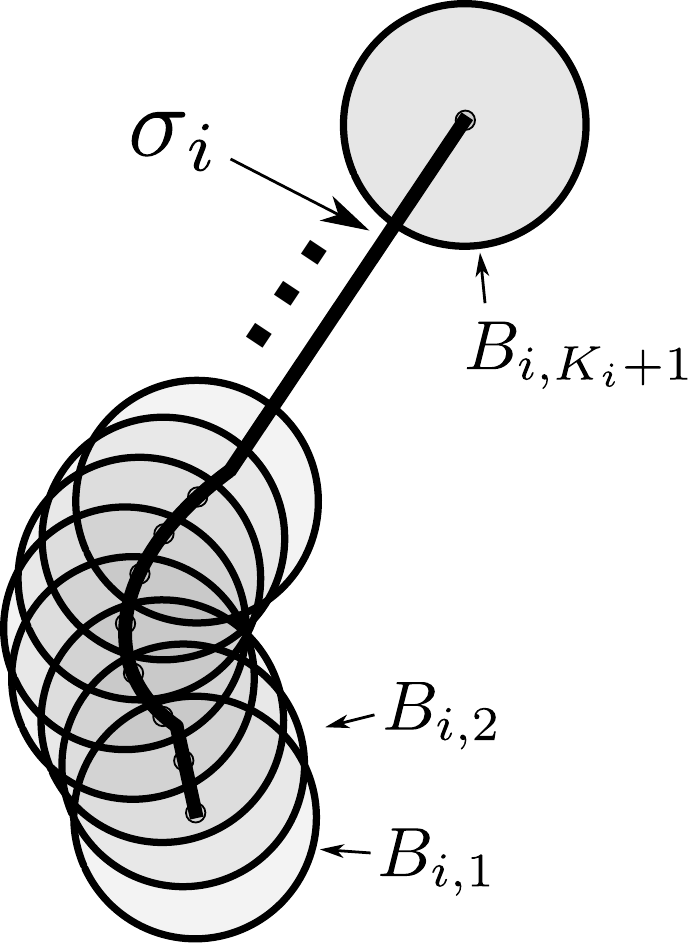}}
        \label{figure:allballs}} \quad\quad\quad
      \subfigure[]{\scalebox{0.22}{\includegraphics{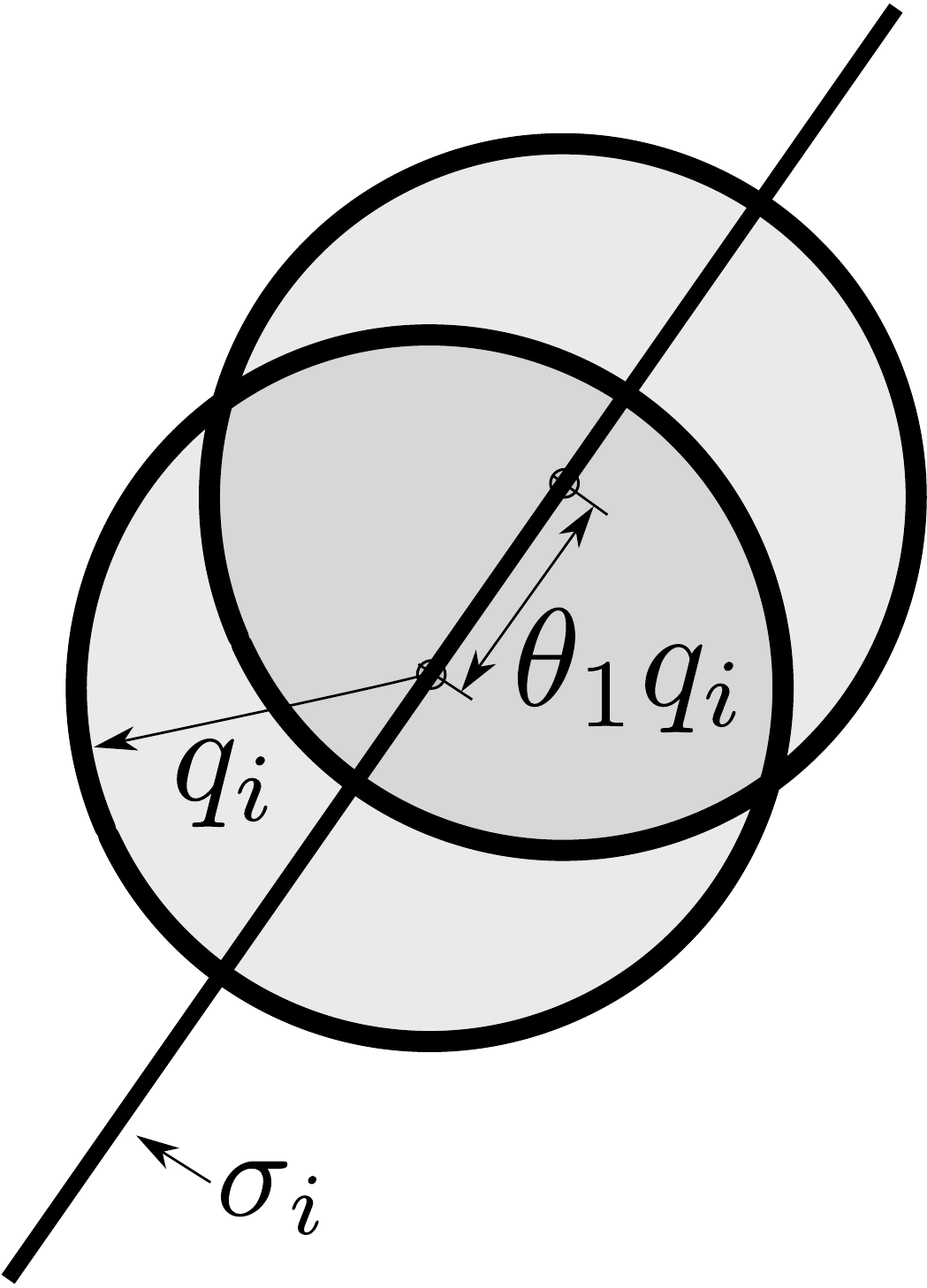}} 
        \label{figure:twoballs}} }
    \caption{A region with obstacles is illustrated in Figure~\ref{figure:obstacles_only}. An
      optimal path is also shown and denoted as $\sigma^*$. In Figure~\ref{figure:X_n}, the region
      that lies outside the light orange shaded shaded area (enclosed by the dotted lines)
      illustrates the set $X_i$. Notice that the dotted lines, hence all the points in the set
      $X_i$, have distance at least $\lambda_i$ to any point in the obstacle region. In
      Figure~\ref{figure:segments}, the optimal path is divided into parts and the segments $\rho_1,
      \rho_2, \dots, \rho_5$ are constructed so that $\rho_1$, $\rho_3$, and $\rho_5$ (solid lines)
      lie inside $X_i$, whereas $\rho_2$ and $\rho_4$ (dotted lines) lie inside $X_\mathrm{free}
      \setminus X_i$. Note that these path segments together make up the path $\sigma_i$, i.e.,
      $\sigma_i = \rho_1 \vert \rho_2 \vert \dots \vert \rho_5$. In Figure~\ref{figure:allballs},
      the tiling of $\sigma_i$ with overlapping balls in $B_i$ are illustrated. Finally, in
      Figure~\ref{figure:twoballs}, this tiling is shown in detail.}
    \label{figure:proof1}
  \end{center}
\end{figure}

\subsubsection*{Construction of the sequence $\{X_i\}_{i \in \mathbb{N}}$}
Let $\theta_1 \in (0,1)$ be a constant. First,  define a sequence
$\{\lambda_i \}_{i \in \mathbb{N}}$ of real numbers as $\lambda_i = \min\left\{ \delta,
\frac{1+\theta_1}{2+ \theta_1} r_i \right\}$, where $\delta$ is given by
Assumption~\ref{assumption:obstacles}.
Construct a sequence $\{ X_i\}_{i \in \mathbb{N}}$ of subsets of $X_\mathrm{free}$ as follows
(see Figure~\ref{figure:X_n}):
$$
X_i := X_\mathrm{free} \setminus \{x \in X_\mathrm{free} \vert \exists x' \in X_\mathrm{obs} .
\Vert x - x' \Vert < \lambda_i \}.
$$
Denote the boundary of $X_i$ by $\partial X_i$. The precise value of the constant $\theta_1$
will become clear shortly.

\subsubsection*{Construction of the sequence $\{ \sigma_i\}_{i \in \mathbb{N}}$}
Assume, without loss of any generality, that the beginning and the end of the optimal path
$\sigma^*$ are at least $\delta$  away from the obstacles, i.e., for $\tau \in \{0, s^*\}$, it is the case
that $\Vert \sigma^*(\tau) - x \Vert \ge \delta$ for all $x \in X_\mathrm{obs}$.

Construct a sequence $\{\sigma_i\}_{i \in \mathbb{N}}$ of paths, where $\sigma_i : [0, s_i]
\to X_i$, as follows (see Figure~\ref{figure:segments} for an illustration).
First,  break up the optimal path $\sigma^*$ into a minimum number of segments such that each
segment completely lies either inside $X_i$ or inside $\Cl(X_\mathrm{free})\setminus X_i$. For this
purpose, define a sequence $\{\tau_0,\tau_1, \dots, \tau_{m+1} \}$ of scalars as follows:
\begin{itemize}
\item $\tau_0 = 0$ and $\tau_{m + 1} = s^*$, 
\item for all $j \in \{1, 2, \dots, m\}$,  $\tau_j \in \partial X_i$, and
\item for all $\{0, 1, \dots, m \}$, either $\sigma^*(\tau') \in X_i$ for all
  $\tau \in (\tau_j, \tau_{j+1})$ or $\sigma^*(\tau') \notin X_i$ for all $\tau \in (\tau_j,
  \tau_{j + 1})$.
\end{itemize}
Second, for each segment of $\sigma^*$ defined by its end points $\sigma^*(\tau_j)$ and
$\sigma^*(\tau_{j+1})$,  construct a path that lies in $X_i$ and agrees with this segment of
$\sigma^*$ in its endpoints. More precisely,  define a set $\{\rho_0, \rho_1, \dots, \rho_m\}$
of paths as follows: (i) if $\sigma^*(\tau') \in X_i$ for $\tau' \in (\tau_j, \tau_{j+1})$, then
let $\rho_j$ be the $j$th segment of $\sigma^*$, i.e., $\rho_j(\tau' - \tau_j) = \sigma^*(\tau')$
for all $\tau' \in [\tau_j, \tau_{j+1}]$, (ii) if, on the other hand, $\sigma^*(\tau') \notin X_i$
for all $\tau' \in (\tau_j, \tau_{j+1})$, then let $\rho_i$ be a minimum-cost continuous path in
$X_i$ starting from $\sigma^*(\tau_j)$ and ending in $\sigma^*(\tau_{j+1})$. Such a path
always exists by the construction of $X_i$ and Assumption~\ref{assumption:obstacles}.
Finally,  define $\sigma_i$ as the concatenation of path segments $\rho_1, \rho_2, \dots,
\rho_m$, i.e.,  define $\sigma_i := \rho_0 \vert \rho_1 \vert \dots \vert \rho_m$.

Two important properties of the sequence of paths $\{\sigma_i\}_{i \in \mathbb{N}}$  follow.
First, each path $\sigma_i$ is at least $\lambda_i$ away from the obstacles. More
precisely $\Vert \sigma_i(\tau) - x \Vert \ge \lambda_i$
 for all $i \in \mathbb{N}$, $\tau \in [0, s_i]$, and $x \in X_\mathrm{obs}$. This claim follows from
Assumption~\ref{assumption:obstacles} and the fact that $\lambda_i \le \delta$ for all $i$.
The second property,  regarding the costs of the paths in $\{\sigma_i \}_{i \in
  \mathbb{N}}$, is stated in the following lemma:
\begin{lemma} \label{lemma:continuous_sigma_n} %
  Under Assumptions~\ref{assumption:additivecost} and \ref{assumption:continuouscost},  $\lim_{i \to \infty} c(\sigma_i) = c (\sigma^*)$ holds.
\end{lemma}
\begin{proof}
  First, for all $i \in \mathbb{N}$, $\sigma_i$ is continuous by construction.
  Second, note that  $X_i \subseteq X_{i+1}$ for all $i \in \mathbb{N}$, and that
  $\cup_{i \in \mathbb{N}} X_i = X_\mathrm{free}$.
  Let $m_i$ be the number of segments of $\sigma^*$ that lie in $X_\mathrm{free} \setminus X_i$.
  By Assumption~\ref{assumption:continuouscost}, the cost of each such segment is bounded by $\kappa
  \, \lambda_i$. Then, by Assumption~\ref{assumption:additivecost}, $\vert c(\sigma_i) - c(\sigma^*)
  \vert \le m_i \kappa \lambda_i$.
  Finally, since $m_i$ is a monotonically non-increasing function of $i$, and $\lambda_i$ converges
  to zero as $i$ approaches infinity, the lemma follows.
\end{proof}

\subsubsection*{Construction of the sequence $\{ B_i \}_{i \in \mathbb{N}}$ of sets of balls}

Next, for each $i \in \mathbb{N}$, construct a set $B_i$ of equal-radius overlapping balls
that collectively ``cover'' the path $\sigma_i$.

Recall the constant $\theta_1\in (0,1)$, introduced above. The amount of overlap between two
consecutive balls will be parametrized by this constant $\theta_1$. Denote the radius of each ball
in $B_i$ by $q_i$, which we define as $q_i := \frac{\lambda_i}{(1 + \theta_1)}$ for reasons to
become clear shortly.

Construct the set $B_i = \{B_{i,1}, B_{i,2}, \dots, B_{i, K_i+1} \}$ of balls recursively
as follows (see Figures~\ref{figure:allballs} and \ref{figure:twoballs}):
\begin{itemize}
\item Let $B_{i,1}$ be centered at $\sigma_i(0)$.
\item For all $k > 1$, let $B_{i,k}$ centered at $\sigma_i(\tau)$, where $\tau$ is such that the
  centers of $B_{i,k}$ and $B_{i, k-1}$ are exactly $\theta_1\, q_i$ apart from each other.
\item Let $K_i$ be the number of balls that can possibly be generated in this manner and, finally,
  let $B_{i,K_{i}+1}$ be the ball centered at $\sigma_i(s_i)$.
\end{itemize}

Now, consider paths that can be constructed by connecting points from balls in $B_{i}$ in a
certain way.
More precisely, for all $k \in \{1,2, \dots, K_i+1\}$, let $x_k$ be a point from the ball $B_{i,k}$.
Consider a path $\sigma_i'$ that is constructed by joining the line segments connecting two
consecutive points in $\{x_1, x_2, \dots, x_{K_i + 1}\}$, i.e., $\sigma_i' = {\tt Line}(x_1, x_2)
\vert {\tt Line}(x_2, x_3) \vert \dots \vert {\tt Line}(x_{K_i}, x_{K_i + 1})$. With an abuse of
notation, let $\Sigma_i'$ denote the set of all such paths.

First, notice that, by the choice of the radius $q_i$ of the balls in $B_i$, any path $\sigma_i' \in
\Sigma_i'$ is collision-free.
\begin{lemma}
  Let Assumption~\ref{assumption:obstacles} hold. For any path $\sigma_i' \in \Sigma_i'$,
  where $\sigma_i' : [0,s_i'] \to X$, $\sigma_i'(\tau) \in X_\mathrm{free}$ for all $\tau \in [0,s_i']$.
\end{lemma}
  
\begin{proof}
  Let $\{x_1, x_2, \dots, x_{K_i + 1} \}$ be the set of points that forms $\sigma_i' = {\tt
    Line}(x_1, x_2) \vert \dots \vert {\tt Line}(x_{K_i}, x_{K_i+1})$. It will be shown that each segment
  ${\tt Line}(x_{k}, x_{k+1})$ is obstacle-free, i.e., lies in $X_\mathrm{free}$, which implies the
  lemma. Let $y_k$ denote the center of the ball $B_{i,k}$.
  Note that $x_k$ and $y_k$ have distance at most $q_i$.
  Similarly, noting that $x_{k+1}$ has distance at most $q_i$ to $y_{k+1}$ and that $y_k$ and
  $y_{k+1}$ are at most $\theta_1 \, q_i$ apart from each other, one can conclude that $x_{k + 1}$ has
  distance at most $(1 + \theta_1) q_i$ to $y_k$. Indeed, $y_k$ has distance at most $(1+\theta_1)
  \, q_i = \lambda_i$ to any point in the convex combination $[x_k, x_{k + 1}]$.
  Finally, recall that, as a consequence of Assumption~\ref{assumption:obstacles}, any point with
  distance at most $\lambda_i$ from $y_k$ lies in $X_\mathrm{free}$, which completes the proof.
\end{proof}

Second, notice that the cost of $\sigma_i'$ is ``close'' to that of $\sigma_i$. Indeed, the
following lemma holds.
\begin{lemma} \label{lemma:convergence_of_sigma_n_prime} Let
  Assumptions~\ref{assumption:additivecost} and \ref{assumption:continuouscost} hold. Let $\{
  \sigma_i' \}_{i \in \mathbb{N}}$ be any sequence of paths such that $\sigma_i' \in \Sigma_i'$ for
  all $i \in \mathbb{N}$, then $\lim_{i \to \infty} \vert c(\sigma_i') - c(\sigma^*) \vert = 0$.
\end{lemma}
\begin{proof}
  The distance between a path $\sigma_i' \in \Sigma_i'$ and the path $\sigma_i$ approaches
   zero as $i$ approaches infinity. Moreover, by Assumptions~\ref{assumption:additivecost} and
  \ref{assumption:continuouscost}, $\vert c(\sigma_i') - c(\sigma_i) \vert$ converges to zero.
  Hence, the result follows from Lemma~\ref{lemma:continuous_sigma_n}.
\end{proof}

Let us define $\varepsilon_i = \sup_{\sigma_i' \in \Sigma_i'} \vert c(\sigma_i') - c^* \vert$.
Hence, Lemma~\ref{lemma:convergence_of_sigma_n_prime} establishes that $\lim_{i \to \infty}
\varepsilon_i = 0$.

Third, according to our choice of $q_i$,  any path in $\Sigma_i'$ is ``constructable'' by the
RRG in the following sense.
\begin{lemma} \label{lemma:constructibility} %
  For all $i \in \mathbb{N}$ and for all $k \in \{1,2, \dots, K_i\}$,  $\Vert x_{k+1}
  - x_k \Vert \le r_i$, where $x_k \in B_{i,k}$.
\end{lemma}
\begin{proof}
  Recall that the radius of each ball in $B_i$ is $q_i = \frac{\lambda_i}{(1 + \theta_1)}$. Given
  any two points $x_{k}$ and $x_{k+1}$, both $x_k$ and $x_{k+1}$ have distances at most
  $q_i$ to the centers of the balls $B_{i, k}$ and $B_{i, k+1}$, respectively. Note also that the
  centers of these balls have distance at most $\theta_1 \, q_i$ to each other. Using the triangle
  inequality together with the definitions of $\lambda_i$ and $q_i$, one obtains $\Vert x_{k+1} - x_k
  \Vert \le (2 + \theta_1) q_i = \frac{2 + \theta_1}{1 + \theta_1} \lambda_i \le r_i$.
\end{proof}
Intuitively, Lemma~\ref{lemma:constructibility} establishes that if each ball in $B_i$ contains at
least one node of the RRG by the end of iteration $i$, then the RRG algorithm will indeed connect
these vertices with edges and construct a path $\sigma_i'$ from $\Sigma_i'$.

\subsubsection*{The probability that a path from $\Sigma_i'$ is constructed}
Let $A_{i,k}$ be the event that the RRG has no node inside the ball $B_{i,k}$ at the end of
iteration $i$.
Moreover, let $A_{i}$ be the event that at least one of the balls in $B_i$ contains no node of the
RRG at the end of iteration $i$, i.e., $A_i = \bigcup_{k = 1}^{K_i+1} A_{i,k}$.

It is appropriate to explain how the main result of this theorem is related to the sequence $\{A_{i}\}_{i \in
  \mathbb{N}}$ of events. Recall that $\varepsilon_i$ was defined as $\varepsilon_i =
\sup_{\sigma_i' \in \Sigma_i'} \vert c(\sigma_i') - c^* \vert$ and that it was already established
that $\varepsilon_i$ converges to zero as $i$ approaches infinity. Hence, given any $\varepsilon$,
one can find a number $i_1 \in \mathbb{N}$ such that for all $i \ge i_1$,
$\varepsilon_i < \varepsilon$.
Recall also that $r_i$ converges to zero as $i$ approaches infinity. Hence, there exists some $i_2
\in \mathbb{N}$ such that $r_i \le \delta$ for all $i \ge i_2$.
Let $i_0 = \max\{i_1, i_2 \}$ and note the following upper bound:
\begin{eqnarray*}
  & & \hspace{-0.25in} \sum_{i = 0}^\infty \PP\left(\left\{{\cal Y}^\mathrm{RRG}_i 
  > c^*+\varepsilon\right\}\right) \\
  & & \hspace{-0.22in} = \sum_{i = 0}^{i_0 - 1} \PP \left(\left\{ {\cal Y}^\mathrm{RRG}_i > c^*
      + \varepsilon \right\}\right) 
  + \sum_{i = i_0}^\infty \PP\left(\left\{ {\cal Y}^\mathrm{RRG}_i > c^*+\varepsilon \right\}\right) \\
  & & \hspace{-0.22in} \le  i_0 + \sum_{i = i_0}^\infty \PP\left(\left\{{\cal Y}^\mathrm{RRG}_i >
      c^* + \varepsilon_i\right\}\right) \le  i_0 + \sum_{i = i_0}^\infty \PP(A_i).
\end{eqnarray*}

To complete the proof, it will be established that $\sum_{i = 0}^\infty \PP (A_i) < \infty$, which implies
that $\sum_{i = 0}^\infty \PP (\{{\cal Y}^\mathrm{RRG}_i > c^* + \varepsilon \}) \le \infty$ for all
$\varepsilon > 0$, which in turn implies the result.

To formally show this claim,  define the following sequence of events. For all
$i \in \mathbb{N}$, let $C_i$ be the event that for any point $x \in X_\mathrm{free}$ the RRG
algorithm includes a node $v$ such that $\Vert x - v \Vert \le \eta$ and that the line segment
${\tt Line}(x, v)$ joining $v$ and $x$ is obstacle-free.
Under the assumptions of Theorem~\ref{theorem:rrtrrgequivalency}, the following lemma holds:
\begin{lemma} \label{lemma:proof_pos:exponential_coverage} %
  Let the assumptions of Theorem~\ref{theorem:rrtrrgequivalency} and
  Assumption~\ref{assumption:obstacles} hold, then there exist constants $a, b \in \mathbb{R}_{>0}$
  such that $\PP (C_i^c) \le a e^{-b i}$ for all $i \in \mathbb{N}$.
\end{lemma}
\begin{proof}
  The set $X_\mathrm{free}$ can be partitioned into finitely many convex sets such that each
  partition is bounded by a ball of radius $\eta$, since $X_\mathrm{free}$ is a bounded subset of
  $\mathrm{R}^d$. Since, by Theorem~\ref{theorem:rrtrrgequivalency}, the probability that each such
  ball does not include a node of the RRG decays to zero exponentially as $i$ approaches infinity,
  the probability that at least one such ball does not contain a node of the RRG also decays to zero
  exponentially (this follows from the union bound). Moreover, if each partition contains a node of
  the RRG, then any point inside the partition can be connected to the associated node of the RRG
  with an obstacle-free line segment of length less than $\eta$, since each partition is convex and
  is bounded by a ball of radius $\eta$.
\end{proof}

The following lemma constitutes another important step in proving the main
result of this theorem. 
\begin{lemma} \label{lemma:finite_A_given_C} %
  Let $\gamma > 2^d (1+ 1/d) \mu(X_\mathrm{free})$, and $\theta_2 \in (0,1)$ be a constant. Then,
  $\sum_{i = 1}^\infty \PP \big(A_i \biggiven \bigcap_{j = \lfloor \theta_2 \, i \rfloor}^i C_j\big)
  < \infty$.
\end{lemma}

\begin{proof}
  It is desired to compute $\PP (A_i \given \cap_{j = \lfloor \theta_2 \, i \rfloor}^i C_j)$ under
  the theorem's assumptions. Since $\cap_{j = \lfloor \theta_2 \, i\rfloor}^i C_j$ is given, it can
  be assumed that, from iteration $\lfloor \theta_2 \, i \rfloor$ to iteration $i$, the RRG has the
  following coverage property: for every point $x \in X_\mathrm{free}$, the RRG includes a node $v$
  that is at most $\eta$ away from $x$ and the line segment between $x$ and $v$ lies in the
  obstacle-free region.
  
  Recall that the length of $\sigma_i$ was $s_i$. A bound on the number of balls in
  $B_i$ with respect to $i$, $s_i$, and other constants of the problem can be derived as follows.
  Note that the segment of $\sigma_i$ that starts at the center of $B_{i,k}$ and ends at the center
  of $B_{i,k+1}$ has length at least $\theta_1 \, q_i$ (recall that distance between the centers of
  two consequtive balls in $B_i$ is $\theta_1 q_i$ by construction), except for the last segment,
  which has length less than $\theta_1 q_i$. Thus, for all $i \ge i_2$,
  \begin{eqnarray*}
    \vert B_i \vert  
    = K_i + 1 & \le & \frac{s_i}{\theta_1 \, q_i} 
    = \frac{(2+ \theta_1) s_i}{\theta_1 r_i} \\
    & = & \frac{2 + \theta_1}{\theta_1} s_i \left(\frac{\gamma}{\zeta_d} \right)^{1/d}
    \left(\frac{i}{\log i}\right)^{1/d}
  \end{eqnarray*}

  Also, given $\cap_{j = \lfloor \theta_2 \, i \rfloor}^{i} C_j$, the probability that
  a single ball, say $B_{i,1}$, does not contain a node of the RRG at the end of iteration $i$ can
  be upper-bounded as follows:
  \begin{eqnarray*}
    \PP (A_{i,1} \given \cap_{j = \lfloor \theta_2 \, i \rfloor}^i C_j) 
    & \le & \left( 1 -\frac{\mu(B_{i,1})} {\mu(X_\mathrm{free})} \right)^{i - \theta_2 \, i}  \\
    & & \hspace{-0.5in} = \left(1 - \frac{\gamma}{(2 + \theta_1)^d \mu(X_\mathrm{free})} 
      \frac{\log i}{i} \right)^{(1 - \theta_2) \,  i} 
  \end{eqnarray*}
  Using the fact that $(1 - 1/f(i))^r \le e^{-r/f(i)}$, the right-hand side can further be bounded
  to obtain the following inequality:
  \begin{eqnarray*}
    \PP (A_{i,1} \given \cap_{i = \lfloor \theta_2 \, i \rfloor}^i C_j) 
    \le e^{-\frac{(1 - \theta_2) \gamma} {(2+\theta_1)^d \mu(X_\mathrm{free})} \log i} 
    = i^{-\frac{(1 - \theta_2) \gamma}{(2 + \theta_1)^d \mu(X_\mathrm{free})}}
  \end{eqnarray*}

  As a result of the discussion above, we have that
  \begin{eqnarray*}
    \PP \left(A_i \biggiven \cap_{j = \lfloor \theta_2 \, i \rfloor}^i C_j \right) 
    & \le & \PP \left(\cup_{i = 1}^{\vert B_{i} \vert} A_{i,k} 
      \biggiven \cap_{j = \lfloor \theta_2 \, i \rfloor}^i C_j\right)\\ 
    & & \hspace{-1.5in} \le \sum_{i = 1}^{\vert B_i \vert} 
    \PP\left(A_i \biggiven \cap_{j = \lfloor \theta_2 \, i \rfloor}^i C_j\right) 
    = \vert B_i \vert \,\, \PP \left(A_{i, 1} \biggiven \cap_{j = \lfloor \theta_2 \, i \rfloor}^i C_j \right) \\
    & & \hspace{-1.5in} \le \frac{2 \, s_i \, \mu(X_\mathrm{free})^{1/d}}{\gamma^{1/d}} 
    \left(\frac{i}{\log i} \right)^{1/d} i^{-\frac{(1 - \theta_2) \gamma}
      {(2+\theta_1)^d \mu(X_\mathrm{free})}} \\
    & & \hspace{-1.5in} = \frac{2 \, s_i \, \mu(X_\mathrm{free})^{1/d}}{\gamma^{1/d}} \frac{1}{(\log i)^{1/d}}
    i^{-\left(\frac{(1 - \theta_2) \gamma}{\mu(X_\mathrm{free}) (2 + \theta_1)^d} - \frac{1}{d} \right)}.
  \end{eqnarray*}

  Note that $\sum_{i = 1}^\infty \PP (A_i \given \cap_{j = \lfloor \theta_2 \, i \rfloor}^i C_j) <
  \infty$, for all choices of $\gamma$ such that $\frac{(1 - \theta_2)
    \gamma}{\mu(X_\mathrm{free}) (2 + \theta_1)^d} - \frac{1}{d} \ge 1$ holds, i.e., whenever
  $\gamma > 2^d (1 + 1/d) \mu(X_\mathrm{free})$, noting that for any such choice of $\gamma$,
  there exists $\theta_1 > 0$ and $\theta_2 > 0$ satisfying the former inequality. Hence, the
  lemma follows.
\end{proof}

\begin{lemma} \label{lemma:finite_C} %
  Let $\theta_2 \in (0,1)$ be a constant. Then, 
  $\sum_{i = 1}^\infty \PP \left(\big(\cap_{j = \lfloor \theta_2\, i \rfloor}^i C_j\big)^c\right)
  < \infty$.
\end{lemma}
\begin{proof}    
  Note the following inequalities:
  \begin{eqnarray*}
    \sum_{i = 1}^\infty \PP \left((\cap_{j = \lfloor \theta_2 \, i 
        \rfloor}^i C_j)^c\right) 
    & = & \sum_{i = 1}^\infty \PP \left(\cup_{j = \lfloor \theta_2 \, i \rfloor}^i C_j^c\right) \\
    & & \hspace{-1.0in}\le \sum_{i = 1}^\infty \sum_{j = \lfloor \theta_2 \, i \rfloor}^i \PP (C_j^c), 
    \le \sum_{i = 1}^\infty \sum_{j = \lfloor \theta_2 \, i \rfloor}^i a' e^{- b' j},
  \end{eqnarray*}
  in which the right-hand side is finite for all $\theta_2\in(0,1)$, and where the second
  inequality follows from Lemma~\ref{lemma:proof_pos:exponential_coverage}.
\end{proof}

Finally,  the relationship of Lemmas~\ref{lemma:finite_A_given_C} and
\ref{lemma:finite_C} to the main result can be pointed out. Note the following inequality:
\begin{eqnarray*}
  \PP (A_i \given \cap_{j = \lfloor \theta_2 \, i \rfloor}^i C_j)
  & & \hspace{-0.25in} = \frac{ \PP \left(A_i \cap (\cap_{j = \lfloor \theta_2 \, i \rfloor}^i C_j)\right)}
  {\PP(\cap_{j = \lfloor \theta_2 \, i \rfloor}^i C_j)} \\
  & & \hspace{-1.0in} \ge  \PP \left(A_i \cap (\cap_{j = \lfloor  \theta_2 \, i \rfloor}^i C_j)\right) \\
  & & \hspace{-1.0in} = 1 - \PP \left(A_i^c \cup (\cap_{j = \lfloor \theta_2 \, i \rfloor}^i C_j)^c \right) \\
  & & \hspace{-1.0in} \ge  1 - \PP (A_i^c) 
  - \PP \left((\cap_{j = \lfloor \theta_2 \, i \rfloor}^i C_j)^c \right) \\
  & & \hspace{-1.0in} = \PP (A_i) - \PP\left((\cap_{j = \lfloor \theta_2 \, i \rfloor}^i C_j)^c \right).
\end{eqnarray*}
Hence, taking infinite sums of both sides with respect to $n$ and rearranging, the
following inequality is established:
\begin{eqnarray*}
  \sum_{i = 1}^\infty \PP (A_i) 
  \le 
  \sum_{i = 1}^\infty \PP (A_i \given \cap_{j = \lfloor \theta_2 \, i \rfloor}^i C_j) +
  \sum_{i = 1}^\infty \PP \left((\cap_{j = \lfloor  \theta_2 \, i \rfloor}^i C_j)^c\right)
\end{eqnarray*}
By Lemmas~\ref{lemma:finite_A_given_C} and \ref{lemma:finite_C}, the right-hand side is finite,
whenever $\gamma > 2^d(1 + 1/d) \mu(X_\mathrm{free})$, which implies that, under the same
conditions, $\sum_{i = 1}^\infty \PP (A_i) < \infty$, which in turn implies the result.

\subsection*{Proof of Lemma~\ref{lemma:complexity_of_obstaclefree}}

The proof of this lemma employs results from the Random Geometric Graph (RGG) theory (see,
e.g.,~\cite{penrose.book03}). First, some of the relevant results in the theory of RGGs will be introduced, and then the proof of this lemma will be given.

An RGG $G (V_n, r_n) = (V_n , E_n)$ in $X_\mathrm{free}$ with $n$ vertices is formed as follows: (i)
$V_n$ is a set of $n$ vertices that are sampled identically and independently from $X_\mathrm{free}$ according to a
distribution with a continuous density function $f(x)$, (ii) two vertices $v, v' \in
V_n$ are connected with an edge if and only if the distance between them is less than $r_n$, i.e.,
$\Vert v - v' \Vert \le r_n$.

The following lemma is a special case of Proposition 3.1
in~\cite{penrose.book03}.
\begin{lemma}[see \cite{penrose.book03}] \label{lemma:edgecount} %
  If $r_n$ satisfies $\lim_{n \to \infty} r_n = 0$, then
  $$ 
  \lim_{n \to \infty} \frac{\EE[ \vert E_n \vert ]}{n^2 \, (r_n)^{d}} = \frac{1}{2} \, \zeta_d \,
  \int_{x \in \mathbb{R}^d} f (x)^2 \, d x.
  $$
\end{lemma}
Note that, if $f(x)$ is the uniform density function, then $\int_{x \in
  \mathbb{R}^d} f(x)^2 \,d x = 1/\mu(X_\mathrm{free})$.

Let $V^\circ_n$ denote the number of vertices that are isolated, i.e., those that do not have any
neighbors. The following lemma characterizes the ``density'' of the isolated vertices in the limit.
\begin{lemma} \label{lemma:isolatedvertexcount} %
  If $\zeta_d (r_n)^d = \gamma \frac{\log n}{n}$ for some $\gamma \in \mathbb{R}_{>0}$, then the
  number of isolated vertices, $\vert V^\circ_n\vert$, satisfies the following:
  $$
  \frac{\EE [ \vert V^\circ_n \vert ] }{n} 
  \le \left( 1 - \gamma \frac{\log n}{n}\right)^{n-1}.
  $$
  Thus, $\lim_{n \to \infty} \frac{\EE [\vert V^\circ_n \vert]}{n} = 0$; moreover,
  $\sum_{n = 1}^{\infty} \frac{\EE [\vert V^\circ_n \vert]}{n}$ is finite whenever $\gamma > 1$.
\end{lemma}
\begin{proof}
  Reference \cite{penrose.book03} does not consider the case when $n (r_n)^d \to \infty$ as $n \to
  \infty$. However, following the reasoning in the proof for the case when $n
  (r_n)^d \to \rho \in \mathbb{R}_{>0}$ (see the proof of Proposition 3.3 in~\cite{penrose.book03}),
  it is possible construct a proof for this lemma. Let us sketch this proof.
  
  Following~\cite{penrose.book03} (note that the the notation in \cite{penrose.book03} is slightly
  different from that in this paper), and after appropriate simplifications, one gets
  \begin{eqnarray*}
    \frac{\EE[V^\circ_n ]}{n} = \int_{x \in X_\mathrm{free}} \left(1 - I_n(x) \right)^{n-1}
    f(x) dx,
  \end{eqnarray*}
  where $f(x)$ is the density function, which is uniform over $X_\mathrm{free}$ in the case under
  consideration, and $I_n (x) = \int_{x' \in {\cal B}_{x,r_n}} x' dx'$, which evaluates to $I_n =
  \gamma\frac{log n}{n}$ whenever $(r_n)^d = \gamma \frac{\log}{n}$ (the reader is referred to
  \cite{penrose.book03} for details). Hence,
  \begin{eqnarray*}
    \frac{\EE[V^\circ_n ]}{n} 
    & = & \int_{x \in X_\mathrm{free}} \left(1 - \gamma\frac{\log
        n}{n}\right)^{n - 1} \frac{1}{\mu(X_\mathrm{free})} d x \\
    & = & \left(1 - \gamma\frac{\log n}{n}\right)^{n-1},
  \end{eqnarray*}
  which converges to zero as $n$ approaches infinity. Moreover, the sum of these terms is finite
  whenever $\gamma > 1$.
\end{proof}

Consider the random geometric graph with the vertex set $V_i$ that is formed with the ${\tt Sample}$
procedure, i.e., $V_i = \cup_{j = 1}^i \{{\tt Sample}(j)\}$. Let $V^0_i$ denote the set of isolated
vertices and let $V^1_i$ denote the set of the remaining vertices, i.e., $V^1_i = V_i \setminus V^0_i$. We
will denote by $E^0_i$ and $E^1_i$ the set of all edges that connect the vertices in $V^1_i$ and those
in $V^0_i$, respectively. Hence, we have that $E^0_i = \emptyset$ and $E^1_i = E_n$ for all $i \in
\mathbb{N}$. 

Consider the RRG algorithm. Let $V^2_i$ denote the set of all vertices that result from extending the
graph when the corresponding sample is isolated, i.e., no vertices were present within the ball of
volume $\gamma \frac{\log {\cal N}_i}{{\cal N}_i}$ centered at the sample, ${\tt Sample} (i)$, in
iteration $i$.
Let $E^2_i$ denote the set of all edges that are connected to the vertices in $V^2_i$. 

Notice the following two facts: (i) the random geometric graph with node set $V_i$ can be described
by $G_{RGG,i} = (V^0_i \cup V^1_i, E^0_i \cup E^1_i)$, and (ii) the graph maintained by the RRG
algorithm is a subgraph of $G_{RRG,i} = ({\cal V}^\mathrm{RRG}_i, {\cal E}^\mathrm{RRG}_i) = (V^1_i
\cup V^2_i, E^1_i \cup E^2_i)$.
\newpage
Notice that the number of calls to ${\tt ObstacleFree}$ in iteration $i$ is exactly the number of
edges created at that iteration, hence is less than or equal to $\vert E_i^1 \cup E_i^2 \vert =
\vert E_i^1 \vert + \vert E_i^2 \vert$.
The following lemmas lead immediately to the result:
\begin{lemma}
  $$
  \limsup_{i \to \infty} \EE \left[ \frac{\vert E_i^1 \vert}{ {\cal N}_i \log ({\cal N}_i)  }
  \right] \le \frac{1}{2} \frac{\gamma}{\mu(X_\mathrm{free})}.
  $$
\end{lemma}
\begin{proof}
  Consider the expected value of $\vert E_i^1 \vert/ (n^2 \, (r_n)^d)$, when $n = {\cal N}_i$.
  Substituting $\zeta_d (r_n)^d = \gamma \frac{\log ({\cal N}_i)}{{\cal N}_i}$ and using
  Lemma~\ref{lemma:edgecount}, the result follows.
\end{proof}

\begin{lemma}
  $
  \limsup_{i \to \infty} \EE \left[ \frac{\vert E_i^2 \vert}{{\cal N}_i \log({\cal N}_i)} \right]
  $
  is finite.
\end{lemma}
\begin{proof}
  Let $M^2_i$ denote the number of edges added to
  \noindent $E^2_i$ at iteration $i$. Note that 
  $$
  \EE [M^2_i] \le \EE \left[ \frac{\gamma \frac{\log ({\cal N}_i)}{{\cal
        N}_i}}{\mu (X_\mathrm{free})}  \vert V^2_i \vert \right] 
  = \frac{\gamma}{\mu(X_\mathrm{free})} 
  \EE\left[\frac{\log ({\cal N}_i)}{ {\cal N}_i } \vert V_i^2 \vert \right].
  $$
  Hence, noting $E^2_i = \sum_{j = 1}^i M^2_i$,
  \begin{eqnarray*}
    \limsup_{i \to \infty} \EE \left[ \frac{\vert E_i \vert}{{\cal N}_i \log ({\cal N}_i)} \right] 
    & \le &  \limsup_{i \to \infty} \EE \left[ \frac{ \sum_{j = 1}^i \frac{\log ({\cal N}_j)}{{\cal N}_j}
        \vert V_j^2 \vert}{{\cal N}_i \log ({\cal N}_i)}  \right] \\
  \end{eqnarray*}
  Consider the expectation in the right hand side. Noting that, surely, $\log ({\cal N}_j) \le \log
  ({\cal N}_i)$ for all $j \le i$, the following inequalities hold: 
  \begin{eqnarray*} 
    \EE \left[ \frac{ \sum_{j = 1}^i \frac{\log ({\cal N}_j)}{{\cal N}_j} \vert V_j^2 \vert}{{\cal N}_i \log
        ({\cal N}_i)} \right] 
    & \le & \EE\left[ \frac{1}{{\cal N}_i}  \sum_{j = 1}^i \frac{\vert V_j^2\vert}{{\cal N}_j} \right] \\
    & \hspace{-0.3in}\le & \hspace{-0.2in} \EE \left[ \frac{1}{{\cal N}_i} \right]^{1/2} \, \EE
    \left[ \sum_{j = 1}^i \frac{\vert V^2_j \vert}{{\cal N}_j}  \right]^{1/2} \\
    & \hspace{-0.3in}\le & \hspace{-0.2in} \EE \left[ \frac{1}{{\cal N}_i} \right]^{1/2} \, \left(\sum_{j =
      1}^i  \EE \left[ \frac{\vert V^2_j \vert}{{\cal N}_j}  \right]\right)^{1/2} \\
  \end{eqnarray*}
  where the second inequality follows from the Cauchy-Schwarz inequality. Notice that the first term
  approaches zero as $i \to \infty$, whereas, by Lemma~\ref{lemma:isolatedvertexcount}, the second
  term is finite. Hence, it is concluded that
  $$
  \limsup_{i \to \infty} \EE \left[ \frac{1}{{\cal N}_i} \right]^{1/2} \, \left(\sum_{j = 1}^i \EE
    \left[ \frac{\vert V^2_j \vert}{{\cal N}_j} \right]\right)^{1/2} < \infty, 
  $$
  which implies the lemma.
\end{proof}

\end{document}